%% file: iclr2027_conference.tex
\documentclass{article} 
\PassOptionsToPackage{dvipsnames}{xcolor}
\usepackage{iclr2027_conference,times}
\usepackage{xcolor}

\usepackage[utf8]{inputenc} 
\usepackage[T1]{fontenc}    
\usepackage{hyperref}       
\usepackage{url}            
\usepackage{booktabs}       
\usepackage{amsfonts}       
\usepackage{nicefrac}       
\usepackage{microtype}      

\usepackage{latexsym}
\usepackage{graphicx}
\usepackage{amsmath}
\usepackage{amssymb}
\usepackage{mathtools}
\usepackage{amsthm}
\usepackage{array}
\usepackage{longtable}
\usepackage[most]{tcolorbox}
\usepackage{bibentry}
\usepackage{bm} 
\usepackage{tabularx}
\usepackage{multirow} 
\usepackage{algorithm}
\usepackage{algorithmic}
\usepackage{appendix}
\usepackage{wrapfig}   
\usepackage{bbm}
\usepackage{enumitem}
\tcbuselibrary{listings}
\makeatletter
\def\blankfootnote{\xdef\@thefnmark{}\@footnotetext}
\makeatother

\definecolor{royalblue(web)}{rgb}{0.25, 0.41, 0.88}
\definecolor{blue-violet}{rgb}{0.54, 0.17, 0.89}
\definecolor{brightmaroon}{rgb}{0.76, 0.13, 0.28}
\definecolor{darkmagenta}{rgb}{0.55, 0.0, 0.55}
\definecolor{bleudefrance}{rgb}{0.19, 0.55, 0.91}
\definecolor{palatinateblue}{rgb}{0.15, 0.23, 0.89}
\definecolor{royalblue(web)}{rgb}{0.25, 0.41, 0.88}
\definecolor{whitesmoke}{rgb}{0.96, 0.96, 0.96}
\definecolor{thulianpink}{rgb}{0.87, 0.44, 0.63}
\definecolor{amber(sae/ece)}{rgb}{1.0, 0.49, 0.0}
\definecolor{darkblue}{rgb}{0.0, 0.0, 0.55}
\definecolor{alizarin}{rgb}{0.82, 0.1, 0.26}
\definecolor{asparagus}{rgb}{0.53, 0.66, 0.42}
\definecolor{darkspringgreen}{rgb}{0.09, 0.45, 0.27}
\definecolor{columbiablue}{rgb}{0.61, 0.87, 1.0}
\definecolor{wildblueyonder}{rgb}{0.64, 0.68, 0.82}
\definecolor{trolleygrey}{rgb}{0.5, 0.5, 0.5}
\definecolor{paleaqua}{rgb}{0.74, 0.83, 0.9}
\definecolor{bubblegum}{rgb}{0.99, 0.76, 0.8}
\definecolor{coralred}{rgb}{1.0, 0.25, 0.25}
\definecolor{green(ryb)}{rgb}{0.4, 0.69, 0.2}
\definecolor{flame}{rgb}{0.89, 0.35, 0.13}
\definecolor{bittersweet}{rgb}{1.0, 0.44, 0.37}
\definecolor{darksalmon}{rgb}{0.91, 0.59, 0.48}
\definecolor{emerald}{rgb}{0.31, 0.78, 0.47}
\definecolor{green(pigment)}{rgb}{0.0, 0.65, 0.31}
\definecolor{codegreen}{rgb}{0,0.6,0}
\definecolor{codegray}{rgb}{0.5,0.5,0.5}
\definecolor{codepurple}{rgb}{0.58,0,0.82}
\definecolor{backcolour}{rgb}{0.96,0.96,0.94}
\definecolor{bluegray}{rgb}{0.3, 0.38, 0.47}
\definecolor{whitesmoke}{rgb}{0.96, 0.96, 0.96}
\definecolor{codegreen}{rgb}{0,0.6,0}
\definecolor{codegray}{rgb}{0.5,0.5,0.5}
\definecolor{codepurple}{rgb}{0.58,0,0.82}
\definecolor{backcolour}{rgb}{0.96,0.96,0.94}
\definecolor{darkred}{rgb}{0.8, 0, 0}

\newcommand{\gain}[1]{\textcolor{darkred}{\small{(+#1)}}}

\tcbuselibrary{breakable}
\usepackage{listings}
\lstdefinestyle{mystyle}{
  basicstyle=\scriptsize\ttfamily,
  frame=single, 
  columns=fixed, 
}

\lstdefinestyle{newstyle}{
  basicstyle=\footnotesize\ttfamily\color{codegreen},
  backgroundcolor=\color{backcolour},
  frame=shadowbox, 
  rulecolor=\color{red},
  frameround=tttt, 
  keywordstyle=\color{magenta},
  commentstyle=\color{green},
  stringstyle=\color{red},
  showstringspaces=false,
  numbers=left,
  numberstyle=\tiny\color{gray},
  breaklines=true
}

\newtheorem{assumption}{Assumption}

\newtheorem{theorem}{Theorem}

\newtheorem{definition}{Definition}

\newcommand{\ours}{{\fontfamily{qpl}\selectfont ROSA}}
\newcommand{\parad}{{\fontfamily{qpl}\selectfont T$^2$PAM}}

\renewcommand{\eqref}[1]{(\ref{#1})}
\input{math_commands}

\title{Test-Time Policy Adaptation for Enhanced Multi-Turn Interactions with LLMs}
\author{Anonymous authors}

\author{
  Chenxing Wei$^{\dagger \S}$, Hong Wang$^{\circ}$, Ying He$^{ \spadesuit}$,  Fei Yu$^{ \ddagger}$, Yao Shu$^\wr$\thanks{corresponding author.}\\
$^\dagger$School of Computing and Data Science, The University of Hong Kong \\
$^\circ$University of Science and Technology of China \\
$^{\S}$Shenzhen Loop Area Institute \\
$^{\wr}$Hong Kong University of Science and Technology (Guangzhou) \\
$^{\spadesuit}$ Shenzhen University\\
$^{\ddagger}$School of Information Technology, Carleton University \\
\texttt{weichenxing@connect.hku.hk}, \texttt{yaoshu@hkust-gz.edu.cn}\\
}

\iclrfinalcopy 

\begin{document}
\maketitle

\begin{abstract}
  Large Language Models (LLMs) employ multi-turn interaction as a fundamental paradigm for completing complex tasks. However, their performance often degrades in extended interactions, as they are typically trained on static, single-turn data, which hinders their ability to adapt to real-time user feedback. To address this limitation, we first propose a new paradigm: \textit{\underline{T}est-\underline{T}ime \underline{P}olicy \underline{A}daptation for \underline{M}ulti-Turn Interactions} (\parad{}), which utilizes user feedback from the ongoing interaction as a reward signal to estimate a latent optimal policy aligned with user preferences, then updates a small subset of parameters to steer the model toward this policy, ultimately enabling efficient in-conversation self-correction. We then introduce \textit{Optimum-\underline{R}eferenced \underline{O}ne-\underline{S}tep \underline{A}daptation} (\ours{}), a lightweight algorithm that operationalizes \parad{}. \ours{} guides the model parameters toward a theoretical optimal policy in a single, efficient update step, avoiding costly iterative gradient-based optimization and minimizing computational overhead. We provide a rigorous theoretical analysis guaranteeing that the policy of \ours{} converges to the preference of user as the number of interactions increases. Extensive experiments on challenging benchmark demonstrate that \ours{} achieves significant improvements in both effectiveness and efficiency.
\end{abstract}

\section{Introduction}
\label{sec:introduction}

\begin{figure}[t]
\vspace{-4mm}
\centering
\includegraphics[width=1.0\columnwidth]{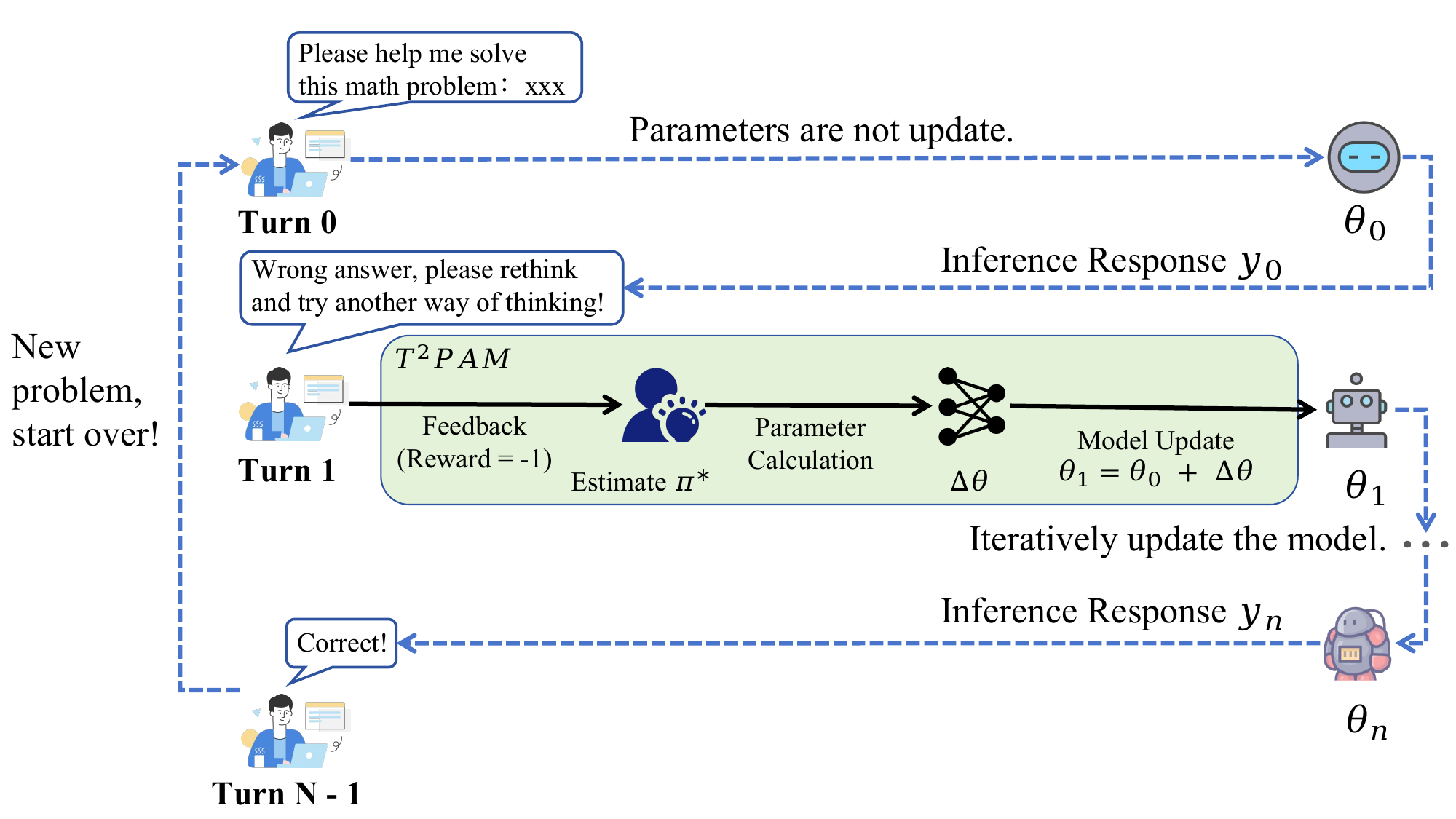}
\caption{An illustration of the \textit{Test-Time Policy Adaptation for Multi-Turn Interactions} (\parad{}) paradigm. Different from static inference where the policy of model remains fixed ($\theta_0$, Turn 0), this paradigm treats conversational feedback as an active signal that guides real-time parameter updates (e.g., from $\theta_0$ to $\theta_1$). This iterative process of in-conversation self-correction allows the policy to progressively evolve and align with the preference of user ($\theta_n$) throughout the interaction.}
\label{fig:overall}
\vspace{-5mm}
\end{figure}

Multi-turn conversation is the predominant interaction paradigm between human and  \textit{Large Language Models} (LLMs)~\citep{li2025singleturnsurveymultiturninteractions,yi2025surveyrecentadvancesllmbased}. This conversational modality is essential for real-world applications~\citep{zhang2025surveymultiturninteractioncapabilities}, as it enables users to progressively refine initially underspecified intentions into concrete objectives~\citep{overcoming_miscalibrated_conversational, zheng2023judging}, engaging the model in a collaborative problem-solving process~\citep{chen-etal-2023-chatcot}. However, a fundamental mismatch exists between this prevalent use case and existing LLM alignment methodologies~\citep{laban2025llmslostmultiturnconversation, van-miltenburg-etal-2025-measure}. Prevailing alignment methods, \textit{Supervised Fine-Tuning} (SFT)~\citep{sft, wei-etal-2025-flexora, lester-etal-2021-power} and \textit{Reinforcement Learning from Human Feedback} (RLHF)~\citep{ouyang2022traininglanguagemodelsfollow, wei2025redit}, predominantly rely on single-turn data for both training~\citep{shu2024ferret} and evaluation~\citep{chang2023surveyevaluationlargelanguage}. This paradigm misalignment not only limits the potential of the model in complex interactions~\citep{irvine2023rewardingchatbotsrealworldengagement, hendrycks2021what}, but also creates a significant gap between its benchmark performance and its practical utility~\citep{shinn2023reflexion, wu2024mathchat}. Consequently, while the combination of SFT for imparting extensive knowledge~\citep{chu2025sft, wei2025paftpromptagnosticfinetuning} and RLHF for aligning with human preferences~\citep{rafailov2023direct, meng2024simpo} endows models with strong single-turn capabilities~\citep{zeng2024evaluating}, these models often exhibit a pronounced degradation in performance during multi-turn interactions~\citep{wang2024mint}. In fact, previous work has highlighted that such models often perform poorly in multi-turn scenarios, resulting in diminished capabilities and increased instability~\citep{laban2025llmslostmultiturnconversation}. While multi-turn training strategies have been explored~\citep{shi2025wildfeedbackaligningllmsinsitu, qu2024recursive, chen2025learning}, they are frequently hindered by the costs of collecting high-quality data, training on long context sequences~\citep{li2025singleturnsurveymultiturninteractions}, and the risk of behavioral collapse caused by training-inference distribution mismatch~\citep{ren2025learning}.

To address these challenges, we propose a new paradigm: \textit{\underline{T}est-\underline{T}ime \underline{P}olicy \underline{A}daptation for \underline{M}ulti-Turn Interactions} (\parad{}), shifting the existing static training paradigm to a flexible test-time adaption one. Specifically, it requires a model trained on single-turn data to perform effective online policy adaptation during multi-turn reasoning. This paradigm utilizes conversational user feedback as a reward signal to refine its policy and align with the underlying user intent, as illustrated in Figure \ref{fig:overall}. Importantly, this adaptation must be computationally lightweight to remain imperceptible without incurring unaffordable inference latency or GPU memory overhead. Under this paradigm, the model dynamically instantiates a user-specific policy for each conversational context, enhancing multi-turn interaction reliability. While target preferences naturally vary across tasks—such as strict correctness for math versus personal styles in others—the ultimate objective remains universally consistent: generating a response the user fully accepts. 

Existing methodologies~\citep{shani2024multiturn} tackle adaptation but lack the lightweight, feedback-driven updates of \parad{}. Specifically: (1) \textit{Prompt Engineering}~\citep{hu2024-localized, chen-etal-2023-chatcot, shinn2023reflexion} alters behavior via context rather than direct parameter updates. (2) \textit{Retrieval-Augmented Generation (RAG)}~\citep{gao2024retrievalaugmentedgenerationlargelanguage, RAG} increases inference overhead by lengthening context and heavily relies on external database quality. (3) \textit{Model Editing (ME)}~\citep{fang2025alphaedit, yao-etal-2023-editing} updates parameters to internalize facts, which differs from adapting to conversational preferences. (4) Existing \textit{test-time methods}~\citep{li2025testtime, zuo2025ttrltesttimereinforcementlearning, hu2025slotsamplespecificlanguagemodel, liu2023design} target single-turn tasks and incur significant computational latency through extensive sampling. See Appendix~\ref{sec:related_work} for details.

To bridge this gap, we introduce \textit{Optimum-\underline{R}eferenced \underline{O}ne-\underline{S}tep \underline{A}daptation} (\ours{}), a lightweight online adaptation algorithm that operationalizes our proposed paradigm \parad{}. The core principle of \ours{} is to leverage user feedback to analytically compute an estimate of the optimal policy and then steer the model towards this target in a single, efficient update step. This approach avoids costly iterative optimization, enabling principled in-conversation self-correction with minimal computational overhead. Our main contributions are summarized as follows:
\begin{itemize}[topsep=0pt,leftmargin=6mm,itemsep=0pt]
\item We demonstrate that current LLMs underperform in multi-turn interactions and propose \parad{} paradigm to address this issue (Section ~\ref{sec:motivation}).
\item We propose \ours{}, a practical algorithm to implement this paradigm, which updates model parameters and align user preferences quickly during multi-turn interactions (Section ~\ref{sec:method}).
\item We establish a solid theory for \ours{}, ensuring that its gap with user preferences narrows as the number of interaction turns increases (Section ~\ref{sec:theory}).
\item We conduct extensive experiments on multiple challenging datasets. Our results show that \ours{} outperforms baseline methods in both effectiveness and efficiency (Section ~\ref{sec:results}).
\end{itemize}

\begin{figure*}[t]
    \centering
    \includegraphics[width=1.0\textwidth]{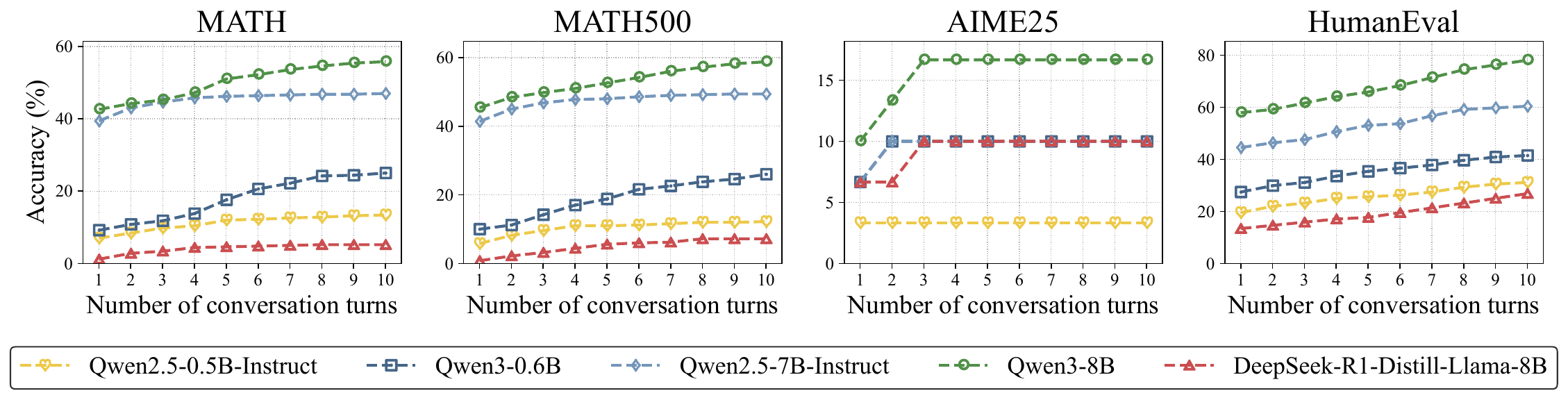} 
    \vspace{-5mm}
    \caption{
       LLM accuracy after 10 rounds of interaction with humans. Although LLM accuracy shows a continuous and gradual improvement, this prompt-based correction process is inefficient.
    }
    \label{fig:accuracy_line_chart}
    \vspace{-5mm}
\end{figure*}

\section{The \parad{} Paradigm}
\label{sec:motivation}

\begin{wrapfigure}{r}{0.46\textwidth} 
    \vspace{-4mm} 
    \centering
    \includegraphics[width=0.43\textwidth]{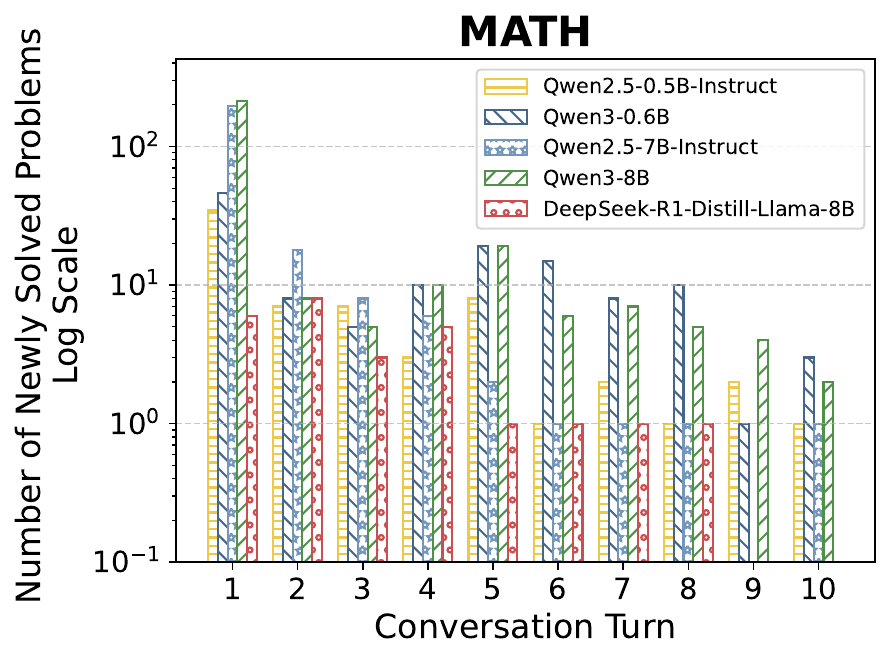} 
    \vspace{-5mm}
    \caption{
        Number of \textit{newly} solved problems per turn on the \texttt{MATH} dataset. 
    }
    \label{fig:newly_solved_bar_chart}
    \vspace{-3mm}
\end{wrapfigure}

The performance of LLMs often degrades in multi-turn interactions, because their alignment on static, single-turn datasets creates a paradigm mismatch that hinders their ability to adapt to user feedback or correct initial errors~\citep{laban2025llmslostmultiturnconversation}. To show this inefficiency, we empirically evaluated several LLMs on reasoning tasks. We first plot the cumulative accuracy over 10 conversational turns where human-like prompts were provided after each incorrect attempt. The results in Figure~\ref{fig:accuracy_line_chart} show that while multi-turn interaction gradually improves accuracy, the process exhibits sharply diminishing returns. To diagnose this, Figure~\ref{fig:newly_solved_bar_chart} plots the number of \textit{newly} solved problems at each conversational turn on the \texttt{MATH} dataset. The data reveal that the vast majority of problems are solved on the first attempt, with very few successful corrections in subsequent turns. This demonstrates that current models treat user interactions as passive context rather than as active signals for policy correction, highlighting a critical gap in their ability to perform efficient test-time adaptation.

To address this gap, we propose a new paradigm: \textit{test-time policy adaptation for multi-turn interactions} (\parad{}). As summarized in Table~\ref{tab:paradigm_comparison}, \parad{} resolves the trade-off in traditional approaches. While prompting lacks real-time adaptability and multi-turn training is costly, yielding a static policy, \parad{} synthesizes the benefits of both. It operates during inference with zero additional offline training cost but, through online parameter modification, achieves high, policy-level adaptability that is more direct than prompting and more flexible than offline training. Notably, this paradigm shifts model alignment from a static, offline training stage to a dynamic, online inference process. More specifically, it requires methods to update the model's policy in real-time by directly leveraging rich feedback signals from live conversations. We formally define \parad{} below:
\begin{tcolorbox}[colback=black!4!white,colframe=black!80!white,title=\textbf{Paradigm: Test-Time Policy Adaptation for Multi-Turn Interactions (\parad{})},breakable,left=2mm, right=2mm, top=2mm, bottom=2mm]
{\fontfamily{qpl}\selectfont
Let an inference-time multi-turn interaction be a sequence of interactions indexed by turn $k \in \{1, \dots\}$. At the beginning of turn $k$, the language model is defined by a policy $\pi_{\theta_{k-1}}$ with parameters $\theta_{k-1}$. The paradigm proceeds as follows:
\begin{enumerate}[leftmargin=*, topsep=2pt, itemsep=1pt]
    \item \textbf{Generation}: The model generates a response $\mathbf{y}_{k} \sim \pi_{\theta_{k-1}}(\cdot|\mathbf{x})$ given the conversational context $\mathbf{x}$.
    \item \textbf{Feedback}: The subsequent interaction of user provides feedback, which is mapped to a scalar reward $r_k$ indicating task success (i.e., $r_k=+1$) or failure (i.e., $r_k=-1$).
    \item \textbf{Adaptation}: If the task succeeds (i.e., $r_k=+1$), the multi-turn interaction is finished. Otherwise, an \textbf{effective and efficient} online adaptation function $\mathcal{A}$ updates the model parameters at inference time based on this failure feedback (i.e., $r_k=-1$) such that the model is more likely to succeed in the next turn:
    $$
        \theta_k = \mathcal{A}(\theta_{k-1}, r_k, \mathbf{y}_{k}; \mathbf{x}) = \theta_{k-1} + \Delta\theta_k \ .
    $$
\end{enumerate}
}
\end{tcolorbox}

\begin{table*}[ht]
\centering
\vspace{-2mm}
\caption{Conceptual comparison of paradigms for improving multi-turn LLM performance.}
\label{tab:paradigm_comparison}
\resizebox{\textwidth}{!}{
\begin{tabular}{lccc}
\toprule
\textbf{Feature} & \textbf{Prompt-based Methods} & \textbf{Multi-turn Data Training} & \textbf{T²PAM (Ours)} \\
\midrule
\textbf{Intervention Timing} & During inference & During training & During inference \\
\textbf{Operating Mode} & Reactive & \textbf{Proactive} (at training) & \textbf{Proactive} (at inference) \\
\textbf{Additional Offline Training Cost} & \textbf{Zero} (uses single-turn model) & High & \textbf{Zero} (uses single-turn model) \\
\textbf{Inference Cost} & Low (long context) & \textbf{Near-zero} & Low (preference alignment) \\
\textbf{Parameter Modification} & No & Yes (offline) & Yes (online) \\
\textbf{Real-time Adaptability} & Low (context-dependent) & Low (static policy) & \textbf{High} (policy-level) \\
\bottomrule
\end{tabular}
}
\end{table*}

\section{Optimum-Referenced One-Step Adaptation}
\label{sec:method}
To solve the paradigm we proposed above, we develop the \textit{Optimum-Referenced One-Step Adaptation} (\ours{}) approach (Algorithm~\ref{alg:rosa}), which enables effective and efficient online adaptation of a language model policy in direct response to real-time user feedback during multi-turn interactions. The core principle is to guide the model parameters towards a theoretical optimum in a single, efficient update step, avoiding iterative gradient-based optimization. This approach first defines the \textit{Reinforcement Learning from Human Feedback} (RLHF) objective (Section \ref{sec:rlhf_objective}) to maximize reward with KL regularization. It then leverages a closed-form analytical solution to directly identify the optimal policy (Section \ref{sec:optimal_policy}), applying exponential re-weighting to observed responses for practical one-step updates. Finally, parameter updates are efficiently computed via linearized optimization using the \textit{Conjugate Gradient} algorithm (Section \ref{sec:parameter_update}).

\subsection{The RLHF Objective for Turn-Wise Adaptation}
\label{sec:rlhf_objective}
We propose to solve the \parad{} paradigm above using \textit{Reinforcement Learning from Human Feedback} (RLHF) techniques~\citep{ouyang2022traininglanguagemodelsfollow}. In this approach, we learn from a reward signal $r(\mathbf{x}, \mathbf{y})$ that reflects human preference given the context $\mathbf{x}$ and the response $\mathbf{y}$. Specifically, we model this feedback as a binary signal where $r(\mathbf{x}, \mathbf{y}) \in \{-1, +1\}$ corresponds to negative and positive feedback, respectively. The objective is to find an updated policy $\pi_\theta$ that maximizes the expected reward while penalizing significant divergence from the policy of the previous turn $\pi_{\theta_{k-1}}$ for stable and controlled updates. The deviation is measured by the \textit{Kullback-Leibler} (KL) divergence. This leads to the following turn-wise optimization objective for turn $k$:
\begin{equation}
\label{eq:rlhf_objective}
\max_{\pi_{\theta}} \quad \mathbb{E}_{\mathbf{y} \sim \pi_{\theta}(\cdot|\mathbf{x})} \left[ r(\mathbf{x}, \mathbf{y}) \right] - \beta D_{\text{KL}} \left( \pi_\theta(\cdot|\mathbf{x}) \,\|\, \pi_{\theta_{k-1}}(\cdot|\mathbf{x}) \right) 
\end{equation}

where $\beta > 0$ is a coefficient that controls the strength of the KL regularization.

\subsection{From Theoretical Optimum to a Practical One-Step Update}
\label{sec:optimal_policy}

While the objective presented in Equation~\eqref{eq:rlhf_objective} is conventionally optimized using iterative gradient-based methods \citep{sgd,kingma2017adammethodstochasticoptimization}, such approaches are often characterized by their computational intensity and slow convergence, rendering them impractical for real-time online adaptation scenarios. Our methodology circumvents this inefficiency by leveraging a critical insight: this specific optimization problem admits a well-established closed-form analytical solution~\citep{rafailov2023direct}. Rather than relying on incremental approximations, we can directly ascertain the optimal policy. This foundational result is formalized in Theorem~\ref{thm:optimal_policy} (proof in Appendix~\ref{sec:optima_policy}).
\begin{tcolorbox}[colback=royalblue(web)!5!white,colframe=royalblue(web)!70!white, left=0.5mm, right=1mm, top=1mm, bottom=1mm]
\begin{theorem}[Closed-Form Optimal Policy]
\label{thm:optimal_policy}
\textit{\fontfamily{ppl}\selectfont
Let $Z_k(\mathbf{x}) = \sum_{\mathbf{y}' \in \mathcal{Y}} \pi_{\theta_{k-1}}(\mathbf{y}'|\mathbf{x}) \exp\left(\frac{1}{\beta} r(\mathbf{x}, \mathbf{y}')\right)$ be the partition function over the entire response space $\mathcal{Y}$, the policy $\pi^*_{\theta_k}$ that maximizes the turn-wise RLHF objective in Equation~\eqref{eq:rlhf_objective} is given by:
\begin{equation}
\label{eq:optimal_policy}
\pi^*_{\theta_k}(\mathbf{y}|\mathbf{x}) = \frac{1}{Z_k(\mathbf{x})} \pi_{\theta_{k-1}}(\mathbf{y}|\mathbf{x}) \exp\left(\frac{1}{\beta} r(\mathbf{x}, \mathbf{y})\right) \ .
\end{equation}}
\end{theorem}
\end{tcolorbox}

Theorem~\ref{thm:optimal_policy} demonstrates that the optimal policy is a re-weighted version of the reference policy, where the probability of a given response is exponentially modulated by its associated reward. In practical applications, feedback is typically received for only a \textit{single} generated response, $\mathbf{y}_k$, often corresponding to a negative reward ($r_k = -1$) for an incorrect output. This constraint necessitates the construction of an update target utilizing solely the observed data point $(\mathbf{x}, \mathbf{y}_k, r_k)$. We achieve this by applying the exponential re-weighting derived from the optimal policy in Equation~\eqref{eq:optimal_policy} exclusively to the observed response, thereby yielding a practical target value (derivation in Appendix~\ref{sec:derivation_practica_target}):
\begin{equation}
\label{eq:practical_target}
\tilde{\pi}^*_{\theta_k}(\mathbf{y}|\mathbf{x}) =
\begin{cases}
\frac{1}{Z_k(\mathbf{x})}\pi_{\theta_{k-1}}(\mathbf{y}|\mathbf{x}) \exp\!\left(\tfrac{1}{\beta} r_k\right), & \text{if } \mathbf{y} = \mathbf{y}_k, \\[1.2ex]
\frac{1}{Z_k(\mathbf{x})}\pi_{\theta_{k-1}}(\mathbf{y}|\mathbf{x}), & \text{if } \mathbf{y} \neq \mathbf{y}_k \:.
\end{cases}
\end{equation}
where $Z_k(\mathbf{x}) = 1 - \left(1 - \exp\!\left(\tfrac{1}{\beta} r_k\right)\right)\pi_{\theta_{k-1}}(\mathbf{y}_k|\mathbf{x})$.
This formulation provides a direct learning signal for a one-step parameter update. For an incorrect response with reward $r_k=-1$, the target probability is scaled down relative to the current policy, effectively instructing the model to diminish the likelihood of generating that specific erroneous output in the future. This approach transforms an otherwise intractable global optimization problem into a targeted, sample-wise correction, forming the fundamental basis for our efficient adaptation mechanism.

\vspace{-1mm}
\begin{algorithm}[t!]
\caption{Optimum-Referenced One-Step Adaptation (\ours{})}
\label{alg:rosa}
\begin{algorithmic}[1]
\STATE \textbf{Input:} Initial model parameters $\theta_0$, hyperparameter $\beta$.
\STATE $k \leftarrow 1$
\WHILE{\textbf{true}} 
    \STATE \COMMENT{Step 1: Generate response and receive feedback}
    \STATE Generate response $\mathbf{y}_k \sim \pi_{\theta_{k-1}}(\cdot|\mathbf{x})$.
    \STATE Receive reward $r_k$ based on user feedback.
    \IF{$r_k = +1$}
        \STATE \textbf{Terminate }\COMMENT{Stop immediately on success signal}
    \ENDIF
    \STATE \COMMENT{Step 2: Construct the practical online target (Section~\ref{sec:optimal_policy})}
    \STATE Compute target value $\tilde{\pi}^*_{\theta_k}(\mathbf{y}_k|\mathbf{x}) = \frac{1}{Z_k(\mathbf{x})}\pi_{\theta_{k-1}}(\mathbf{y}_k|\mathbf{x}) \exp\left(\frac{1}{\beta} r_k\right)$.
    
    \STATE \COMMENT{Step 3: Compute parameter update via linearized optimization (Section~\ref{sec:parameter_update})}
    \STATE Define residual $\mathbf{d}_k = \tilde{\pi}^*_{\theta_k} - \pi_{\theta_{k-1}}$.
    \STATE Solve $(\mathbf{J}_k^\top \mathbf{J}_k) \Delta\theta_k = \mathbf{J}_k^\top \mathbf{d}_k$ for $\Delta\theta_k$ using \textit{Conjugate Gradient} method.
    
    \STATE \COMMENT{Step 4: Update model parameters}
    \STATE Update parameters: $\theta_k \leftarrow \theta_{k-1} + \Delta\theta_k$.
\ENDWHILE
\end{algorithmic}
\end{algorithm}
\vspace{-1mm}

\subsection{Efficient Parameter Update via Linearized Optimization}
\label{sec:parameter_update}
With a practical target policy $\tilde{\pi}^*_{\theta_k}$ established, the subsequent step involves computing the parameter update $\Delta\theta_k$ that adjusts the current policy $\pi_{\theta_{k-1}}$ towards this target. This is accomplished through linearized optimization. This linearization is chosen for its computational ease and efficiency, allowing for rapid online adaptation without the prohibitive costs of higher-order optimization methods, as demonstrated in our efficiency analysis in Section~\ref{sec:exp-efficiency}. Initially, the policy function is approximated using a first-order Taylor expansion around the current parameters $\theta_{k-1}$:
\begin{equation}
\label{eq:taylor_expansion}
\pi_{\theta_{k-1} + \Delta\theta_k}(\mathbf{y}_k|\mathbf{x}) \approx \pi_{\theta_{k-1}}(\mathbf{y}_k|\mathbf{x}) + \nabla_\theta \pi_{\theta_{k-1}}(\mathbf{y}_k|\mathbf{x})^\top \Delta\theta_k \:.
\end{equation}
Our objective is to determine $\Delta\theta_k$ such that the updated policy $\pi_{\theta_{k-1} + \Delta\theta_k}$ closely matches our target $\tilde{\pi}^*_{\theta_k}$. For the single data point $(\mathbf{x}, \mathbf{y}_k)$, this yields a linear system of equations:
\begin{equation}
\label{eq:linear_system}
\mathbf{J}_k \Delta\theta_k \approx \tilde{\pi}^*_{\theta_k}(\mathbf{y}_k|\mathbf{x}) - \pi_{\theta_{k-1}}(\mathbf{y}_k|\mathbf{x}) \:.
\end{equation}
where $\mathbf{J}_k = \nabla_\theta \pi_{\theta_{k-1}}(\mathbf{y}_k|\mathbf{x})^\top$ represents the Jacobian of the policy output with respect to the model parameters. To obtain a stable, least-squares solution for $\Delta\theta_k$, we solve the following equations:
\begin{equation}
\label{eq:normal_equations}
(\mathbf{J}_k^\top \mathbf{J}_k) \Delta\theta_k = \mathbf{J}_k^\top \left( \tilde{\pi}^*_{\theta_k}(\mathbf{y}_k|\mathbf{x}) - \pi_{\theta_{k-1}}(\mathbf{y}_k|\mathbf{x}) \right) \:.
\end{equation}
Explicitly forming the Hessian-approximating matrix $\mathbf{J}_k^\top \mathbf{J}_k$ is computationally prohibitive for models with a large number of parameters. As a consequence, we employ the \textit{Conjugate Gradient (CG) algorithm}~\citep{Atkinson1988}, an iterative solver that efficiently determines the solution to Equation~\eqref{eq:normal_equations} without materializing this matrix. This is critical for memory efficiency, as it avoids storing the full Hessian-like matrix, making our approach incur less GPU memory overhead, as shown in Appendix~\ref{sec:exp-efficiency}. The CG algorithm only requires the computation of the matrix-vector product $(\mathbf{J}_k^\top \mathbf{J}_k)\mathbf{p}$ for an arbitrary vector $\mathbf{p}$. This computation is performed in a matrix-free manner by efficiently chaining two operations using automatic differentiation: a Jacobian-vector product (JVP) to compute $\mathbf{J}_k\mathbf{p}$, followed by a vector-Jacobian product (VJP) to compute $\mathbf{J}_k^\top(\mathbf{J}_k\mathbf{p})$. Positive-feedback updates and learning from successful trajectories are evaluated in Appendix~\ref{sec:successful_trajectories}.

Once the optimal $\Delta\theta_k$ is computed via CG method, the model parameters are updated in one step:
\begin{equation}
\label{eq:final_update}
\theta_k \leftarrow \theta_{k-1} + \Delta\theta_k \:.
\end{equation}
This entire procedure, encompassing feedback reception and parameter update computation, constitutes one complete cycle of \ours{}, as comprehensively detailed in Algorithm~\ref{alg:rosa}.

\section{Theoretical Results}
\label{sec:theory}

Having established the mechanics of \ours{}, we now provide its theoretical underpinnings. This section demonstrates that our \ours{} is not merely an effective heuristic but a principled algorithm with formal guarantees. Our analysis unfolds in three stages: we first prove that each corrective step is guaranteed to be productive (Section~\ref{sec:theory-1}), then bound the accumulated gains over corrective updates (Section~\ref{sec:theory-2}), and finally, provide a unified bound that accounts for the practical approximation errors inherent in our efficient update step (Section~\ref{sec:theory-3}).

Of note, a central aspect of our theoretical analysis revolves around the \textit{Kullback-Leibler} (KL) divergence, specifically $D_{\text{KL}}(\pi_{\text{user}}^* \| \tilde{\pi}^*_{\theta_k})$. This metric quantifies the dissimilarity between the underlying user optimal policy $\pi_{\text{user}}^*$ (representing the true preferences from the user and the ideal way to solve the task) and our adapted policy $\tilde{\pi}^*_{\theta_k}$. Minimizing this divergence is crucial because it directly implies that the generated responses from a model are becoming increasingly aligned with what the user desires and expects. When the model policy closely mirrors the user optimal policy, it is inherently more likely to produce correct and satisfactory outputs, thereby increasing the probability of task success and reducing the number of interaction turns required to achieve user intent. 

For this analysis, define the ideal user policy as any distribution supported entirely on the acceptable-response set $\mathcal A(\mathbf x)$. Thus, $\pi_{\text{user}}^*(\mathbf y\mid\mathbf x)=0$ for unacceptable responses. The bounds below concern reliable negative feedback identifying such responses, at a fixed $\mathbf x$, with finite initial KL divergence; they do not require a unique distribution over acceptable answers.

\subsection{Monotonic Error Reduction}\label{sec:theory-1}

Our first key result establishes that the adaptation mechanism in \ours{} is \textit{provably productive}. Each corrective feedback, the resulting update is guaranteed to reduce the KL divergence between the user policy and our estimated target policy, as formally shown in Theorem~\ref{thm:monotonic_error} (proof in Appendix~\ref{sec:monotonic_error_reduction}).
\begin{tcolorbox}[colback=royalblue(web)!3!white,colframe=royalblue(web)!90!white, left=0.5mm, right=1mm, top=1mm, bottom=1mm]
\begin{theorem}[Monotonic Error Reduction]
\label{thm:monotonic_error}
\textit{\fontfamily{ppl}\selectfont
Let $\pi_{\normalfont\text{user}}^*$ be the underlying user policy and $\tilde{\pi}^*_{\theta_k}$ be the practical target policy in Equation~\eqref{eq:practical_target} after receiving negative feedback $r_k=-1$ on an incorrect response $\mathbf{y}_k$ at turn $k$. Suppose $\pi_{\theta_k} = \tilde{\pi}^*_{\theta_k}$ by applying the exact policy update in \ours{}. Letting $C = 1 - \exp(-1/\beta) \in (0,1)$, the change in KL divergence from the previous turn is bounded as follows:
\begin{equation}
\label{eq:monotonic_error}
\normalfont
D_{\text{KL}}(\pi_{\text{user}}^* \| \tilde{\pi}^*_{\theta_k}) - D_{\text{KL}}(\pi_{\text{user}}^* \| \tilde{\pi}^*_{\theta_{k-1}}) \leq -C \cdot \pi_{\theta_{k-1}}(\mathbf{y}_k | \mathbf{x}) < 0 \ .
\end{equation}
}
\end{theorem}
\end{tcolorbox}
\noindent\textbf{Remark.} This theorem provides a powerful guarantee for the reliability of \ours{}. The upper bound of this reduction consists of two highly informative terms, both of which offer significant insights into the adaptation dynamics. First, the term $C$ acts as a scaling factor regulated by $\beta$, where a smaller $\beta$ yields a larger $C$ that leads to a faster reduction in error. While this implies rapid initial learning, it introduces a critical trade-off with approximation errors that we formally quantify in Section~\ref{sec:theory-3}. Second, the $\pi_{\theta_{k-1}}(\mathbf{y}_k|\mathbf{x})$ term reveals that the error reduction is strictly proportional to the model's prior confidence in its own mistake. This elegant property indicates that the most impactful learning signals come from correcting high-confidence errors instead of random, low-probability hallucinations. In summary, this result confirms that every piece of corrective negative feedback is guaranteed to be productive, which provides strong theoretical justification for the one-step adaptation design in \ours{} by avoiding the potential instability of iterative optimization.

\subsection{Cumulative Convergence Guarantee}\label{sec:theory-2}

While Theorem~\ref{thm:monotonic_error} guarantees improvement at each step, our second theorem extends this result to the entire multi-turn interaction, providing a bound on the cumulative error over reliably rejected responses in our Theorem \ref{thm:cumulative_error} (proof in Appendix~\ref{sec:cumulative_error_bound}).
\begin{tcolorbox}[colback=royalblue(web)!3!white,colframe=royalblue(web)!90!white, left=0.5mm, right=1mm, top=1mm, bottom=1mm]
\begin{theorem}[Cumulative Error Bound]
\label{thm:cumulative_error}
\textit{\fontfamily{ppl}\selectfont
Suppose $\pi_{\theta_k} = \tilde{\pi}^*_{\theta_k}$ by applying exact policy update in \ours{}, after $K$ reliably rejected-response updates, the KL divergence between the underlying user policy $\pi_{\text{user}}^*$ and the practical target policy $\tilde{\pi}^*_{\theta_K}$ in Equation~\eqref{eq:practical_target} is bounded as follows:
\begin{equation}
\label{eq:cumulative_error}
\normalfont
D_{\text{KL}}(\pi_{\text{user}}^* \| \tilde{\pi}^*_{\theta_K}) \leq D_{\text{KL}}(\pi_{\text{user}}^* \| \pi_{\theta_{0}}) - C \sum_{k=1}^K \pi_{\theta_{k-1}}(\mathbf{y}_{k} | \mathbf{x}) \ .
\end{equation}
}
\end{theorem}
\end{tcolorbox}
\noindent\textbf{Remark.} The benefits of exact corrective updates accumulate over time through the subtracted summation term. The bound describes progress toward the acceptable-response set; convergence to a particular distribution within that set requires additional conditions. Appendix~\ref{sec:cumulative_error_bound} gives the corresponding conditional sample-efficiency analysis under diverse failure modes.

\subsection{Unified Error Bound for the Adapted Policy}\label{sec:theory-3}
The previous theorems guarantee our \textit{target} policy improves. However, the \textit{final} policy, $\pi_{\theta_{k}}$, is subject to the approximation error from the first-order Taylor expansion used for our efficient update. The following unified theorem combines the guaranteed improvement from feedback with the accumulated error to provide a comprehensive bound on the true performance of \ours{} (proof in Appendix~\ref{sec:unified_error_bound}).
\begin{tcolorbox}[colback=royalblue(web)!3!white,colframe=royalblue(web)!90!white, left=0.5mm, right=1mm, top=1mm, bottom=1mm]
\begin{theorem}[Unified Convergence Bound]
\label{thm:unified_bound}
\textit{\fontfamily{ppl}\selectfont
Assume $\log\pi_\theta$ is Lipschitz-smooth with constant $L$ along each update, and let $\epsilon_k$ bound the weighted linearized target residual defined in Appendix~\ref{sec:unified_error_bound}. After $K$ reliably rejected-response updates in \ours{}, the divergence of the final adapted policy $\pi_{\theta_K}$ from the underlying user policy $\pi_{\text{user}}^*$ is bounded by:
\begin{equation}
\label{eq:unified_bound}
\normalfont
D_{\text{KL}}(\pi_{\text{user}}^* \| \pi_{\theta_K}) \le \underbrace{D_{\text{KL}}(\pi_{\text{user}}^* \| \pi_{\theta_0})\vphantom{-C\sum_{k=1}^K\pi_{\theta_{k-1}}(\mathbf{y}_k | \mathbf{x})}}_{\text{Initial Error}} \underbrace{-C\sum_{k=1}^K\pi_{\theta_{k-1}}(\mathbf{y}_k | \mathbf{x})}_{\text{Improvement}} + \underbrace{\frac{L}{2} \sum_{k=1}^{K} \|\Delta\theta_k\|^2_2 + \sum_{k=1}^{K}\epsilon_k}_{\text{Approx. Error}} \ .
\end{equation}
}
\end{theorem}
\end{tcolorbox}
\noindent\textbf{Remark.} This unified bound rigorously quantifies the inherent trade-off in online policy adaptation. Each turn reduces the KL divergence from the underlying user optimal policy by a reward-driven term $C\pi_{\theta_{k-1}}(\mathbf{y}_k|\mathbf{x})$, while incurring an approximation error $\frac{L}{2}\lVert\Delta\theta_k\rVert_2^2$ due to linearization, together with a linearized target-fitting error $\epsilon_k$. Convergence requires the net progress per turn to remain positive. This balance is affected by two factors. Firstly, the approximation error is controlled because the update magnitude $\Delta\theta_k$ is constrained by $\pi_{\theta_{k-1}}(\mathbf{y}_k|\mathbf{x})$ according to Equation~\eqref{eq:practical_target}, which ensures the guaranteed improvement can effectively outweigh the approximation cost. Additionally, our algorithm effectively minimizes the fitting error $\epsilon_k$ by employing the Conjugate Gradient (CG) solver. Secondly, the regularization coefficient $\beta$ modulates this trade-off by directly determining the constant $C$ and indirectly affecting the update norm $\|\Delta\theta_k\|_2$. A smaller $\beta$ yields a larger $C$ to accelerate learning, but risks amplifying the approximation error, whereas a larger $\beta$ stabilizes updates at the cost of slower progress. This interplay explains the two-phase behavior observed in practice, which features rapid initial corrections followed by stable, fine-grained refinements as detailed in Appendix~\ref{sec:ablation_beta}. The theorem therefore serves as both a robust theoretical guarantee and a practical design guide for balancing adaptation speed and stability.

\begin{table*}[ht]
\centering
\caption{
    Accuracy (\%) of \ours{} across diverse domains. \textbf{Baseline} is standard multi-turn interaction. (+A+B) denotes update location (LM: LM Head, HS: Hidden States) and reward type (R: rule-based, M: model-based). \textcolor{darkred}{Red} values show absolute improvements. Further details on parameter updates and reward models are provided in Appendix~\ref{sec:parameter} and~\ref{sec:reward}, respectively.
}
\label{tab:main_result}
\resizebox{\textwidth}{!}{
\begin{tabular}{cc cc  cc  cc  c}
\toprule
& & \multicolumn{2}{c}{\textbf{Mathematical Reasoning}} & \multicolumn{2}{c}{\textbf{General Reasoning}} & \multicolumn{2}{c}{\textbf{Multilingual Reasoning}} & \textbf{Code Gen.}\\
\cmidrule(lr){3-4} \cmidrule(lr){5-6} \cmidrule(lr){7-8} \cmidrule(lr){9-9}
\textbf{Model} & \textbf{Method} & \textbf{MATH} & \textbf{MATH-500} & \textbf{MMLU-R} & \textbf{SuperGPQA} & \textbf{MT-AIME24} & \textbf{MT-MATH100} & \textbf{HumanEval} \\

\midrule
\multirow{4}{*}{\shortstack{Qwen2.5-0.5B \\ -Instruct}} & Baseline & 13.40 & 12.20 & 7.27 & 1.90 & 3.48 & 15.40 & 31.09  \\
& \ours{} (+LM + R) & \textbf{30.40} \gain{17.00} & 28.00 \gain{15.80} & 9.07 \gain{1.80} & 5.63 \gain{3.73} & 3.67 \gain{0.19} & 22.80 \gain{7.40} & 37.19 \gain{6.10} \\ 
& \ours{} (+HS + R) & 25.40 \gain{12.00} & 25.00 \gain{12.80} & 11.00 \gain{3.73} & 5.00 \gain{3.10} & 4.90 \gain{1.42} & 20.90 \gain{5.50} & 37.27 \gain{6.18} \\ 
& \ours{} (+LM + M) & 27.00 \gain{13.60} & \textbf{28.40} \gain{16.20} & \textbf{13.72} \gain{6.45} & \textbf{6.57} \gain{4.67} & \textbf{6.13} \gain{2.65} & \textbf{25.20} \gain{9.80} & \textbf{39.37} \gain{8.28} \\

\midrule
\multirow{4}{*}{Qwen3-0.6B} & Baseline & 25.00 & 26.00 & 18.60 & 4.20 & 4.80 & 31.30 & 41.46  \\
& \ours{} (+LM + R) & 50.20 \gain{25.20} & 51.60 \gain{25.60} & 33.40 \gain{14.80} & 9.13 \gain{4.93} & 7.58 \gain{2.78} & 56.60 \gain{25.30} & 45.73 \gain{4.27} \\ 
& \ours{} (+HS + R) & 50.80 \gain{25.80} & 50.60 \gain{24.60} & 36.00 \gain{17.40} & 9.70 \gain{5.50} & 7.90 \gain{3.10} & 51.90 \gain{20.60} & 47.27 \gain{5.81} \\ 
& \ours{} (+LM + M) & \textbf{52.20} \gain{27.20} & \textbf{54.60} \gain{28.60} & \textbf{40.68} \gain{22.08} & \textbf{15.73} \gain{11.53} & \textbf{9.43} \gain{4.63} & \textbf{59.40} \gain{28.10} & \textbf{49.37} \gain{7.91} \\

\midrule
\multirow{4}{*}{\shortstack{Qwen2.5-7B \\ -Instruct}} & Baseline & 47.00 & 49.40 & 45.36 & 19.31 & 19.24 & 60.34 & 57.92 \\
& \ours{} (+LM + R) & 63.40 \gain{16.40} & 62.40 \gain{13.00} & 62.17 \gain{16.81} & 37.26 \gain{17.95} & 27.14 \gain{7.90} & 73.16 \gain{12.82} & 63.41 \gain{5.49} \\ 
& \ours{} (+HS + R) & 64.40 \gain{17.40} & 63.40 \gain{14.00} & 67.31 \gain{21.95} & 36.27 \gain{16.96} & 26.75 \gain{7.51} & 72.27 \gain{11.93} & 64.24 \gain{6.32} \\ 
& \ours{} (+LM + M) & \textbf{65.20} \gain{18.20} & \textbf{65.60} \gain{16.20} & \textbf{68.47} \gain{23.11} & \textbf{40.67} \gain{21.36} & \textbf{30.21} \gain{10.97} & \textbf{75.13} \gain{14.79} & \textbf{67.36} \gain{9.44} \\

\midrule
\multirow{4}{*}{Qwen3-8B} & Baseline & 55.80 & 58.80 & 51.35 & 27.61 & 30.37 & 74.74 & 78.04  \\
& \ours{} (+LM + R) & 65.80 \gain{10.00} & \textbf{72.80} \gain{14.00} & 67.27 \gain{15.92} & 36.11 \gain{8.50} & 40.16 \gain{9.79} & 85.16 \gain{10.42} & 81.71 \gain{3.67} \\ 
& \ours{} (+HS + R) & 65.80 \gain{10.00} & 66.20 \gain{7.40} & 68.37 \gain{17.02} & 37.73 \gain{10.12} & 42.27 \gain{11.90} & 86.93 \gain{12.19} & 82.37 \gain{4.33} \\ 
& \ours{} (+LM + M) & \textbf{67.40} \gain{11.60} & 68.40 \gain{9.60} & \textbf{70.36} \gain{19.01} & \textbf{40.34} \gain{12.73} & \textbf{43.93} \gain{13.56} & \textbf{88.37} \gain{13.63} & \textbf{83.65} \gain{5.61} \\

\midrule
\multirow{4}{*}{\shortstack{DeepSeek-R1\\-Distill-Llama-8B}} & Baseline & 5.20 & 7.20 & 30.46 & 10.37 & 4.73 & 17.35 & 25.00  \\
& \ours{} (+LM + R) & 7.80 \gain{2.60} & 18.40 \gain{11.20} & 41.14 \gain{10.68} & 20.49 \gain{10.12} & 6.13 \gain{1.40} & 21.17 \gain{3.82} & 39.03 \gain{14.03} \\ 
& \ours{} (+HS + R) & 8.40 \gain{3.20} & 18.20 \gain{11.00} & 42.18 \gain{11.72} & 21.34 \gain{10.97} & 7.27 \gain{2.54} & 23.85 \gain{6.50} & 38.37 \gain{13.37} \\ 
& \ours{} (+LM + M) & \textbf{8.60} \gain{3.40} & \textbf{20.80} \gain{13.60} & \textbf{45.79} \gain{15.33} & \textbf{24.97} \gain{14.60} & \textbf{8.19} \gain{3.46} & \textbf{24.67} \gain{7.32} & \textbf{39.26} \gain{14.26} \\

\bottomrule
\end{tabular}
}
\end{table*}

\begin{figure*}[t!]
    \centering 
    \includegraphics[width=0.95\textwidth]{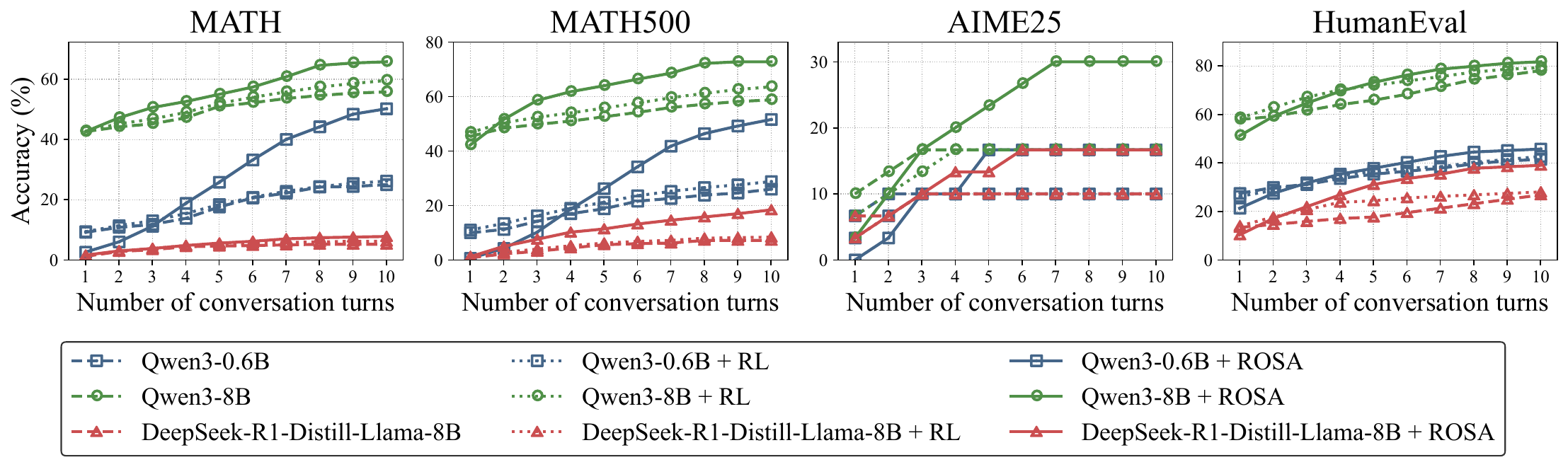} 
    \vspace{-3mm}
    \caption{
        Comparison of \textit{baselines}, \textit{RL} (Appendix \ref{sec:ablation_strategy}), and \textit{\ours{}} across datasets. Unlike the slow improvement in Figure~\ref{fig:accuracy_line_chart}, \ours{} yields steeper learning curves and higher absolute accuracy, demonstrating its superior efficiency in online error correction during multi-turn interactions.
    }
    \label{fig:result_1}
    \vspace{-5mm}
\end{figure*}

\section{Empirical Results}
\label{sec:results}

We validate the effectiveness and efficiency of \ours{} in dynamic multi-turn settings, demonstrating state-of-the-art performance across diverse tasks (Section~\ref{sec:sota_performance}) and robust online error correction (Section~\ref{sec:error_correction}). Setup details (datasets, baselines, evaluation protocol, feedback robustness) and ablations (optimization strategy, $\beta$) are detailed in Appendices~\ref{sec:setup}, \ref{sec:eval_protocol}, \ref{sec:feedback_robustness}, and \ref{sec:ablation_studies}.

\noindent\textbf{Evaluation protocol.} The primary verifiable-task evaluation allows up to 10 turns, stops at a verified correct answer, and resets adaptation between problems. Baseline and \ours{} use the same textual-feedback protocol; \ours{} additionally uses binary or model-based scalar rewards. Accuracy measures problems solved within the budget, and correction uplift measures recovery from first-turn failures. Full protocols and reward calibration are in Appendices~\ref{sec:eval_protocol} and~\ref{sec:reward_calibration}.

\subsection{Effectiveness and Generalizability Across Task Domains}
\label{sec:sota_performance}

To validate that \ours{} effectively operationalizes \parad{}, we evaluate it across diverse settings. \parad{} demands dynamic, instance-specific adaptation via real-time feedback, and our results confirm \ours{} fulfills this across three key dimensions:

\textbf{Broad Applicability and Effectiveness.} \ours{} consistently outperforms standard multi-turn \textit{baselines} across mathematical, general, code, and multilingual domains (Table~\ref{tab:main_result}; Appendix~\ref{sec:datasets},~\ref{sec:additional_result}). Correcting errors online, it overcomes traditional stagnation, matching or exceeding expensive multi-turn training (Appendix~\ref{sec:comparison_with_training}) and advanced test-time compute baselines without massive latency (Appendix~\ref{sec:test_time_compute}), while scaling robustly to massive foundation models (Appendix~\ref{sec:scalability_large_models}).

\textbf{Architectural and Signal Flexibility.} Fulfilling \parad{}'s lightweight constraints, \ours{} achieves significant improvements updating either the \textit{LM Head} or \textit{Hidden States} (Appendix~\ref{sec:parameter}). It also adapts dynamically to varying feedback granularities: while dense \textit{model-based} rewards or Process Reward Models (PRM, Appendix~\ref{sec:prm_integration}; calibration in Appendix~\ref{sec:reward_calibration}) yield peak performance, even sparse \textit{rule-based} rewards deliver substantial gains, corroborating Theorem~\ref{thm:monotonic_error}.

\textbf{Universal Preference Alignment.} Crucially, \parad{} aims to satisfy diverse user preferences. Our task-agnostic evaluations (Appendix~\ref{sec:extended_eval}) show that \ours{} translates distinct feedback into precise policy alignments. It maximizes objective reasoning accuracy given correctness feedback, and seamlessly aligns with stylistic or subjective human preferences: on ALOE, the mean Alignment Level improves from 3.78 to 4.40 (Appendix~\ref{sec:aloe}, Figure~\ref{fig:aloe_alignment}); IFEval strict accuracy surges from 58.31\% to 70.26\%; and MT-Bench subjective alignment rises from 60.36 to 75.57 (Appendix~\ref{sec:open_ended_delayed}, Table~\ref{tab:open_ended}). Even in delayed-reward scenarios like UI Agents (Appendix~\ref{sec:open_ended_delayed}), \ours{} efficiently steers the model toward any user-defined acceptable response.

\vspace{-2mm}
\begin{table}[t]
\centering
\caption{Comparison of Correction Uplift (\%).}
\label{tab:correction_uplift}
\resizebox{0.8\columnwidth}{!}{
\begin{tabular}{ll cccc}
\toprule
\textbf{Model} & \textbf{Method} & \textbf{MATH} & \textbf{AIME25} & \textbf{MATH-500} & \textbf{HumanEval} \\
\midrule
\multirow{2}{*}{Qwen2.5-0.5B-Instruct} & Baseline & 6.88 & 0.00 & 6.79 & 14.39 \\
& \ours{} & \textbf{25.48} \gain{18.60} & \textbf{6.67} \gain{6.67} & \textbf{24.05} \gain{17.26} & \textbf{26.09} \gain{11.70} \\
\midrule
\multirow{2}{*}{Qwen3-0.6B} & Baseline & 17.40 & 3.57 & 17.78 & 19.33\\
& \ours{} & \textbf{48.87} \gain{31.47} & \textbf{16.67} \gain{13.10} & \textbf{51.31} \gain{33.53} & \textbf{31.01} \gain{11.68} \\
\midrule
\multirow{2}{*}{Qwen2.5-7B-Instruct} & Baseline & 12.54 & 3.57 & 13.65 & 28.57 \\
& \ours{} & \textbf{41.53} \gain{28.99} & \textbf{20.69} \gain{17.12} & \textbf{36.91} \gain{23.26} & \textbf{40.00} \gain{11.43} \\
\midrule
\multirow{2}{*}{Qwen3-8B} & Baseline & 23.00 & 7.41 & 24.54 & 47.83 \\
& \ours{} & \textbf{40.42} \gain{17.42} & \textbf{27.59} \gain{20.18} & \textbf{52.94} \gain{28.40} & \textbf{62.50} \gain{14.67} \\
\midrule
\multirow{2}{*}{\shortstack{DeepSeek-R1-Distill\\-Llama-8B}} & Baseline & 4.05 & 3.57 & 6.45 & 15.49 \\
& \ours{} & \textbf{6.30} \gain{2.25} & \textbf{13.79} \gain{10.22} & \textbf{17.41} \gain{10.96} & \textbf{31.97} \gain{16.48}\\
\bottomrule
\end{tabular}
}
\vspace{-2mm}
\end{table}

\subsection{Effectiveness in Online Error Correction}
\label{sec:error_correction}

A core claim of our work is that \ours{} fundamentally transforms LLMs from passive context-processors into active, continuous learners during multi-turn interactions. While \ours{} consistently matches or surpasses computationally heavy test-time scaling methods and expensive multi-turn training paradigms (Appendices~\ref{sec:test_time_compute} and \ref{sec:comparison_with_training}), our primary focus is dissecting the dynamics of its in-conversation self-correction. Traditional prompt-based models suffer from sharply diminishing returns because they rarely recover in subsequent turns after an initial failure, which leads to a severe performance plateau (Figure~\ref{fig:result_2}, Appendix~\ref{sec:per_turn_analysis}). \ours{} breaks this stagnation by translating failure feedback into targeted policy updates. As evidenced by the \textit{Correction Uplift} metric (Table~\ref{tab:correction_uplift}), our method sustains a significantly higher volume of newly solved problems deep into the conversation. This continuous correction generates significantly steeper learning curves compared to the slow, plateauing trajectories of zero-shot baselines or direct RL optimization (Figure~\ref{fig:result_1}).

Crucially, this empirical self-correction capability aligns perfectly with our theoretical guarantees in Theorem~\ref{thm:cumulative_error}, demonstrating that the model progressively aligns with user preferences by effectively learning from consecutive failures. Because continuous online updates naturally raise concerns regarding operational feasibility and model degradation, we rigorously verify the robustness of our approach. Specifically, we demonstrate that \ours{} remains highly stable over extended interaction horizons of up to 20 turns (Appendix~\ref{sec:extended_horizons}) and strictly preserves general knowledge without catastrophic forgetting (Appendix~\ref{sec:knowledge_retention}). Furthermore, \ours{} achieves these dynamic capabilities with remarkable practical viability. It requires negligible GPU memory overhead and rapidly overtakes static baselines in time-to-accuracy efficiency (Appendix~\ref{sec:exp-efficiency}), with further insights provided through a detailed case study (Appendix~\ref{sec:case_study}) and an analysis of multi-sample trade-offs (Appendix~\ref{sec:multi_sample_feedback}).

\section{Conclusions and Limitations}
\label{sec:Conclusion}
We address LLM performance degradation in multi-turn dialogues by introducing the \parad{} paradigm and its practical algorithm, \ours{}. \ours{} enables efficient in-conversation self-correction via real-time parameter updates. While theoretically and empirically validated, \ours{} is currently less effective on tasks heavily bottlenecked by the pre-trained knowledge of model.

\newpage
\medskip

\subsection*{AI use statement}

We have not used generative AI tools for any of the tasks whose disclosure is
required---proposing or refining hypotheses, developing conceptual frameworks,
designing methodology or experiments, implementing methods, processing or
cleaning data, analysing or interpreting results, formulating or proving
mathematical claims, or translation---all of which were carried out entirely by
the authors. We did use generative AI tools for two tasks whose disclosure is
recommended: creating and refining the figures in this paper, and copy-editing
the manuscript to improve grammar and readability. All AI-assisted figures were
verified against the raw experimental outputs by the authors, and all
AI-suggested text edits were reviewed by an author and restricted to wording,
with no change to the technical content, claims, or conclusions. We have
reviewed all AI-assisted work and take responsibility for the final content of
this work, including text, claims, and artifacts produced with the aid of
generative AI.

\bibliography{iclr2027_conference}
\bibliographystyle{iclr2027_conference}

\newpage
\medskip

\appendix

\section{Related Work} \label{sec:related_work}

Research on improving the multi-turn capabilities of LLMs has largely proceeded along three main fronts: in-context learning, fine-tuning with multi-turn data, and reinforcement learning.

\paragraph{In-Context Learning and Prompting Strategies.}
A prominent line of work enhances multi-turn performance without modifying model parameters by leveraging the context window to guide the model's reasoning~\citep{ou-etal-2024-inductive, sun-etal-2024-parrot}. For instance, ChatCoT~\citep{chen-etal-2023-chatcot} models the chain-of-thought process as a multi-turn interaction to improve reasoning. Similarly, Reflexion~\citep{shinn2023reflexion} refines model behavior by converting environmental feedback into textual summaries, which are appended to the context for subsequent turns. MathChat~\citep{wu2024mathchat} extends this by introducing a user agent that can execute tools and inject the resulting feedback into the conversation. These approaches use textual context as the adaptation channel. Our method is complementary: it retains conversational feedback while using its scalar evaluation to construct an explicit policy target and update a small parameter subset.

\paragraph{Fine-Tuning with Multi-Turn Data.}
Another approach involves fine-tuning the model on datasets specifically designed to capture multi-turn dynamics~\citep{wang2024mint, zhou-etal-2024-enhancing}. For instance, WildChat~\citep{shi2025wildfeedbackaligningllmsinsitu} leverages live user feedback to automatically construct a preference dataset for subsequent fine-tuning. Addressing challenges within this domain, Codesteer~\citep{chen2025codesteer} identifies a ``gradient cancellation" issue, where gradients from early turns can interfere with those from later, more informative ones, and mitigates this by up-weighting the loss from the final turns of the interaction. However, a key limitation of such offline SFT approaches is their potential insufficiency in cultivating robust self-correcting behavior~\citep{li2025singleturnsurveymultiturninteractions, yi2025surveyrecentadvancesllmbased}. This challenge often stems from a distribution mismatch between the errors present in the training data and those produced by the model at inference time, as well as the risk of "behavioral collapse," where the model overfits to a narrow set of correction patterns.

\paragraph{Reinforcement Learning Approaches.}
Several methods employ reinforcement learning (RL) to teach models to self-improve over multiple rounds~\citep{zhou2025sweetrltrainingmultiturnllm, liu2025spiralselfplayzerosumgames}. For instance, RISE~\citep{qu2024recursive} utilizes multi-round offline RL with reward supervision, applying a majority vote over candidate outputs at inference time. SCoRe~\citep{kumar2024traininglanguagemodelsselfcorrect} adopts a two-stage process, first teaching the model to self-correct and then maximizing this capability via RL. Other works have explored multi-round group preference optimization by decomposing conversations into single-turn problems~\citep{shi-etal-2024-direct, yang2025zeroguiautomatingonlinegui}. While these RL-based strategies can cultivate sophisticated, self-correcting behaviors, they often face significant challenges, including high computational costs and training instability, particularly when applied to long, multi-turn dialogue contexts.

\paragraph{Test-Time Adaptation and Positioning.}
Test-time optimization~\citep{zhang2025surveytesttimescalinglarge, zuo2025ttrltesttimereinforcementlearning, chen2025learning} also changes model behavior during inference. Our distinction is the combination of an interaction-specific scalar feedback signal, an analytically constructed KL-regularized policy target, and a one-step linearized parameter update. Prompting methods adapt through context; offline multi-turn learning transfers behavior learned before deployment; sampling-based methods spend compute exploring candidate outputs. These mechanisms are complementary, and none is excluded by the definition of user preference. \ours{} focuses on lightweight adaptation within an ongoing interaction, with zero additional offline training cost; its online generation, reward evaluation, and update costs are reported separately in Appendix~\ref{sec:exp-efficiency}.

\section{Proofs} \label{sec:math}
\subsection{Proof of Theorem~\ref{thm:optimal_policy}} \label{sec:optima_policy}

\begin{proof}
Fix $\mathbf x$ and write $p=\pi_{\theta_{k-1}}(\cdot\mid\mathbf x)$. For $\beta>0$, define
\begin{equation}
J(\pi)=\mathbb E_{\pi}[r(\mathbf x,\mathbf y)]-\beta D_{\mathrm{KL}}(\pi\|p),\qquad
Z_k(\mathbf x)=\sum_{\mathbf y}p(\mathbf y)\exp(r(\mathbf x,\mathbf y)/\beta).
\end{equation}
Assume $0<Z_k(\mathbf x)<\infty$, and define $\pi^*_{\theta_k}(\mathbf y\mid\mathbf x)=p(\mathbf y)\exp(r(\mathbf x,\mathbf y)/\beta)/Z_k(\mathbf x)$. Then
\begin{align}
-J(\pi)
&=\beta\mathbb E_{\pi}\!\left[\log\frac{\pi(\mathbf y\mid\mathbf x)}{p(\mathbf y)\exp(r(\mathbf x,\mathbf y)/\beta)}\right]\\
&=\beta D_{\mathrm{KL}}(\pi\|\pi^*_{\theta_k})-\beta\log Z_k(\mathbf x).
\end{align}
Thus $\arg\max_\pi J(\pi)=\arg\min_\pi D_{\mathrm{KL}}(\pi\|\pi^*_{\theta_k})$. The objective equality retains both $\beta$ and the normalization constant; dropping them is valid only when identifying the optimizer. Over distributions supported on $p$, the minimum KL divergence is zero, attained by $\pi^*_{\theta_k}$. This distribution-space optimum need not be exactly realizable by a restricted parameter update; Theorem~\ref{thm:unified_bound} accounts for approximation.
\end{proof}

\subsection{Derivation of Equation~\ref{eq:practical_target}}
\label{sec:derivation_practica_target}
\begin{definition}[Single-Sample Feedback Constraint]
    \label{assump:single_sample_feedback}
    In practical applications, feedback is typically received for only a single generated response, $\mathbf{y}_k$. We define the turn-local surrogate reward $\tilde r_k$, rather than assume that unobserved responses have zero true user utility:
    \begin{equation}
        \tilde r_k(\mathbf{x}, \mathbf{y}) = r_k \cdot \mathbb{I}(\mathbf{y} = \mathbf{y}_k) =
        \begin{cases}
            r_k, & \text{if } \mathbf{y} = \mathbf{y}_k \\
            0, & \text{if } \mathbf{y} \neq \mathbf{y}_k
        \end{cases}
    \end{equation}
\end{definition}

\begin{proof}[Derivation of the Practical Target from the Theoretical Optimum]
Our goal is to derive the practical, single-sample update target (Equation~\ref{eq:practical_target}) and its corresponding partition function from the general theoretical optimal policy (Equation~\ref{eq:optimal_policy}) under the Single-Sample Feedback Constraint (Definition~\ref{assump:single_sample_feedback}).

\textbf{1. Derivation of the Practical Target Policy $\tilde{\pi}^*_{\theta_k}$.}
We substitute the surrogate reward from Definition~\ref{assump:single_sample_feedback} into the general policy formula from Equation~\ref{eq:optimal_policy}. This naturally yields a piecewise expression:
\begin{itemize}
    \item For the observed response where $\mathbf{y} = \mathbf{y}_k$, the reward is $r_k$, yielding:
    \begin{equation}
        \tilde{\pi}^*_{\theta_k}(\mathbf{y}|\mathbf{x}) = \frac{1}{Z_k(\mathbf{x})}\pi_{\theta_{k-1}}(\mathbf{y}|\mathbf{x}) \exp\!\left(\tfrac{1}{\beta} r_k\right)
    \end{equation}
    \item For all other responses where $\mathbf{y} \neq \mathbf{y}_k$, the reward is $0$, yielding:
    \begin{equation}
        \tilde{\pi}^*_{\theta_k}(\mathbf{y}|\mathbf{x}) = \frac{1}{Z_k(\mathbf{x})}\pi_{\theta_{k-1}}(\mathbf{y}|\mathbf{x}) \exp\!\left(0\right) = \frac{1}{Z_k(\mathbf{x})}\pi_{\theta_{k-1}}(\mathbf{y}|\mathbf{x})
    \end{equation}
\end{itemize}
Combining these two results gives the piecewise form in Equation~\ref{eq:practical_target}.

\textbf{2. Derivation of the Practical Partition Function $Z_k(\mathbf{x})$.}
Next, we apply the same constrained reward to the general partition function definition by splitting the sum over the entire response space $\mathcal{Y}$:
\begin{align*}
    Z_k(\mathbf{x}) &= \sum_{\mathbf{y}' \in \mathcal{Y}} \pi_{\theta_{k-1}}(\mathbf{y}'|\mathbf{x}) \exp\left(\frac{1}{\beta} r_k \cdot \mathbb{I}(\mathbf{y}' = \mathbf{y}_k)\right) \\
    &= \pi_{\theta_{k-1}}(\mathbf{y}_k|\mathbf{x}) \exp\left(\frac{1}{\beta} r_k\right) + \sum_{\mathbf{y}' \neq \mathbf{y}_k} \pi_{\theta_{k-1}}(\mathbf{y}'|\mathbf{x}) \exp\left(0\right) \\
    &= \pi_{\theta_{k-1}}(\mathbf{y}_k|\mathbf{x}) \exp\left(\frac{1}{\beta} r_k\right) + \left(1 - \pi_{\theta_{k-1}}(\mathbf{y}_k|\mathbf{x})\right) \\
    &= 1 - \left(1 - \exp\left(\frac{1}{\beta} r_k\right)\right)\pi_{\theta_{k-1}}(\mathbf{y}_k|\mathbf{x})
\end{align*}
This confirms the expression for the practical partition function used in Equation~\ref{eq:practical_target}. 
\end{proof}

\subsection{Proof of Theorem~\ref{thm:monotonic_error}} \label{sec:monotonic_error_reduction}

\begin{proof}
Fix $\mathbf x$, let $P=\pi_{\mathrm{user}}^*$ and $p_k=\pi_{\theta_k}$, and set $q_k=\tilde\pi^*_{\theta_k}$. By definition, $P(\mathcal A)=1$. Reliable rejection gives $\mathbf y_k\notin\mathcal A$ and hence $P(\mathbf y_k)=0$; this feedback-reliability condition is distinct from the definition of the ideal policy. The target uses the turn-local surrogate, with $r_k=-1$:
\begin{equation}
q_k(\mathbf y)=\frac{p_{k-1}(\mathbf y)\exp[-\mathbb I(\mathbf y=\mathbf y_k)/\beta]}{Z_k},\qquad
Z_k=1-Cp_{k-1}(\mathbf y_k),\quad C=1-e^{-1/\beta}.
\end{equation}
Consequently,
\begin{align}
D_{\mathrm{KL}}(P\|q_k)-D_{\mathrm{KL}}(P\|p_{k-1})
&=\mathbb E_P\log\frac{p_{k-1}}{q_k}
=\log Z_k+\frac{P(\mathbf y_k)}{\beta}\\
&=\log(1-Cp_{k-1}(\mathbf y_k))
\leq-Cp_{k-1}(\mathbf y_k)<0.
\end{align}
With exact updates, $p_{k-1}=q_{k-1}$ (and $q_0=p_0$), giving the stated result. This is a guarantee for reliable rejection of unacceptable outputs; a noisy negative score on an acceptable output does not satisfy its premise.
\end{proof}

\subsection{Proof of Theorem~\ref{thm:cumulative_error}} \label{sec:cumulative_error_bound}

\begin{proof}[Proof of Theorem~\ref{thm:cumulative_error}]
We want to bound the final estimation error after $K$ turns, $D_{\text{KL}}(\pi_{\text{user}}^* \| \tilde{\pi}^*_{\theta_K})$. We can express this final error as the initial error at turn 0 plus the sum of all one-step changes in error from turn 1 to $K$:
\begin{equation}
    D_{\text{KL}}(\pi_{\text{user}}^* \| \tilde{\pi}^*_{\theta_K}) = D_{\text{KL}}(\pi_{\text{user}}^* \| \pi_{\theta_{0}}) + \sum_{k=1}^K \left( D_{\text{KL}}(\pi_{\text{user}}^* \| \tilde{\pi}^*_{\theta_k}) - D_{\text{KL}}(\pi_{\text{user}}^* \| \tilde{\pi}^*_{\theta_{k-1}}) \right)
\end{equation}
where we define $\tilde{\pi}^*_{\theta_0} = \pi_{\theta_0}$ as the initial policy.

From Theorem~\ref{thm:monotonic_error}, we have an upper bound for each one-step change in error, where $C = 1 - \exp(-1/\beta)$:
\begin{equation}
    D_{\text{KL}}(\pi_{\text{user}}^* \| \tilde{\pi}^*_{\theta_k}) - D_{\text{KL}}(\pi_{\text{user}}^* \| \tilde{\pi}^*_{\theta_{k-1}}) \leq -C \cdot \pi_{\theta_{k-1}}(\mathbf{y}_k | \mathbf{x})
\end{equation}
We can apply this inequality to the summation term. By summing the upper bounds for each step from $k=1$ to $K$, we get an upper bound for the total change:
\begin{equation}
    \sum_{k=1}^K \left( D_{\text{KL}}(\pi_{\text{user}}^* \| \tilde{\pi}^*_{\theta_k}) - D_{\text{KL}}(\pi_{\text{user}}^* \| \tilde{\pi}^*_{\theta_{k-1}}) \right) \leq \sum_{k=1}^K \left( -C \cdot \pi_{\theta_{k-1}}(\mathbf{y}_k | \mathbf{x}) \right)
\end{equation}
Substituting this bounded sum back into our expression for the final error, we arrive at the desired result:
\begin{equation}
    D_{\text{KL}}(\pi_{\text{user}}^* \| \tilde{\pi}^*_{\theta_K}) \leq D_{\text{KL}}(\pi_{\text{user}}^* \| \pi_{\theta_{0}}) - C \sum_{k=1}^K \pi_{\theta_{k-1}}(\mathbf{y}_{k} | \mathbf{x})
\end{equation}
This completes the proof. 
\end{proof}

\paragraph{Sample efficiency under diverse failure modes.}
Let $D_0=D_{\mathrm{KL}}(P\|p_0)<\infty$. If a sequence of $K$ reliably rejected-response updates satisfies $p_{k-1}(\mathbf y_k)\geq p_{\min}>0$, the exact-update bound gives
\begin{equation}
D_K\leq D_0-CKp_{\min}.
\end{equation}
Conditionally on such a sequence, $K\geq\lceil(D_0-\eta)/(Cp_{\min})\rceil$ suffices for $D_K\leq\eta$, for $0<\eta<D_0$. This is a finite-run conditional bound, not an unconditional sampling-complexity guarantee: the lower bound on rejected-response probability need not persist, and interactions can stop after success. High-confidence recurrent errors yield larger one-step gains; diffuse, low-probability failures may yield smaller gains. With approximate updates, replace $Cp_{\min}$ by a positive lower bound $\gamma$ on $Cp_{k-1}(\mathbf y_k)-\epsilon_k-L\|\Delta\theta_k\|_2^2/2$, giving $D_K\leq D_0-K\gamma$ over the corresponding sequence. Empirical analyses of longer horizons and multiple feedback samples are provided in Appendices~\ref{sec:extended_horizons} and~\ref{sec:multi_sample_feedback}.

\paragraph{Acceptability versus a unique ideal distribution.}
Let $a=p_{k-1}(\mathcal A)>0$ and $b=1-a$. An exact rejection update preserves relative probabilities within $\mathcal A$ and yields $q_k(\mathcal A)=a/Z_k$, so
\begin{equation}
b-q_k(\mathcal A^c)=\frac{aC p_{k-1}(\mathbf y_k)}{1-Cp_{k-1}(\mathbf y_k)}>0.
\end{equation}
This is the relevant improvement in unacceptable-answer probability. Rejection alone does not identify an arbitrary preferred distribution among acceptable answers. In particular, if all unacceptable mass vanishes, the remaining distribution is the initial policy conditioned on $\mathcal A$, unless additional feedback distinguishes acceptable answers.

\subsection{Proof of Theorem~\ref{thm:unified_bound}} \label{sec:unified_error_bound}

\begin{assumption}[Lipschitz-Smooth Log-Policy]
\label{assump:lipschitz_log_policy}
For every $\mathbf y$ on the support of $P=\pi_{\mathrm{user}}^*$, the log-policy gradient is $L$-Lipschitz along the segment from $\theta_{k-1}$ to $\theta_k$. All expectations below are finite. With $\Delta\theta_k=\theta_k-\theta_{k-1}$,
\begin{equation}\label{eq:lipschitz_bound}
\log p_k(\mathbf y)=\log p_{k-1}(\mathbf y)+\nabla\log p_{k-1}(\mathbf y)^\top\Delta\theta_k+R_k(\mathbf y),\quad
|R_k(\mathbf y)|\leq\frac L2\|\Delta\theta_k\|_2^2.
\end{equation}
\end{assumption}
\paragraph{Linearized target-fitting error.}
Write $p=p_{k-1}$, $q=q_k=\tilde\pi^*_{\theta_k}$, and define
\begin{equation}\label{eq:linear_target_residual}
e_k(\mathbf y)=p(\mathbf y)+\nabla p(\mathbf y)^\top\Delta\theta_k-q(\mathbf y),\qquad
\epsilon_k=\mathbb E_{\mathbf y\sim P}\frac{|e_k(\mathbf y)|}{p(\mathbf y)}<\infty.
\end{equation}
Thus $\epsilon_k$ measures mismatch between the linearized policy prediction and the desired target, whereas $L\|\Delta\theta_k\|_2^2/2$ controls the log-policy Taylor remainder. Neither error is discarded.

\begin{proof}
Decompose each change in divergence as
\begin{align}
D_{\mathrm{KL}}(P\|p_k)-D_{\mathrm{KL}}(P\|p)
&=\underbrace{D_{\mathrm{KL}}(P\|q)-D_{\mathrm{KL}}(P\|p)}_{A_k}\label{eq:decomposition_term_a}\\
&\quad+\underbrace{D_{\mathrm{KL}}(P\|p_k)-D_{\mathrm{KL}}(P\|q)}_{B_k}.\label{eq:decomposition_term_b}
\end{align}
\textbf{Bounding Term A (Improvement).}
Reliable rejection and the definition of $P$ give
\begin{equation}\label{eq:term_a_bound}
A_k=\log Z_k\leq-Cp(\mathbf y_k),\qquad Z_k=1-Cp(\mathbf y_k).
\end{equation}
\textbf{Bounding Term B (Approximation Error).}
On the support of $P$, $q(\mathbf y)=p(\mathbf y)/Z_k$, and Equation~\eqref{eq:linear_target_residual} implies
\begin{equation}
\nabla\log p(\mathbf y)^\top\Delta\theta_k
=\frac1{Z_k}-1+\frac{e_k(\mathbf y)}{p(\mathbf y)}.
\end{equation}
Consequently, applying the log-policy expansion rather than a triangle inequality for KL divergence yields
\begin{align}\label{eq:term_b_bound}
B_k&=\mathbb E_P\log\frac{q}{p_k}\\
&=-\log Z_k-\left(\frac1{Z_k}-1\right)
-\mathbb E_P\frac{e_k}{p}-\mathbb E_PR_k\notag\\
&\leq-\log Z_k-\left(\frac1{Z_k}-1\right)+\epsilon_k+\frac L2\|\Delta\theta_k\|_2^2\notag\\
&\leq\epsilon_k+\frac L2\|\Delta\theta_k\|_2^2.\notag
\end{align}
The last inequality follows from $\log u\leq u-1$ with $u=1/Z_k$. Summing $A_k+B_k$ gives
\begin{equation}\label{eq:final_unified_bound}
D_{\mathrm{KL}}(P\|p_K)\leq D_{\mathrm{KL}}(P\|p_0)
-C\sum_{k=1}^Kp_{k-1}(\mathbf y_k)
+\sum_{k=1}^K\epsilon_k+\frac L2\sum_{k=1}^K\|\Delta\theta_k\|_2^2.
\end{equation}
\end{proof}

\paragraph{Interpretation of approximation control.}
Equation~\eqref{eq:practical_target} controls the requested change in probability through the rejected answer's mass. Translating a small target change into a small parameter update additionally depends on the conditioning of the linearized solve; the probability alone is not a parameter-norm bound. CG reduces the numerical residual of that solve. A small normal-equation residual alone does not bound the full $P$-weighted residual in Equation~\eqref{eq:linear_target_residual}: target realizability and appropriate conditioning are also needed. Under these conditions, accurate CG solves support small $\epsilon_k$. The positive-net-progress requirement is $Cp_{k-1}(\mathbf y_k)>\epsilon_k+L\|\Delta\theta_k\|_2^2/2$. The empirical results support practical effectiveness but do not establish this inequality at every update.

\section{Experimental setting}\label{sec:setup}

We conduct a comprehensive evaluation of \ours{} across a diverse set of tasks and models to validate its generalizability, effectiveness, and efficiency. 

\subsection{Datasets.} \label{sec:datasets}

To demonstrate the broad applicability of \ours{}, we select challenging benchmarks spanning four distinct problem-solving domains. A summary of these datasets is provided in Table~\ref{tab:datasets}, followed by detailed descriptions.

\begin{table}[h!]
\centering
\caption{Overview of the datasets used for evaluation. "N/A" indicates that the dataset is primarily for evaluation and does not have a standard, predefined training set.}
\label{tab:datasets}
\begin{tabular}{llrr}
\toprule
\textbf{Domain} & \textbf{Dataset Name} & \textbf{Training Size} & \textbf{Test Size} \\
\midrule
\multirow{3}{*}{Mathematical Reasoning} & \texttt{MATH} & 7,500 & 5,000 \\
& \texttt{AIME25} & N/A & 30 \\
& \texttt{MATH-500} & N/A & 500 \\
\midrule
\multirow{3}{*}{General Reasoning} & \texttt{GPQA-diamond} & N/A & 198 \\
& \texttt{MMLU-Redux} & N/A & 3,000 \\
& \texttt{SuperGPQA} & 26,500 & N/A \\
\midrule
Code Generation & \texttt{HumanEval} & N/A & 164 \\
\midrule
Multilingual Reasoning & \texttt{MCLM} & N/A & 156 \\
\bottomrule
\end{tabular}
\end{table}

\paragraph{Mathematical Reasoning.}
This domain focuses on complex, multi-step mathematical problem-solving. We use three standard benchmarks. \texttt{MATH}~\citep{hendrycksmath2021} is a dataset of 12,500 challenging competition mathematics problems from high school level, covering topics like algebra, geometry, and calculus. \texttt{AIME25}~\citep{aime} is a curated set of 25 highly difficult problems from the American Invitational Mathematics Examination (AIME), designed to test advanced reasoning capabilities. \texttt{MATH-500}~\citep{lightman2023lets} is a well-known evaluation subset of the \texttt{MATH} test set, consisting of 500 problems often used for efficient model assessment.

\paragraph{General Reasoning.}
To evaluate reasoning on a broad range of topics, we use three expert-level question-answering datasets. \texttt{GPQA-diamond}~\citep{rein2024gpqa} is a challenging set of graduate-level, Google-proof questions written by domain experts, where the "diamond" subset represents the highest-quality questions. \texttt{MMLU-Redux}~\citep{hendryckstest2021} is a revised and cleaned version of the Massive Multitask Language Understanding benchmark, which covers 57 diverse subjects from elementary mathematics to US history and law. \texttt{SuperGPQA}~\citep{pteam2025supergpqascalingllmevaluation} significantly expands upon GPQA, containing nearly 5,000 expert-validated questions across 285 graduate-level disciplines.

\paragraph{Code Generation.}
We test the ability of models to generate functionally correct code from natural language descriptions using \texttt{HumanEval}~\citep{chen2021evaluatinglargelanguagemodels}. This dataset consists of 164 hand-written programming problems with function signatures, docstrings, and unit tests to verify the correctness of the generated code.

\paragraph{Multilingual Reasoning.}
To assess reasoning capabilities across different languages, we use the \texttt{MCLM}~\citep{son2025linguistic} benchmark. This dataset was created by translating challenging English reasoning benchmarks into multiple languages. Our evaluation focuses on its subsets, including multilingual versions of IMO, AIME, and MATH problems (\texttt{M-IMO}, \texttt{MT-AIME24}, and \texttt{MT-MATH100}).

\paragraph{Dataset Usage in Experiments.}
Our primary evaluation of effectiveness of \ours{} is conducted on official, held-out test sets to simulate real-world performance. For experiments where a dedicated test set is not available, or for ablation studies, we utilize the corresponding training or development sets for analysis. This ensures a comprehensive assessment of \ours{} capabilities across different conditions while maintaining a clear distinction between final evaluation and component analysis. Specifically, we only sample part of the data from the \texttt{SuperGPQA} training set for testing, and the rest of the data sets are tested on the test set.

\subsection{Models}
\label{sec:models}

Our model selection separates breadth from mechanism analysis: cross-model experiments test generality, while most controlled MATH diagnostics use Qwen3-0.6B. ALOE and successful-trajectory experiments use Qwen3-8B and are reported separately; results across these settings are not pooled. Our evaluation includes a variety of recent open-source LLMs to ensure our findings are not model-specific. These models are selected to cover a range of sizes and specializations, as summarized in Table~\ref{tab:models} and detailed below. To mitigate potential data contamination issues with the \texttt{Qwen2.5} series on certain benchmarks, we also conduct validation experiments on the more recent \texttt{Qwen3} and \texttt{DeepSeek-R1} models. All models used are instruction-tuned variants designed for chat and instruction-following tasks.

\begin{table}[h!]
\centering
\caption{Overview of the language models used in our experiments, categorized by scale and specialization.}
\label{tab:models}
\resizebox{\textwidth}{!}{
\begin{tabular}{llcl}
\toprule
\textbf{Category} & \textbf{Model Name} & \textbf{Parameters} & \textbf{Variant} \\
\midrule
\multirow{2}{*}{Small-Scale Models} & \texttt{Qwen2.5-0.5B-Instruct} & 0.5B & Instruct \\
& \texttt{Qwen3-0.6B} & 0.6B & Base \\
\midrule
\multirow{2}{*}{Large-Scale Models} & \texttt{Qwen2.5-7B-Instruct} & 7B & Instruct \\
& \texttt{Qwen3-8B} & 8B & Base \\
\midrule
\multirow{2}{*}{Reasoning-Focused} & \texttt{DeepSeek-R1-Distill-Llama-8B} & 8B & Reasoning-Tuned \\
& \texttt{DeepSeek-R1-Distill-Qwen-7B} & 7B & Reasoning-Tuned \\
\bottomrule
\end{tabular}
}
\end{table}

\paragraph{Small-Scale Models.}
To assess the performance of \ours{} on more compact models, we selected two from the Qwen family, known for their strong general-purpose capabilities. \texttt{Qwen2.5-0.5B-Instruct}~\citep{qwen2025qwen25technicalreport} is a 0.5 billion parameter model from the Qwen2.5 series, optimized for instruction following. \texttt{Qwen3-0.6B}~\citep{yang2025qwen3technicalreport} is a 0.6 billion parameter model from the newer Qwen3 generation, featuring architectural improvements.

\paragraph{Large-Scale Models.}
We evaluate on larger, more capable base models to test the scalability of our approach. These include \texttt{Qwen2.5-7B-Instruct}~\citep{qwen2025qwen25technicalreport}, a widely-used 7 billion parameter instruction-tuned model, and \texttt{Qwen3-8B}~\citep{yang2025qwen3technicalreport}, its 8 billion parameter successor from the Qwen3 series.

\paragraph{Reasoning-Focused Models.}
To specifically test performance on complex reasoning, we use models from the DeepSeek-R1 series, which are explicitly optimized for reasoning capabilities through reinforcement learning~\citep{deepseekai2025deepseekr1incentivizingreasoningcapability}. The models we use are distilled versions of a larger, proprietary model. \texttt{DeepSeek-R1-Distill-Llama-8B} is an 8 billion parameter model that uses a Llama-based architecture. \texttt{DeepSeek-R1-Distill-Qwen-7B} is a 7 billion parameter variant that is instead based on the Qwen architecture, allowing for a more controlled comparison with the general-purpose Qwen models.

\subsection{Evaluation Metrics}
\label{sec:metrics}

We assess \ours{} based on two primary aspects: performance and efficiency.

\paragraph{Performance Metrics.}
To measure problem-solving success, we define two key metrics. \textit{Accuracy} is the final proportion of unique problems solved correctly within a total of $K$ conversational turns. Let $\mathcal{P}$ be the set of all problems, and let $S_i \in \{0, 1\}$ be an indicator variable where $S_i=1$ if problem $i$ is solved at any turn up to $K$. The accuracy is given by:
\begin{equation}
    \text{Accuracy} = \frac{\sum_{i \in \mathcal{P}} S_i}{|\mathcal{P}|}
\end{equation}

\textit{Correction Uplift} measures the ability of model to recover from initial failures. It is the percentage of problems that were answered incorrectly in the first turn but were successfully corrected in a subsequent turn. Let $\mathcal{P}_{\text{fail}} \subset \mathcal{P}$ be the subset of problems that the model failed to solve in the first turn. The Correction Uplift is:
\begin{equation}
    \text{Correction Uplift} = \frac{\sum_{i \in \mathcal{P}_{\text{fail}}} S_i}{|\mathcal{P}_{\text{fail}}|} \times 100\%
\end{equation}

\paragraph{Efficiency Metrics.}
To evaluate the computational overhead of our method, we measure two metrics. \textit{Latency} is the average wall-clock time required for a single generation and update cycle. \textit{Peak GPU Memory} is the maximum GPU memory consumed during this cycle. These metrics are crucial for assessing the practical feasibility of deploying \ours{} in real-world interactive systems.

\subsection{Reward Models}
\label{sec:reward}

To simulate different real-world feedback scenarios, we employ two types of reward models.

\paragraph{Rule-Based Reward Model.}
This model simulates scenarios with definitive, high-level judgments by providing a sparse feedback signal of $\{-1, +1\}$. It programmatically extracts the final answer from a model's response, typically from a \verb|\boxed{}| environment, and compares it to the ground-truth solution. A reward of $+1.0$ is assigned for a correct answer, and $-1.0$ otherwise. This mimics situations where feedback is based solely on the final outcome. The core logic implementation is shown in the following table.

\begin{tcolorbox}[title={\textbf{Core logic for the rule-based reward model.}}, colback=asparagus!10!white,colframe=asparagus!90!white, left=0.5mm, right=1mm, top=1mm, bottom=1mm]
\begin{lstlisting}[language=Python, basicstyle=\small\ttfamily]
class MathVerifyRewardModel:
    def __init__(self, ground_truth_answer: str):
        self.ground_truth_answer = ground_truth_answer

    def get_reward(self, response_text: str) -> float:

        return 1.0 if compute_score(response_text, 
        self.ground_truth_answer) == 1.0 else -1.0

def compute_score(solution_str, ground_truth) -> float:
    retval = 0.0
    try:
        string_in_last_boxed = 
        last_boxed_only_string(solution_str)
        if string_in_last_boxed is not None:
            answer = remove_boxed(string_in_last_boxed)
            if is_equiv(answer, ground_truth):
                retval = 1.0
    except Exception:
        pass
    return retval
\end{lstlisting}
\end{tcolorbox}

\paragraph{Model-Based Reward Model.}
This model simulates more nuanced, fine-grained human feedback by providing a dense, continuous reward score in the range $[-1.0, +1.0]$. We use the large language model, \texttt{Qwen/Qwen3-30B-A3B-Instruct-2507}, as the reward judge. The model is deployed using the VLLM inference engine for efficient scoring. It evaluates the generated response based on correctness, reasoning, and style by comparing it against the problem statement and the ideal solution. The prompt used to elicit the score is shown in the following table.

\begin{tcolorbox}[title={\textbf{The prompt template for the model-based reward system.}}, colback=asparagus!10!white,colframe=asparagus!90!white, left=0.5mm, right=1mm, top=1mm, bottom=1mm]
\begin{lstlisting}[style=mystyle]
A student AI was asked the following problem: {problem}. 
The student AI gave the following answer: {generated_text}. 
The ideal correct solution and answer is {solution}. 
Please grade strictly but fairly. 
Compare the student's answer to the ideal answer. 
Evaluate the student's answer based on correctness, reasoning, and style. 
Note: Based on your evaluation, please provide a floating point score 
from -1.0 (completely wrong) to 1.0 (perfect). 
The score should be placed at the end of your answer in the format: SCORE: [score].
\end{lstlisting}
\end{tcolorbox}

\paragraph{Reward rubric and calibration.}\label{sec:reward_calibration}
We additionally evaluate a discrete reasoning-aware rubric using GPT-5.6. Given the question, generated response, reference reasoning, and reference answer, the judge assigns one of five scores using Table~\ref{tab:reward_rubric}. The final-answer verifier separately determines task success: a positive partial-credit score can trigger an update without terminating an incorrect attempt. This supplementary rubric study is distinct from the original Qwen3-30B-A3B outcome judge and the GPT-5.2 PRM evaluation.
\begin{table}[ht]
\centering\small
\caption{Five-level reward rubric.}\label{tab:reward_rubric}
\begin{tabularx}{\linewidth}{rX}
\toprule Score & Criterion\\\midrule
$+1.0$ & Correct reasoning and correct final answer.\\
$+0.5$ & Essentially correct method with a local arithmetic or presentation error.\\
$+0.1$ & Partially correct core idea, but a critical step is missing or flawed.\\
$-0.5$ & Limited valid progress; the main reasoning is incorrect.\\
$-1.0$ & Incorrect or irrelevant reasoning and answer.\\\bottomrule
\end{tabularx}
\end{table}
For example, consider a triangle with side lengths $7$, $10$, and $x^2$, where $x$ is a positive integer. The valid values are $x\in\{2,3,4\}$ because $3<x^2<17$. A correct derivation and answer receives $+1$; correctly deriving these values but transcribing the final set as $\{2,3,5\}$ receives $+0.5$; using only $x^2<17$ and returning $\{1,2,3,4\}$ receives $+0.1$; reversing the triangle inequality to $17\leq x^2$ receives $-0.5$; simply adding $7+10$ and answering $x=17$ receives $-1$.

For calibration, five annotators independently rated 50 responses to 50 MATH problems, with 10 responses sampled from each GPT-5.6 score level. They received the same rubric, problem, generated response, reference reasoning, and reference answer. Pairwise exact agreement was 93.6\%, all-five agreement was 86.0\%, and Fleiss' $\kappa$ was 0.920. GPT-5.6 agreed with individual annotators on 96.4\% of ratings and with the majority on all 50 responses; mean Cohen's $\kappa$ was 0.955. These figures describe this small, score-stratified mathematical-reasoning calibration set, rather than population-wide reliability or calibration of the separate ALOE judge. Downstream rubric results are reported in Table~\ref{tab:prm_results}.

\subsection{Parameter Update Mechanisms}
\label{sec:parameter}

To implement the policy update $\Delta\theta$ computed in Section~\ref{sec:parameter_update}, we introduce two distinct, lightweight update mechanisms. These methods are designed to be computationally efficient, allowing for real-time policy adaptation during the inference phase without significant overhead.

\paragraph{1. LM Head Update via LoRA.}
The first mode targets the final layer of the model, the language modeling (LM) head. The LM head is typically a linear layer (an MLP matrix) that projects the final hidden state representation of the model into the vocabulary space to produce logits. We augment this layer by adding a Low-Rank Adaptation (LoRA)~\citep{hu2022lora} matrix. Specifically, a low-rank decomposition, represented by two matrices $\mathbf{A} \in \mathbb{R}^{d \times r}$ and $\mathbf{B} \in \mathbb{R}^{r \times V}$ (where $d$ is the hidden size, $V$ is the vocabulary size, and $r \ll d,V$ is the rank), is added to the original LM head weight matrix. During our online update process, only the parameters of these small LoRA matrices $\mathbf{A}$ and $\mathbf{B}$ are modified. The parameter update $\Delta\theta$ calculated by the CG solver is applied directly to the flattened weights of $\mathbf{A}$ and $\mathbf{B}$. This approach confines the policy optimization to a very small subset of the total model parameters, preserving the model's foundational knowledge while enabling rapid and efficient adaptation of its final output probabilities. The specific LoRA configuration is shown in Table~\ref{tab:lora_config}.

\begin{table}[h!]
\centering
\caption{LoRA Hyperparameter Configuration.}
\label{tab:lora_config}
\begin{tabular}{lc}
\toprule
\textbf{Hyperparameter} & \textbf{Value} \\
\midrule
\texttt{target\_modules} & lm\_head \\
\texttt{Rank} & 1 \\
\texttt{lora\_alpha} & 8 \\
\texttt{lora\_dropout} & 0.1 \\
\bottomrule
\end{tabular}
\end{table}

\paragraph{2. Hidden State Modification.}
The second mode operates not on the model's weights, but directly on its activations~\citep{hu2025slotsamplespecificlanguagemodel}. Instead of modifying a layer, we intercept the final hidden state $\mathbf{H} \in \mathbb{R}^{1 \times d}$ just before it is passed to the LM head. We then compute an update vector $\Delta\mathbf{H} \in \mathbb{R}^{1 \times d}$ (which in this context represents our $\Delta\theta$) and add it directly to the hidden state to produce a modified activation:
\begin{equation}
    \mathbf{H}_{\text{new}} = \mathbf{H} + \Delta\mathbf{H}
\end{equation}
This new hidden state, $\mathbf{H}_{\text{new}}$, is then passed to the original, unmodified LM head to generate the final logits. This method is implemented using model hooking techniques, which allow us to register a forward hook on the LM head layer. The hook intercepts the input ($\mathbf{H}$), applies the additive modification, and returns the transformed tensor as the new input for the layer's forward pass. This approach completely avoids any updates to the persistent model weights and instead performs a transient, state-dependent policy correction on the activation flow.

\subsection{Multi-Turn Evaluation Protocol} \label{sec:eval_protocol}

To ensure reproducibility and a strictly fair comparison between our proposed \ours{} framework and all baseline methods, we standardize the multi-turn interaction construction process across all datasets. The primary verifiable-task evaluation uses at most 10 turns, stops upon verified success, and resets adaptation before the next problem. Its interaction construction follows these steps:

\begin{itemize}
    \item \textbf{Initial Turn:} The interaction begins by providing the exact same initial task prompt to all evaluated methods.
    \item \textbf{Simulated User Feedback:} If the model generates an incorrect answer, we programmatically provide a standardized, simulated user feedback message. For instance, the default template used in our primary evaluation is: \textit{``Wrong answer, please rethink and try another way of thinking!''}. 
    \item \textbf{Context Appending:} For each subsequent turn, the model's previous response and the corresponding user feedback are appended to the conversational history. This growing context is then used as the input for the next generation step.
\end{itemize}

\textbf{Fairness and Consistency:} It is critical to emphasize that both the baseline models and \ours{} receive the exact same initial prompts and simulated feedback messages throughout the evaluation. The only variation in the provided context stems naturally from the different intermediate responses generated by the models themselves across turns. This controls the textual-feedback construction. The scalar reward used for adaptation is specified separately in Appendix~\ref{sec:reward}.

\subsection{Robustness to Feedback Templates} \label{sec:feedback_robustness}

A common challenge in standard prompt-based multi-turn interactions is the high sensitivity of Large Language Models to the specific phrasing of the feedback. To verify that the performance gains of \ours{} are fundamentally rooted in its dynamic parameter adaptation rather than artifacts of a specific prompt design, we evaluate its robustness across varying feedback formulations.

We conducted this evaluation on the MATH dataset using the Qwen3-0.6B model. In addition to the default feedback template used in our primary experiments, we tested two alternative templates designed to simulate different user interaction styles:
\begin{itemize}
    \item \textbf{Template A (Default):} The standard conversational feedback used across our main evaluation (e.g., \textit{``Wrong answer, please rethink and try another way of thinking!''}).
    \item \textbf{Template B (Direct):} A concise and direct notification of failure (e.g., \textit{``Incorrect''}).
    \item \textbf{Template C (Guiding):} A slightly more encouraging and guiding prompt formulation (e.g., \textit{``That doesn't seem quite right. Could you review your previous steps and see if there's an alternative approach?''}).
\end{itemize}

\begin{table}[h]
\centering
\caption{Robustness to different feedback templates on the MATH dataset using Qwen3-0.6B. The baseline accuracy without \ours{} is 25.00\%.}
\label{tab:feedback_templates}
\resizebox{\textwidth}{!}{
\begin{tabular}{lccc}
\toprule
\textbf{Method} & \textbf{Template A (Default)} & \textbf{Template B (Direct)} & \textbf{Template C (Guiding)} \\
\midrule
\ours{} (+LM+R) & 50.20\% & 48.40\% & 52.60\% \\
\ours{} (+HS+R) & 50.80\% & 50.20\% & 51.60\% \\
\ours{} (+LM+M) & 52.20\% & 53.60\% & 54.40\% \\
\bottomrule
\end{tabular}
}
\end{table}

As demonstrated in Table~\ref{tab:feedback_templates}, \ours{} consistently and significantly outperforms the standard baseline across all evaluated templates. The variation in accuracy among the different templates is relatively small, particularly for the model-based reward variant (+LM+M), which remains highly stable across different phrasing. This empirical evidence confirms that the self-correction capability of \ours{} is highly robust to the syntactic variations of user feedback, validating our claim that the effectiveness of the method stems from its test-time policy optimization mechanism rather than superficial prompt engineering.

\section{Computational Efficiency Analysis}
\label{sec:exp-efficiency}

\subsection{Time-to-Accuracy Efficiency}
\label{sec:time_to_accuracy}

A potential concern with test-time adaptation is the computational overhead introduced by the online parameter updates. Indeed, \ours{} requires calculating gradients and updating the model at intermediate steps, which theoretically increases the latency per turn. To rigorously evaluate this trade-off, we present the Time-to-Accuracy curves in Figure~\ref{fig:time_to_accuracy}. As illustrated, in the initial phase of interaction, \ours{} (solid lines) may exhibit a lower accuracy compared to the baseline (dashed lines) for the same elapsed time. This initial lag is the direct consequence of the additional computational cost required for gradient calculation and backpropagation, which extends the duration of each turn.

However, as the interaction progresses, the benefits of \ours{} become evident. By dynamically modifying model parameters to align with user preferences, \ours{} corrects errors more effectively, resulting in a significantly steeper learning curve. Consequently, \ours{} rapidly overtakes the baseline, achieving a substantially higher accuracy within the same total wall-clock time. This crossover demonstrates that despite the per-turn latency overhead, the superior error-correction rate of \ours{} makes it more time-efficient in the long run for solving complex problems.

\begin{figure}[h]
    \centering
    \includegraphics[width=\columnwidth]{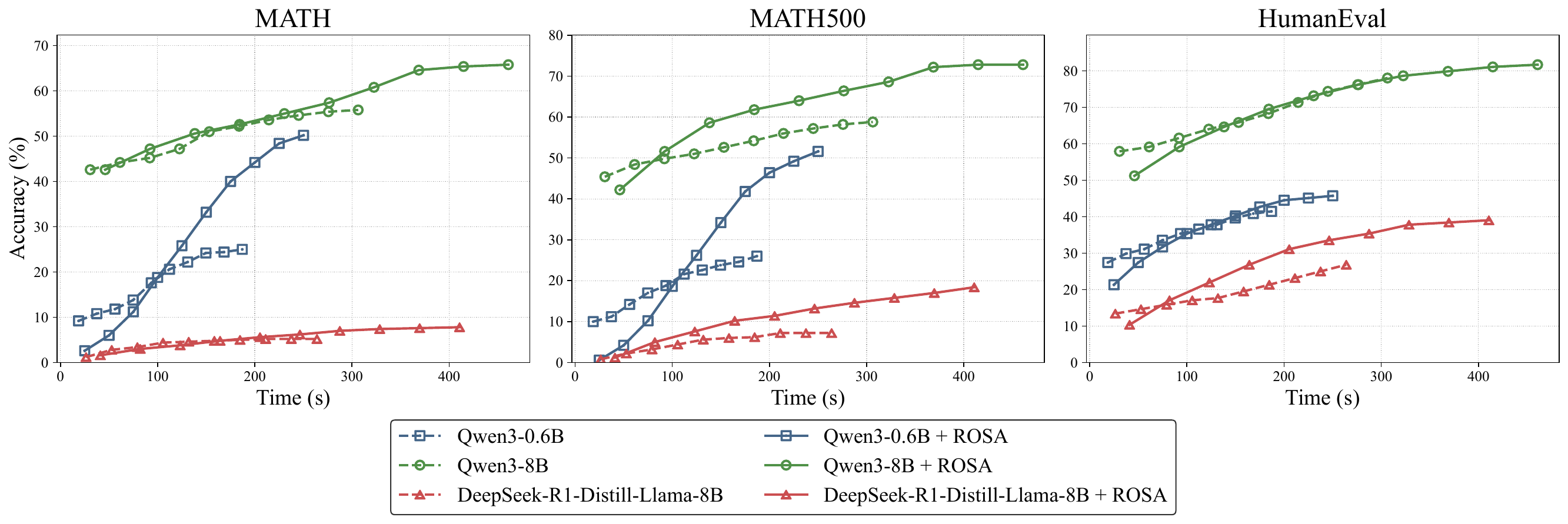} 
    \vspace{-5mm}
    \caption{
    Time-to-Accuracy comparison across different models and datasets. The x-axis represents cumulative wall-clock time. Our method \ours{} consistently exhibits steeper learning curves, highlighting its significant advantage in time efficiency.
    }
    \label{fig:time_to_accuracy}
    \vspace{-5mm}
\end{figure}

\subsection{Memory and Latency Overhead}
\label{sec:memory_latency_overhead}

In this section, we analyze the absolute computational overhead of \ours{} in terms of latency and spatial complexity. \ours{} introduces an explicit parameter update step, which incurs additional time and memory costs. As shown in Table~\ref{tab:detailed_efficiency}, the average update time makes the total processing time per turn roughly double that of the baseline's inference-only time. However, this update process is designed to be executed asynchronously. In a real-world application, the update can be performed in the background while the user is interpreting the response of the model and formulating their next prompt. This parallel processing makes the update latency largely imperceptible to the user, enabling a seamless and responsive interactive experience.

Regarding spatial complexity, \ours{} incurs negligible memory overhead. As shown in Table~\ref{tab:detailed_efficiency}, the update peak shows only a modest increase over the inference peak for models up to 8B parameters. Crucially, to validate the scalability of our approach on larger foundation models, we profiled the overhead on the Qwen3-32B model on the MATH dataset, as detailed in Table~\ref{tab:overhead_32b}. \ours{} requires a maximum memory increase of only $+0.7$ GB on a 32B model. Because our approach utilizes a matrix-free Conjugate Gradient solver, it avoids materializing full Hessian-like matrices. This confirms that our method enhances reasoning capabilities without imposing prohibitive hardware constraints, and the relative latency overhead remains highly practical for real-time deployment even at large scales.

\begin{table}[h]
\centering
\caption{Efficiency analysis of \ours{} on models up to 8B parameters. We report the averaged inference latency and peak GPU memory per turn. The ``Update'' columns show the additional overhead introduced by \ours{}.}
\label{tab:detailed_efficiency}
\resizebox{\linewidth}{!}{
\begin{tabular}{llcccc}
\toprule
\multirow{2}{*}{\textbf{Model}} & \multirow{2}{*}{\textbf{Method}} & \multicolumn{2}{c}{\textbf{Averaged Latency (s)}} & \multicolumn{2}{c}{\textbf{Peak GPU Memory (GB)}} \\
\cmidrule(lr){3-4} \cmidrule(lr){5-6}
& & \textbf{Inference} & \textbf{Update} & \textbf{Inference} & \textbf{Update} \\
\midrule
\multirow{2}{*}{Qwen2.5-0.5B-Instruct} 
& Baseline & 8.42 & 0 & 11.85 & 0 \\
& \ours{} & 7.97 & 7.58 & 11.68 & 18.60 \\
\midrule
\multirow{2}{*}{Qwen3-0.6B} 
& Baseline & 18.71 & 0 & 15.61 & 0 \\
& \ours{} & 18.85 & 13.16 & 16.24 & 23.65 \\
\midrule
\multirow{2}{*}{Qwen2.5-7B-Instruct} 
& Baseline & 19.40 & 0 & 42.01 & 0 \\
& \ours{} & 19.72 & 23.28 & 43.41 & 48.95 \\
\midrule
\multirow{2}{*}{Qwen3-8B} 
& Baseline & 30.64 & 0 & 51.90 & 0 \\
& \ours{} & 30.70 & 28.39 & 51.90 & 58.63 \\
\bottomrule
\end{tabular}
}
\end{table}

\begin{table}[h]
\centering
\caption{Computational overhead on the Qwen3-32B model, evaluated on the MATH dataset.}
\label{tab:overhead_32b}
\begin{tabular}{lcc}
\toprule
\textbf{Model / Method} & \textbf{Peak GPU Memory (GB)} & \textbf{Avg. Time to Solve} \\
\midrule
Qwen3-32B (Baseline) & 94.2 & 1.0x (Reference) \\
+ \ours{} (+LM+R) & 94.9 \textcolor{red}{(+0.7)} & 1.9x \\
\bottomrule
\end{tabular}
\end{table}

\section{Extended Empirical Evaluations}
\label{sec:extended_eval}

\subsection{Generalization to Open-Ended and Delayed-Reward Tasks}
\label{sec:open_ended_delayed}

To comprehensively validate that the effectiveness of \ours{} is not limited to tasks with clean, immediately verifiable reward signals (such as mathematical reasoning or code generation), we extend our evaluation to open-ended, subjective, and delayed-reward scenarios.

\textbf{Delayed-Reward Scenarios (UI Agent Tasks):}
Real-world autonomous agents frequently operate in environments where feedback is exceptionally sparse and delayed. To simulate this, we evaluate \ours{} on the OSWorld and AndroidWorld benchmarks using the UI-TARS-7B model (both SFT and DPO variants). In this setup, the models navigate operating systems with no intermediate feedback. A single, delayed binary reward is provided only after the agent completes a long trajectory (averaging around 15 turns). If the task fails, \ours{} applies a single parameter update based on the entire failed trajectory. 

As shown in Table~\ref{tab:ui_agent}, \ours{} successfully handles long-horizon delayed rewards, yielding consistent improvements in success rates. This proves that our algorithm can efficiently backpropagate sparse failure signals over extended interactions to correct underlying navigation policies.

\begin{table}[h]
\centering
\caption{Success Rate (\%) on delayed-reward UI Agent tasks (OSWorld and AndroidWorld). The models complete long trajectories ($\sim$15 turns) before receiving a single delayed reward.}
\label{tab:ui_agent}
\begin{tabular}{llcc}
\toprule
\textbf{Model} & \textbf{Setup} & \textbf{OSWorld (Online)} & \textbf{AndroidWorld (Online)} \\
\midrule
\multirow{3}{*}{UI-TARS-7B-SFT} 
& Baseline & 13.23 & 27.61 \\
& + \ours{} (+LM+R) & 16.71 \textcolor{red}{(+3.48)} & 30.26 \textcolor{red}{(+2.65)} \\
& + \ours{} (+LM+M) & 17.83 \textcolor{red}{(+4.60)} & 30.93 \textcolor{red}{(+3.32)} \\
\midrule
\multirow{3}{*}{UI-TARS-7B-DPO} 
& Baseline & 14.80 & 28.90 \\
& + \ours{} (+LM+R) & 17.32 \textcolor{red}{(+2.52)} & 30.37 \textcolor{red}{(+1.47)} \\
& + \ours{} (+LM+M) & 18.04 \textcolor{red}{(+3.24)} & 31.68 \textcolor{red}{(+2.78)} \\
\bottomrule
\end{tabular}
\end{table}

\textbf{Subjective \& Open-Ended Scenarios:}
Furthermore, we evaluated \ours{} on tasks lacking objective, deterministic ground truths. We utilized MT-Bench to assess subjective conversational alignment and IFEval to measure strict formatting instruction following. For MT-Bench, we employed GPT-5.2 as a continuous judge to provide fine-grained, continuous reward scores ($r_k \in [-1, 1]$). For IFEval, we utilized a programmatic structural constraint checker.

The results in Table~\ref{tab:open_ended} demonstrate substantial performance leaps. Notably, on MT-Bench, \ours{} improved the subjective score metric by +15.21, and on IFEval, the strict accuracy increased by +11.95. These results decisively confirm that \ours{} is highly task-agnostic; it can dynamically align the policy in-situ with subjective human preferences and complex structural constraints just as effectively as it solves deterministic math problems.

\begin{table}[h]
\centering
\caption{Generalization to open-ended and subjective tasks using Qwen3-0.6B.}
\label{tab:open_ended}
\begin{tabular}{lllcc}
\toprule
\textbf{Task Domain} & \textbf{Benchmark} & \textbf{Metric} & \textbf{Baseline} & \textbf{\ours{}} \\
\midrule
Subjective & MT-Bench & GPT-5.2 Score (Scaled) & 60.36 & 75.57 \textcolor{red}{(+15.21)} \\
Instruction & IFEval & Strict Acc. (\%) & 58.31 & 70.26 \textcolor{red}{(+11.95)} \\
\bottomrule
\end{tabular}
\end{table}

\subsection{Comparison with Test-Time Compute Baselines}
\label{sec:test_time_compute}

To definitively prove that the performance gains of \ours{} stem from algorithmic innovation rather than simply scaling test-time compute, we comprehensively benchmarked our method against several families of test-time scaling and prompt optimization techniques. We evaluated these baselines on the MATH dataset using the Qwen3-0.6B model, enforcing a strict maximum limit of 10 interactions or samples per problem for all methods to ensure a fair comparison.

The baselines are categorized into four distinct paradigms:
\begin{itemize}
    \item \textbf{Automatic Prompt Engineering (APE):} Methods like DSPy, TextGrad, and TSGD-M that optimize the prompt on a training subset (e.g., 200 samples from the MATH training set) using textual gradients or momentum, and apply the best prompt statically during inference.
    \item \textbf{Search \& Sampling:} Methods that generate multiple candidate responses and select the best one, including Best-of-$N$ ($N=10$) and Reranking using a Reward Model (RM).
    \item \textbf{Self-Critique / Reflexion:} Context-based self-correction where the model verbally reflects on its past errors and attempts to self-improve based on textual feedback.
    \item \textbf{Iterative Adaptation:} Test-time optimization methods that update model weights iteratively using standard gradient descent at inference time (e.g., TTRL, SLOT).
\end{itemize}

\begin{table}[h]
\centering
\caption{Comprehensive comparison against test-time compute baselines on the MATH dataset using Qwen3-0.6B. \ours{} achieves the highest accuracy and correction uplift with a highly competitive time overhead.}
\label{tab:test_time_compute}
\resizebox{\textwidth}{!}{
\begin{tabular}{llccc}
\toprule
\textbf{Paradigm} & \textbf{Specific Method} & \textbf{Final Acc. (\%)} & \textbf{Correction Uplift (\%)} & \textbf{Avg. Time to Solve} \\
\midrule
Standard & Zero-Shot Baseline & 26.00 & 17.40 & 1.0x (Reference) \\
\midrule
\multirow{3}{*}{Automatic Prompt Eng.} 
& APE (DSPy~\cite{khattab2024dspy}) & 29.20 & 16.83 & $\sim$1.0x \\
& APE (TextGrad~\cite{yuksekgonul2025optimizing}) & 29.40 & 16.76 & $\sim$1.0x \\
& APE (TSGD-M~\cite{ding2025scalingtextualgradientssamplingbased}) & 30.20 & 17.16 & $\sim$1.0x \\
\midrule
\multirow{2}{*}{Search \& Sampling} 
& Best-of-$N$ ($N=10$) & 36.40 & 20.36 & $\sim$10.0x \\
& Reranking ($N=10$ + RM) & 37.80 & 21.57 & $\sim$10.0x + RM \\
\midrule
Self-Critique & Reflexion~\cite{shinn2023reflexion} & 39.20 & 22.73 & $\sim$2.1x \\
\midrule
\multirow{2}{*}{Iterative Adaptation} 
& TTRL~\cite{zuo2026ttrl} (Iterative) & 45.40 & 31.67 & $\sim$21.3x \\
& SLOT~\cite{hu2025slotsamplespecificlanguagemodel} (Iterative) & 40.37 & 29.13 & $\sim$8.0x \\
\midrule
\textbf{$T^2PAM$ (Ours)} & \textbf{\ours{} (+LM+R)} & \textbf{50.20} & \textbf{48.87} & \textbf{$\sim$1.8x} \\
\bottomrule
\end{tabular}
}
\end{table}

As detailed in Table~\ref{tab:test_time_compute}, \ours{} significantly outperforms all compared baselines in both final accuracy and the ability to correct initial mistakes. We highlight the following insights from this comparison:

\textbf{Vs. Search \& Sampling:} While sampling methods predictably scale performance, they incur massive, linearly scaling computational costs ($\sim$10x latency) and merely select from a static policy distribution without fixing the underlying reasoning flaw. \ours{} dynamically alters the policy in-situ, effectively achieving the correct answer in just a few targeted turns with a fraction of the computational cost ($\sim$1.8x).

\textbf{Vs. Self-Critique / Reflexion:} Context-based self-correction bottlenecks heavily on the model's intrinsic evaluation capabilities. Smaller models often lack strong intrinsic self-critique, causing methods like Reflexion to plateau quickly. \ours{} bypasses this ``verbal reflection'' bottleneck entirely via direct, mathematical policy adaptation.

\textbf{Vs. Iterative Adaptation:} Test-time methods that rely on iterative gradient descent (like TTRL and SLOT) cause prohibitive inference latency due to multiple backpropagation passes per turn. Furthermore, iterative gradient descent on single samples risks catastrophic degradation. \ours{} solves this by deriving a closed-form optimal policy for a highly efficient, single-step linearized update via a Conjugate Gradient (CG) solver, ensuring both speed and stability.

\textbf{Vs. Automatic Prompt Engineering (APE):} While APE improves starting accuracy by establishing a strong initial prompt (a Turn-0 baseline), it fundamentally yields a global, static policy. It cannot dynamically adapt to instance-specific errors mid-conversation. In contrast, \ours{} translates live feedback into an in-situ parameter update, dynamically correcting the reasoning trajectory exactly where static prompts fail.

\subsection{Scalability to Large Foundation Models}
\label{sec:scalability_large_models}

To validate that the effectiveness of \ours{} extends beyond small- and medium-scale models ($\le$ 8B parameters), we evaluated its scalability by applying it to the much larger Qwen3-32B foundation model. As models increase in scale, their zero-shot baseline performance naturally becomes exceptionally high, making it increasingly difficult to extract further performance gains through test-time adaptation.

We benchmarked Qwen3-32B on four challenging datasets: MATH, MMLU-R, MT-AIME24, and HumanEval. The results are summarized in Table~\ref{tab:qwen3_32b_performance}.

\begin{table}[h]
\centering
\caption{Performance scalability of \ours{} on the Qwen3-32B model across diverse benchmarks.}
\label{tab:qwen3_32b_performance}
\resizebox{\textwidth}{!}{
\begin{tabular}{lcccc}
\toprule
\textbf{Model / Method} & \textbf{MATH} & \textbf{MMLU-R} & \textbf{MT-AIME24} & \textbf{HumanEval} \\
\midrule
Qwen3-32B (Baseline) & 76.40 & 73.63 & 67.93 & 85.37 \\
+ \ours{} (+LM+R) & 87.80 \textcolor{red}{(+11.40)} & 85.12 \textcolor{red}{(+11.49)} & 76.19 \textcolor{red}{(+8.26)} & 87.80 \textcolor{red}{(+2.43)} \\
+ \ours{} (+HS+R) & 87.40 \textcolor{red}{(+11.00)} & 86.05 \textcolor{red}{(+12.42)} & 77.27 \textcolor{red}{(+9.34)} & 87.80 \textcolor{red}{(+2.43)} \\
\bottomrule
\end{tabular}
}
\end{table}

As shown in Table~\ref{tab:qwen3_32b_performance}, despite the already formidable baseline performance of the 32B model, \ours{} successfully extracts massive performance gains across all evaluated domains. For instance, on the MATH dataset, \ours{} improves the accuracy by an absolute +11.40\%, pushing the final accuracy to 87.80\%. Similar substantial improvements are observed in MMLU-R (+12.42\%) and MT-AIME24 (+9.34\%). These results decisively confirm that \ours{} is highly scalable and remains a powerful mechanism for test-time policy alignment even for large-scale foundation models.

Furthermore, as previously detailed in our efficiency analysis (Section~\ref{sec:memory_latency_overhead}), \ours{} achieves these significant performance leaps with a negligible peak GPU memory overhead of only $+0.7$ GB on the 32B model. This confirms that \ours{} scales gracefully in terms of both reasoning capability and computational resource requirements.

\subsection{Stability in Extended Interaction Horizons}
\label{sec:extended_horizons}

In our primary evaluation, we set the maximum number of interaction turns to 10. This horizon was driven by two empirical observations. First, from a user experience perspective, human patience for iteratively correcting a model averages around 10 turns. Second, standard prompt-based baselines suffer from ``context fatigue'' and hit a severe performance plateau after 5 to 10 turns due to the expanding context window. Therefore, a 10-turn limit serves as both a realistic boundary for human interaction and a perfect stress test to demonstrate how \ours{} breaks the baseline's stagnation bottleneck.

However, to rigorously address potential concerns regarding policy drift and algorithmic stability over significantly longer horizons, we conducted an extended evaluation on the MATH dataset for up to 20 continuous conversational turns using the Qwen3-0.6B model.

The results of this extended stress test are detailed in Table~\ref{tab:extended_horizon}. 

\begin{table}[h]
\centering
\caption{Extended horizon evaluation on the MATH dataset over 20 conversational turns (Accuracy \%).}
\label{tab:extended_horizon}
\begin{tabular}{lcccc}
\toprule
\textbf{Method} & \textbf{Turn 5} & \textbf{Turn 10} & \textbf{Turn 15} & \textbf{Turn 20} \\
\midrule
Baseline & 19.40 & 25.00 & 27.20 & 27.20 \\
\ours{} & 25.80 \textcolor{red}{(+6.40)} & 50.20 \textcolor{red}{(+25.20)} & 57.40 \textcolor{red}{(+30.20)} & 64.80 \textcolor{red}{(+37.60)} \\
\bottomrule
\end{tabular}
\end{table}

The empirical trajectory clearly highlights the limitations of standard in-context learning: the baseline completely stagnates after turn 15, yielding no further improvements as the massive context bloat overwhelms the model's reasoning capacity. In stark contrast, \ours{} effectively internalizes the continuous negative feedback directly into its parameters, completely bypassing context fatigue. Instead of stagnating or suffering from policy drift, \ours{} sustains a steady and robust performance climb, reaching an impressive 64.80\% accuracy by turn 20.

The extended-horizon results are consistent with the positive-net-progress regime of Theorem~\ref{thm:unified_bound}. They do not establish that its conditions hold at every empirical update. The conditional failure-mode analysis in Appendix~\ref{sec:cumulative_error_bound} explains why the probability mass of rejected outputs and the approximation costs affect the number of useful corrective updates.

\subsection{Integration with Process Reward Models (PRM)}
\label{sec:prm_integration}

While our primary evaluation utilizes binary Outcome Reward Models (ORMs) to provide feedback based on the final answer, \ours{} is mathematically not limited to binary or sparse rewards. As analyzed in our methodology, the derived optimal policy and the Conjugate Gradient update naturally accommodate any continuous scalar reward $r \in \mathbb{R}$. To demonstrate that \ours{} can extend to complex tasks requiring dense supervision, we evaluated its integration with a Process Reward Model (PRM).

In this experiment, we utilized GPT-5.2 as a PRM to evaluate step-by-step Chain-of-Thought (CoT) reasoning on the MATH dataset using the Qwen3-0.6B model. The adaptation mechanism was refined for dense feedback: instead of waiting for the final answer, if a specific reasoning step $s_t$ received negative feedback from the PRM, \ours{} applied the parameter update $\Delta\theta$ strictly to the logits of that flawed step prefix.

Table~\ref{tab:prm_results} also includes the supplementary five-level rubric comparison (Appendix~\ref{sec:reward_calibration}). The rubric achieves 51.80\% accuracy and 6.7 average turns, close to the continuous outcome reward result. The PRM uses a different judge and step-level supervision, so its comparison is not a controlled ablation of the rubric prompt alone.

\begin{table}[h]
\centering
\caption{Performance comparison of \ours{} using Outcome Reward Models (ORM) versus Process Reward Models (PRM) on the MATH dataset using Qwen3-0.6B.}
\label{tab:prm_results}
\begin{tabular}{llcc}
\toprule
\textbf{Method} & \textbf{Reward Granularity} & \textbf{Final Acc. (\%)} & \textbf{Avg. Turns to Solve} \\
\midrule
Baseline & None & 26.00 & 8.9 \\
\ours{} (Rule ORM) & Final Outcome Reward & 50.20 & 6.9 \\
\ours{} (LLM ORM) & Continuous Outcome Reward & 52.20 & 6.7 \\
\ours{} (Rubric) & Five-Level Outcome Reward & 51.80 & 6.7 \\
\ours{} (PRM) & Step-level Process Reward & 60.40 \textcolor{red}{(+10.20)} & 5.8 \\
\bottomrule
\end{tabular}
\end{table}

As demonstrated in Table~\ref{tab:prm_results}, while the standard ORM-based \ours{} already provides a massive improvement over the baseline, integrating step-level supervision elevates the performance even further. \ours{}-PRM accelerates the adaptation process by isolating and correcting specific logical errors mid-generation, boosting the final accuracy to 60.40\% while reducing the average number of turns required to solve the problem from 6.9 down to 5.8. This smoothly demonstrates that \ours{} can efficiently harness fine-grained process rewards to achieve superior alignment in fewer rounds.

\subsection{General Knowledge Retention}
\label{sec:knowledge_retention}

A valid concern regarding test-time parameter adaptation is the potential violation of stationarity assumptions, which could lead to model instability, calibration degradation, or catastrophic forgetting of general knowledge. We clarify that \ours{} is intrinsically designed to prevent such degradation through three structural safeguards:

\begin{itemize}
    \item \textbf{Ephemeral Updates:} Parameter adaptations in \ours{} are strictly instance-specific. The model weights are immediately reset to their pre-trained base state after the interaction for each specific test sample terminates, ensuring that the global stationarity of the model is fully preserved across different tasks.
    \item \textbf{Sparse Parameter Updates:} \ours{} confines its updates strictly to the LM head via a rank-1 LoRA or modifies the transient hidden states. This carefully targeted mechanism alters merely $\sim$0.002\% of the model's parameters, leaving 99.998\% of the foundational weights completely frozen and intact.
    \item \textbf{Theoretical Constraint:} As defined in our methodology, the closed-form optimal policy is strictly regularized by a KL divergence penalty ($\beta$). This explicitly anchors the updated policy to the pre-trained prior, mathematically preventing unconstrained semantic drift.
\end{itemize}

To empirically validate these safeguards and confirm that repeated online updates do not degrade the model, we conducted a rigorous stress test. We forced the Qwen3-0.6B model through 10 consecutive \ours{} updates on unresolved problems from the MATH dataset. We then evaluated this heavily adapted model state on MMLU (to measure general knowledge retention) and Wikitext (to measure language modeling calibration via perplexity).

\begin{table}[h]
\centering
\caption{General knowledge retention stress test on Qwen3-0.6B. The model was evaluated after undergoing 10 consecutive \ours{} updates on unresolved math problems.}
\label{tab:knowledge_retention}
\resizebox{\textwidth}{!}{%
\begin{tabular}{lccc}
\toprule
\textbf{Model State} & \textbf{MMLU (Zero-shot Acc. \%)} & \textbf{Wikitext Perplexity} & \textbf{Average KL Divergence} \\
\midrule
Base Model & 51.03 & 18.2037 & 0.0 \\
\ours{} Adapted Model (10 updates) & 51.36 & 18.2045 & 0.001 \\
\bottomrule
\end{tabular}
}
\end{table}

As demonstrated in Table~\ref{tab:knowledge_retention}, general knowledge and calibration remain practically perfectly intact even after extensive test-time adaptation. The MMLU accuracy remains highly stable, the perplexity on Wikitext is virtually unchanged, and the average KL divergence from the base model is incredibly minimal (0.001). This definitively confirms that \ours{} safely enhances reasoning capabilities without sacrificing the foundational knowledge or stability of the LLM.

\subsection{Exploration with Multi-Sample Feedback}
\label{sec:multi_sample_feedback}

A potential question regarding the \ours{} mechanism is whether updating based solely on a single rejected response could lead to inefficient exploration, potentially increasing the probability of other random, incorrect responses. 

\textbf{Theoretical Justification:} 
Penalizing a single incorrect response does not lead to unconstrained or inefficient exploration. Recent findings (e.g., Zhu et al., \textit{``The Surprising Effectiveness of Negative Reinforcement in LLM Reasoning,''} NeurIPS 2025) demonstrate that suppressing an incorrect generation efficiently redistributes probability mass toward other plausible, logically coherent candidates. This redistribution is strongly guided by the LLM's robust pre-trained prior, rather than scattering towards random or nonsensical outputs.

\textbf{Multi-Sample Feedback (\ours{}-Multi):}
To empirically explore whether broader feedback improves adaptation, we implemented a multi-sample variant, \ours{}-Multi. In this setup, the model generates $M=4$ candidate responses per turn. If all candidates are incorrect, we compute the parameter update $\Delta\theta$ for each individual failure and apply the averaged update to the model. This approach acts as a Monte Carlo estimation of the policy optimization, effectively reducing variance.

The results of this experiment on the MATH dataset using Qwen3-0.6B are shown in Table~\ref{tab:multi_sample}.

\begin{table}[h]
\centering
\caption{Comparison of single-sample ($M=1$) versus multi-sample ($M=4$) feedback on the MATH dataset using Qwen3-0.6B.}
\label{tab:multi_sample}
\resizebox{\textwidth}{!}{%
\begin{tabular}{llccc}
\toprule
\textbf{Method} & \textbf{Feedback Samples / Turn} & \textbf{Final Acc. (\%)} & \textbf{Avg. Turns to Solve} & \textbf{Avg. Time to Solve} \\
\midrule
Baseline & -- & 26.00 & 8.9 & 1.0x \\
\ours{}-Single (Ours) & $M=1$ & 50.20 & 6.7 & $\sim$1.8x \\
\ours{}-Multi & $M=4$ & 52.60 \textcolor{red}{(+2.40)} & 6.5 & $\sim$6.9x \\
\bottomrule
\end{tabular}
}
\end{table}

\textbf{Analysis of the Trade-off:}
As Table~\ref{tab:multi_sample} demonstrates, \ours{}-Multi promotes slightly faster search and finds the correct answer in marginally fewer turns (6.5 vs. 6.7), leading to a higher final accuracy (52.60\%). However, this marginal gain comes at the cost of significantly higher latency ($\sim$6.9x compared to the baseline). Because multi-turn interactions inherently require real-time responsiveness, standard \ours{} ($M=1$) is deliberately designed to prioritize low-latency, real-world conversational applicability while still capturing the vast majority of the adaptation benefits.

\subsection{Additional Empirical Results}
\label{sec:additional_result}

This section presents supplementary empirical results to further validate our findings. First, Table~\ref{tab:additional_benchmarks} reports the performance of all models on three benchmarks—\texttt{AIME25}, \texttt{GPQA-diamond}, and \texttt{M-IMO}—which were omitted from the main text due to space constraints. Second, to provide a more complete picture of model performance, Table~\ref{tab:deepseek_qwen_results} details the \textit{Accuracy} and \textit{Correction Uplift} for the \texttt{DeepSeek-R1-Distill-Qwen-7B} model on both mathematical reasoning and code generation datasets. Across these additional results, a clear and consistent trend emerges: reinforcing the conclusions from our main analysis, our proposed method, \ours{}, significantly enhances both overall task performance and the capacity of model for self-correction.

\begin{table*}[ht]
\centering
\caption{Performance of the \textbf{DeepSeek-R1-Distill-Qwen-8B} model on mathematical reasoning and code generation datasets. The values in \textcolor{darkred}{red} indicate the absolute improvement of \ours{} over the baseline.}
\label{tab:deepseek_qwen_results}
\resizebox{\textwidth}{!}{%
\begin{tabular}{l cc cc cc cc}
\toprule
& \multicolumn{2}{c}{\textbf{MATH}} & \multicolumn{2}{c}{\textbf{AIME25}} & \multicolumn{2}{c}{\textbf{MATH-500}} & \multicolumn{2}{c}{\textbf{HumanEval}} \\
\cmidrule(lr){2-3} \cmidrule(lr){4-5} \cmidrule(lr){6-7} \cmidrule(lr){8-9}
\textbf{Method} & \textbf{Final Acc. $\uparrow$} & \textbf{Correction Uplift $\uparrow$} & \textbf{Final Acc. $\uparrow$} & \textbf{Correction Uplift $\uparrow$} & \textbf{Final Acc. $\uparrow$} & \textbf{Correction Uplift $\uparrow$} & \textbf{Final Acc. $\uparrow$} & \textbf{Correction Uplift $\uparrow$} \\
\midrule
Baseline & 7.60 & 3.14 & 10.00 & 3.57 & 7.40 & 6.09 & 45.12 & 17.05 \\
\ours{} & \textbf{9.80} \gain{2.20} & \textbf{5.65} \gain{2.51} & \textbf{16.67} \gain{6.67} & \textbf{16.67} \gain{13.10} & \textbf{22.20} \gain{14.80} & \textbf{18.62} \gain{12.53} & \textbf{51.22} \gain{6.10} & \textbf{33.75} \gain{16.70} \\
\bottomrule
\end{tabular}
}
\end{table*}

\begin{table*}[ht]
\centering
\caption{Supplementary performance results on additional benchmarks, reporting accuracy (\%). The values in \textcolor{darkred}{red} indicate the absolute improvement of \ours{} variants over the baseline.}
\label{tab:additional_benchmarks}
\resizebox{\textwidth}{!}{%
\begin{tabular}{ll ccc}
\toprule
& & \textbf{Mathematical Reasoning} & \textbf{General Reasoning} & \textbf{Multilingual Reasoning} \\
\cmidrule(lr){3-3} \cmidrule(lr){4-4} \cmidrule(lr){5-5}
\textbf{Model} & \textbf{Method} & \textbf{AIME25} & \textbf{GPQA-diamond} & \textbf{M-IMO} \\
\midrule
\multirow{4}{*}{\shortstack{Qwen2.5-0.5B \\ -Instruct}} & Baseline & 3.33 & 3.54 & 1.99 \\
& \ours{} (+LM + R) & \textbf{6.67} \gain{3.34} & 7.07 \gain{3.53} & 2.09 \gain{0.10} \\ 
& \ours{} (+HS + R) & \textbf{6.67} \gain{3.34} & 8.53 \gain{4.99} & 3.20 \gain{1.21} \\
& \ours{} (+LM + M) & \textbf{6.67} \gain{3.34} & 10.27 \gain{6.73} & 4.71 \gain{2.72} \\
\midrule
\multirow{4}{*}{Qwen3-0.6B} & Baseline & 10.00 & 12.20 & 5.20 \\
& \ours{} (+LM + R) & \textbf{16.67} \gain{6.67} & 9.09 \gain{-3.11} & 5.30 \gain{0.10} \\ 
& \ours{} (+HS + R) & 10.00 \gain{0.00} & 10.54 \gain{-1.66} & 5.30 \gain{0.10} \\
& \ours{} (+LM + M) & 10.00 \gain{0.00} & 13.16 \gain{0.96} & 6.60 \gain{1.40} \\
\midrule
\multirow{4}{*}{\shortstack{Qwen2.5-7B \\ -Instruct}} & Baseline & 10.00 & 26.14 & 10.53 \\
& \ours{} (+LM + R) & \textbf{23.33} \gain{13.33} & 42.24 \gain{16.10} & 17.57 \gain{7.04} \\ 
& \ours{} (+HS + R) & 20.00 \gain{10.00} & 43.16 \gain{17.02} & 18.36 \gain{7.83} \\
& \ours{} (+LM + M) & 20.00 \gain{10.00} & 45.83 \gain{19.69} & 21.21 \gain{10.68} \\
\midrule
\multirow{4}{*}{Qwen3-8B} & Baseline & 16.67 & 41.16 & 20.37 \\
& \ours{} (+LM + R) & 30.00 \gain{13.33} & 69.11 \gain{27.95} & 33.17 \gain{12.80} \\ 
& \ours{} (+HS + R) & 33.33 \gain{16.66} & 70.27 \gain{29.11} & 37.62 \gain{17.25} \\
& \ours{} (+LM + M) & \textbf{36.67} \gain{20.00} & 75.18 \gain{34.02} & 39.16 \gain{18.79} \\
\midrule
\multirow{4}{*}{\shortstack{DeepSeek-R1 \\-Distill-Llama-8B}} & Baseline & 3.33 & 19.03 & 4.36 \\
& \ours{} (+LM + R) & \textbf{16.67} \gain{13.34} & 21.14 \gain{2.11} & 6.32 \gain{1.96} \\ 
& \ours{} (+HS + R) & \textbf{16.67} \gain{13.34} & 22.23 \gain{3.20} & 5.17 \gain{0.81} \\
& \ours{} (+LM + M) & \textbf{16.67} \gain{13.34} & 25.36 \gain{6.33} & 6.36 \gain{2.00} \\
\bottomrule
\end{tabular}%
}
\end{table*}

\subsection{Comparison with Multi-Turn Training Methods}
\label{sec:comparison_with_training}

While our main analysis focuses on test-time adaptation, it is instructive to compare \ours{} with traditional training-based methods for multi-turn dialogue. In this section, we benchmark the performance of \ours{} against two such paradigms on the \texttt{MATH} dataset: Supervised Fine-Tuning (SFT) and Reinforcement Learning (RL).

For the SFT baseline, we first generated a multi-turn dialogue dataset using \texttt{DeepSeek-R1} on the \texttt{MATH} training set, and then fine-tuned the base model on this newly created data. For the RL baseline~\citep{Sheng_2025}, we employed a Group Preference Optimization (GRPO) scheme tailored for multi-turn dialogue, similar to the approach described in our related work (Appendix~\ref{sec:related_work}).

The results, presented in Table~\ref{tab:train_methdo_result}, report both \textit{Accuracy} and \textit{Correction Uplift}. The key finding is that \ours{}, a purely test-time method, achieves performance that is comparable or even superior to these training-based approaches. This highlights a significant advantage of our method: it obviates the need for expensive data collection and resource-intensive model training, offering a more efficient and flexible solution for enhancing multi-turn capabilities.

\begin{table}[ht]
\centering
\caption{Comparison of \ours{} with training-based methods on the \textbf{MATH} dataset for the \textbf{Qwen3-8B} model. Our test-time method achieves performance comparable to full Reinforcement Learning (RL) training and surpasses Supervised Fine-Tuning (SFT), with zero additional offline training cost.}
\label{tab:train_methdo_result}
\begin{tabular}{lcc}
\toprule
\textbf{Method} & \textbf{Final Acc. $\uparrow$} & \textbf{Correction Uplift $\uparrow$} \\
\midrule
Baseline & 55.80 & 23.00 \\
SFT Training & 63.80 & 39.24 \\
RL Training & \textbf{66.20} & \textbf{40.45} \\
\ours{} & \underline{65.80} & \underline{40.42} \\
\bottomrule
\end{tabular}
\end{table}

\subsection{Detailed Per-Turn Correction Analysis}
\label{sec:per_turn_analysis}

In this section, we provide a fine-grained analysis of the error-correction dynamics by examining the distribution of newly solved problems across individual interaction turns. Figure~\ref{fig:result_2} illustrates the number of problems that were initially incorrect but were successfully resolved in each specific round ($k=1$ to $k=10$).

\begin{figure*}[h]
    \centering
    \includegraphics[width=1.0\textwidth]{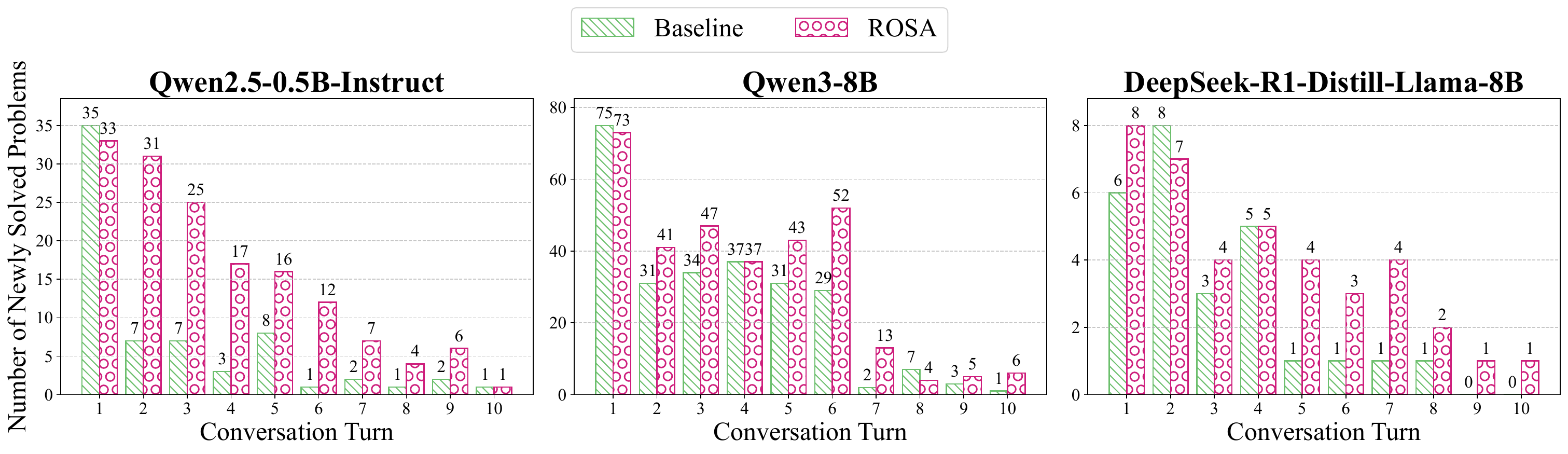} 
    \vspace{-4mm}
    \caption{
    Detailed distribution of newly solved problems per round on the MATH dataset across various models.
    }
    \label{fig:result_2}
\end{figure*}

\textbf{The Stagnation of Baselines:} 
The empirical results reveal a critical limitation of standard in-context learning. For the baseline models (green bars), the vast majority of successful corrections occur in the very first turn ($k=1$). After this initial attempt, the frequency of new corrections drops precipitously, often hitting zero or near-zero by the third or fourth turn. This indicates that if a prompt-based model fails to correct itself immediately, it is highly unlikely to do so later, as the accumulating conversational context fails to provide new structural information for policy shifting.

\textbf{The Long-Tail Correction of \ours{}:}
In contrast, \ours{} (purple bars) demonstrates a robust and sustained ``correction tail.'' Not only does it solve significantly more problems in the first turn, but it continues to resolve a substantial number of difficult problems in later turns ($k=5$ to $k=10$). This sustained performance highlights the fundamental advantage of test-time parameter adaptation: by mathematically shifting the policy distribution based on negative rewards, \ours{} actively explores new reasoning paths that the base model's static distribution could not reach. This continuous learning behavior confirms that \ours{} does not just rely on context, but effectively internalizes the user's corrective feedback to escape local reasoning failures.

\subsection{Learning from Successful Trajectories}
\label{sec:successful_trajectories}
Although the primary correctness protocol focuses on correcting failed attempts and resets between tasks, the practical target supports both positive and negative scalar feedback. We evaluate learning from successful trajectories on BigCodeBench-Hard with Qwen3-8B and at most 10 turns per task. Table~\ref{tab:successful_trajectories} distinguishes instance-specific updates from persistent updates retained across tasks.
\begin{table}[ht]
\centering\small
\caption{Learning from successful trajectories on BigCodeBench-Hard.}\label{tab:successful_trajectories}
\begin{tabular}{llr}
\toprule Adaptation scope & Feedback used & Success (\%)\\\midrule
No parameter updates & None & 3\\
Instance-specific, reset between tasks & Negative only & 14\\
Persistent across tasks & Negative only & 16\\
Persistent across tasks & Positive and negative & 19\\\bottomrule
\end{tabular}
\end{table}
Holding the persistent-update setting fixed, incorporating successful trajectories raises success from 16\% to 19\%. The 14\% and 16\% comparison separately examines retaining updates across tasks. This supplementary persistent setting differs from the default reset protocol and does not by itself establish long-term knowledge retention.

\subsection{Adapting to Individual Preferences on ALOE}
\label{sec:aloe}
We use ALOE, introduced in \citet{rosa-wu-2025-aloe}, to complement the verifiable correctness evaluations with individual-preference feedback. The evaluation uses 100 human-verified personas, with ten interaction turns per persona. The persona is hidden from the responding model, which must adapt through interaction. Qwen3-8B is evaluated against the standard multi-turn baseline under the same judging protocol. GPT-5.6 assigns Alignment Level (AL) scores $s\in\{1,2,3,4,5\}$; \ours{} maps these to online rewards $r=(s-3)/2$ and incorporates positive and negative feedback. Interactions continue for all ten turns, and adaptation is reset between personas.
\begin{figure}[htbp]
\centering
\includegraphics[width=0.78\linewidth]{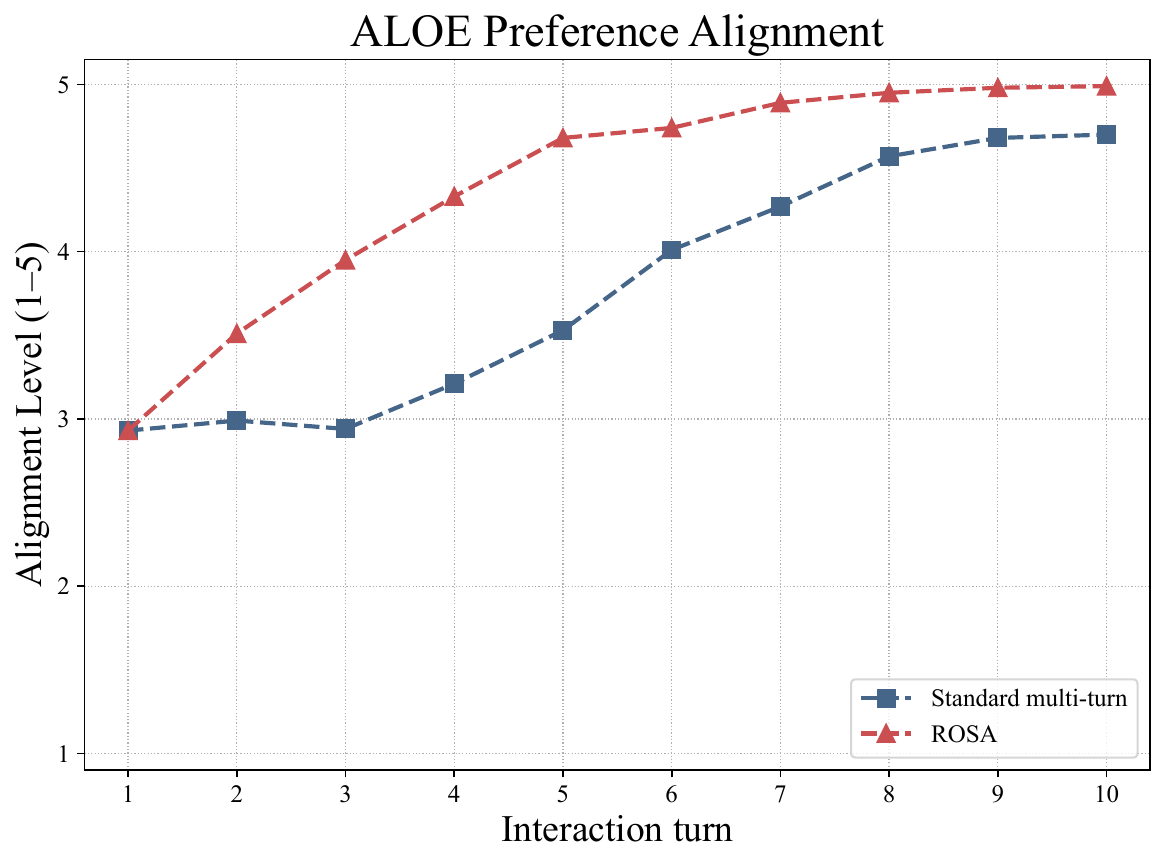}
\caption{Preference alignment on ALOE using Qwen3-8B. Mean Alignment Level (1--5) across 100 personas is shown at each interaction turn. \ours{} improves alignment more rapidly than standard multi-turn interaction, reaching 4.99 versus 4.70 at turn 10. Both methods use the same judging protocol; details are in Appendix~\ref{sec:aloe}.}
\label{fig:aloe_alignment}
\end{figure}

\begin{table}[ht]
\centering\small
\caption{ALOE preference-satisfaction scores (1--5) at each turn. The mean is over ten turns.}\label{tab:aloe_turns}
\resizebox{\linewidth}{!}{\begin{tabular}{lrrrrrrrrrrr}
\toprule Method & 1 & 2 & 3 & 4 & 5 & 6 & 7 & 8 & 9 & 10 & Mean\\\midrule
Baseline & 2.93 & 2.99 & 2.94 & 3.21 & 3.53 & 4.01 & 4.27 & 4.57 & 4.68 & 4.70 & 3.78\\
\ours{} & 2.93 & 3.51 & 3.95 & 4.33 & 4.68 & 4.74 & 4.89 & 4.95 & 4.98 & 4.99 & 4.40\\\bottomrule
\end{tabular}}
\end{table}
Both methods start at 2.93. \ours{} reaches 4.33 by turn 4 and 4.99 by turn 10, compared with 3.21 and 4.70 for the baseline, improving the reported mean by 0.62. These results complement correctness-based feedback: changing the evaluation signal enables adaptation to an additional preference objective without changing the core update rule. The scores are model-judged; the human rubric calibration in Appendix~\ref{sec:reward_calibration} concerns MATH, not this persona evaluation.

\section{Ablation Studies}
\label{sec:ablation_studies}

\subsection{The Importance of the Optimization Strategy}
\label{sec:ablation_strategy}

To isolate the contribution of our proposed optimization method, we conduct an ablation study comparing the full \ours{} framework against a more direct reinforcement learning approach. This baseline, which we term \textit{RL}, directly optimizes the standard RLHF objective function in Equation~\eqref{eq:rlhf_objective}). To simulate the online, multi-turn interaction setting in a comparable manner to \ours{}, we estimate the gradient of $J(\pi_\theta)$ using only a single response $\mathbf{y}$ sampled from the policy $\pi_\theta$ for each prompt $\mathbf{x}$, and then update the model's parameters using this gradient. This approach contrasts with our full \ours{} framework, which first computes a stable target policy $\pi^*$ and then solves for the parameter update $\Delta\theta$.

The results of this comparison are presented in Figure \ref{fig:result_1}. The analysis leads to two clear observations. First, the direct RL optimization (dotted lines) yields only marginal improvements over the baseline models (solid lines) across all three datasets. The proximity of the solid and dotted lines indicates that a naive policy gradient update with a single sample provides a noisy and inefficient learning signal, resulting in minimal performance gains. Second, in stark contrast, \ours{} (dashed lines) consistently and significantly outperforms both the baseline and the RL-enhanced version. The steeper slopes of the dashed lines demonstrate that \ours{} not only achieves a higher absolute accuracy but also accelerates the error correction process over the conversation turns. For example, on the MATH dataset, the Qwen3-8B model enhanced with \ours{} shows a much more rapid accuracy improvement compared to its \textit{RL} counterpart.

The quantitative results of this comparison, presented in Table~\ref{tab:main_performance_math_accuracy} and Table~\ref{tab:main_performance_math_uplift}, demonstrate a clear and consistent advantage for \ours{}. Table~\ref{tab:main_performance_math_accuracy} reveals that \ours{} achieves substantially higher final accuracy across all models and datasets. For instance, on the \texttt{MATH} dataset with the Qwen3-0.6B model, \ours{} surpasses the RL baseline by a remarkable \textbf{+24.00\%}. Furthermore, Table~\ref{tab:main_performance_math_uplift} highlights its superior self-correction capability. In the most significant case, \ours{} boosts the Correction Uplift score by \textbf{+31.31\%} on \texttt{MATH-500} for the same model. The data consistently show that a direct RL update provides only marginal benefits, while our principled optimization strategy yields significant gains in both overall success and the ability to recover from errors.

This ablation study confirms that the superior performance of \ours{} is not merely due to the introduction of an online reward signal. Rather, it is the principled optimization strategy---deriving a stable online target $\pi^*$ and then efficiently solving for the optimal parameter update $\Delta\theta$---that is crucial for achieving effective and efficient test-time adaptation.

\begin{table}[t!]
\centering
\caption{Comparison of Accuracy (\%) on mathematical reasoning datasets with RL and \ours{}.}
\label{tab:main_performance_math_accuracy}
\resizebox{0.8\textwidth}{!}{%
\begin{tabular}{ll cccc}
\toprule
\textbf{Model} & \textbf{Method} & \textbf{MATH} & \textbf{MATH-500} & \textbf{AIME25} & \textbf{HumanEval} \\
\midrule
\multirow{2}{*}{Qwen3-0.6B} & RL & 26.20 & 28.80 & 10.00 & 42.68 \\
& \ours{} & \textbf{50.20} \gain{24.00} & \textbf{51.60} \gain{22.80} & \textbf{16.67} \gain{6.67} & \textbf{45.73} \gain{3.05} \\
\midrule
\multirow{2}{*}{Qwen3-8B} & RL & 59.60 & 63.60 & 16.67 & 79.27 \\
& \ours{} & \textbf{65.80} \gain{6.20} & \textbf{72.80} \gain{9.20} & \textbf{30.00} \gain{13.33} & \textbf{81.71} \gain{2.44} \\
\midrule
\multirow{2}{*}{\shortstack{DeepSeek-R1-Distill\\-Llama-8B}} & RL & 6.20 & 8.40 & 10.00 & 28.05 \\
& \ours{} & \textbf{7.80} \gain{1.60} & \textbf{18.40} \gain{10.00} & \textbf{16.67} \gain{6.67} & \textbf{39.02} \gain{10.97}\\
\bottomrule
\end{tabular}
}
\end{table}

\begin{table}[t!]
\centering
\caption{Comparison of Correction Uplift (\%) on mathematical reasoning datasets with RL and \ours{}.}
\label{tab:main_performance_math_uplift}
\resizebox{0.8\textwidth}{!}{%
\begin{tabular}{ll cccc}
\toprule
\textbf{Model} & \textbf{Method} & \textbf{MATH} & \textbf{MATH-500} & \textbf{AIME25} & \textbf{HumanEval} \\
\midrule
\multirow{2}{*}{Qwen3-0.6B} & RL & 18.54 & 20.00 & 6.90 & 22.31\\
& \ours{} & \textbf{48.87} \gain{30.33} & \textbf{51.31} \gain{31.31} & \textbf{16.67} \gain{9.77} & \textbf{31.01} \gain{8.70} \\
\midrule
\multirow{2}{*}{Qwen3-8B} & RL & 29.37 & 31.58 & 10.71 & 50.00 \\
& \ours{} & \textbf{40.42} \gain{11.05} & \textbf{52.94} \gain{21.36} & \textbf{27.59} \gain{16.88} & \textbf{62.50} \gain{12.50} \\
\midrule
\multirow{2}{*}{\shortstack{DeepSeek-R1-Distill\\-Llama-8B}} & RL & 4.87 & 7.29 & 6.90 & 16.31 \\
& \ours{} & \textbf{6.30} \gain{1.43} & \textbf{17.41} \gain{10.12} & \textbf{13.79} \gain{6.89} & \textbf{31.97} \gain{15.66}\\
\bottomrule
\end{tabular}
}
\end{table}

\subsection{Ablation Study on the Influence of Hyperparameter $\beta$}
\label{sec:ablation_beta}

\textbf{Experimental Setup and the Role of $\beta$.} 
To investigate the sensitivity of our proposed method to its hyperparameters, we conduct an ablation study on the regularization coefficient $\beta$. We vary its value across a wide range of $[0.25, 1.75]$ to observe its impact on model performance. As defined in the standard RLHF objective, $\beta$ controls the trade-off between maximizing the reward and maintaining proximity to the reference policy. In the \ours{} framework, its role is to modulate the intensity of the policy update based on the reward signal $r(\mathbf{x}, \mathbf{y})$, as formulated in our practical update target in Equation~\eqref{eq:practical_target}.A smaller $\beta$ amplifies the reward signal, leading to more aggressive updates, while a larger $\beta$ dampens it, resulting in more conservative updates.


\textbf{Analysis and Conclusions.}
The results of our study are presented in Figure~\ref{fig:beta_ablation}, which illustrates the cumulative accuracy over 10 conversational turns for each tested $\beta$ value. A key observation is that while the initial learning trajectories vary---with smaller $\beta$ values often yielding a steeper initial performance gain---all configurations converge to a similar final accuracy. This convergence can be attributed to the iterative nature of the multi-turn interaction. Although $\beta$ adjusts the magnitude of each corrective step, the consistent directional feedback provided by the reward signal ensures that the model is always guided towards an improved policy. Consequently, over a sufficient number of turns, even a series of conservative updates can accumulate to achieve the correct solution.

From this analysis, we draw two key conclusions. First, for tasks with definitive solutions, such as mathematical reasoning, different search strategies---ranging from aggressive to conservative---are all highly likely to converge to the correct solution given adequate opportunities for self-correction. Second, this study underscores the \textbf{robustness of the \ours{} framework}. The model's final performance demonstrates low sensitivity to the choice of $\beta$ across a wide operational range, indicating that \ours{} can achieve stable and effective results without extensive hyperparameter tuning.

\begin{figure}[ht]
    \centering
    \includegraphics[width=0.8\textwidth]{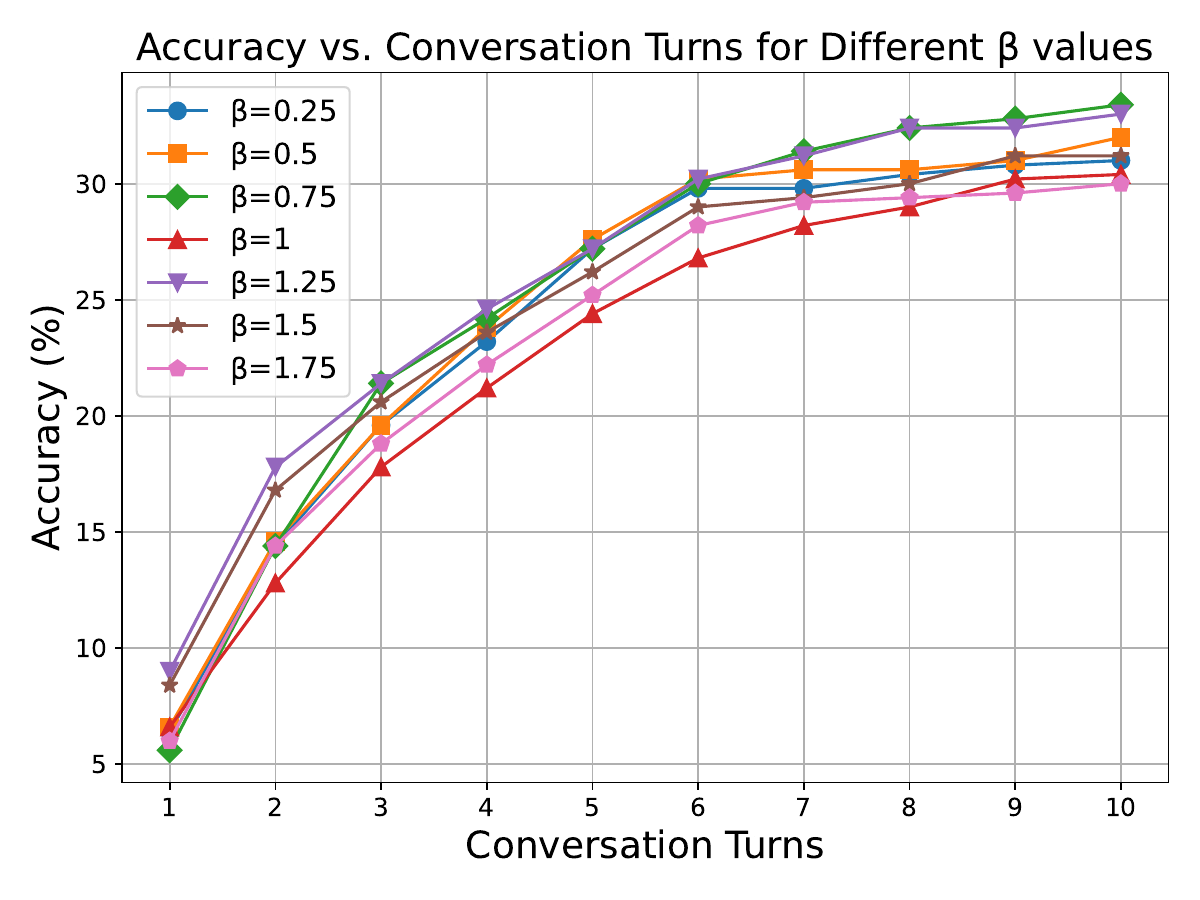} 
    \vspace{-1mm}
    \caption{
        \textbf{Ablation study of the hyperparameter $\beta$ on the MATH dataset.} The figure illustrates the cumulative accuracy over 10 conversational turns for different values of $\beta$, ranging from 0.25 to 1.75.
    }
    \label{fig:beta_ablation}
    \vspace{-1mm}
\end{figure}

\section{Case Study} \label{sec:case_study}

\begin{tcolorbox}[title={\textbf{Qwen3-0.8B, MATH dataset}},
colback=whitesmoke, colframe=darksalmon,  boxrule=2pt, arc=0mm]
{\scriptsize
\begin{lstlisting}[style=mystyle]
Question: Convert the point (0,3) in rectangular coordinates to polar coordinates. 
Enter your answer in the...

True answer: \left( 3, \frac{\pi}{2} \right)

# --- Turn 1 ---
    Baseline answer: 'Okay, so I need to convert the point (0, 3) from rectangular 
    coordinates to polar coordinates. Hmm, let me remember how to do this....
    So putting it all together, the polar coordinates should be (3, pi/2). 
    Let me just'
    
    ROSA answer: 'Okay, so I need to convert the point (0, 3) from rectangular
    coordinates to polar coordinates. Hmm, let me remember how to do this....
    So putting it all together, the polar coordinates should be (3, pi/2). 
    Let me just'

# --- Turn 2 ---

    Baseline answer: 'Okay, let me try to figure this out again. The original 
    point is (0, 3). In rectangular coordinates....Therefore, after checking,
    I think the correct answer is (3, pi/2). However, the user initially 
    thought it was wrong, but after rethinking, it's correct.'
    
    ROSA answer: 'Okay, so I got confused earlier. The original point is (0, 3)
    . If I use polar coordinates,....Alternatively, if someone thought that 
    theta is measured as the angle between the x-axis and the'

# --- Turn 3 ---
    Baseline answer: 'Okay, so the user initially thought that converting 
    (0,3) to polar coordinates woul....Is there any possibility that theta 
    could be negative? No, because theta is always between 0 and 2pi,
    and since (0,3) is in the first quadrant'
    
    ROSA answer: 'Okay, so the user initially thought the answer was 
    (3, pi/2) but the second time they said "Wrong answer....
    **Final Answer**
    The polar coordinates of the point (0, 3) 
    are \boxed{(3, \frac{\pi}{2})}.'

\end{lstlisting}
}
\end{tcolorbox}


\end{document}

%% file: math_commands.tex
\usepackage{amsmath,amsfonts,bm}

\def\eqref#1{equation~\ref{#1}}

\def\1{\bm{1}}

\DeclareMathAlphabet{\mathsfit}{\encodingdefault}{\sfdefault}{m}{sl}
\SetMathAlphabet{\mathsfit}{bold}{\encodingdefault}{\sfdefault}{bx}{n}



%% file: iclr2027_conference.bib
@inproceedings{hu2024-localized,
  title={Localized zeroth-order prompt optimization},
  author={Hu, Wenyang and Shu, Yao and Yu, Zongmin and Wu, Zhaoxuan and Lin, Xiangqiang and Dai, Zhongxiang and Ng, See-Kiong and Low, Bryan Kian Hsiang},
  booktitle={The Thirty-Eighth Conference on Neural Information Processing Systems \normalfont (\textbf{NeurIPS Spotlight})},
  year={2024}
}

@misc{li2025singleturnsurveymultiturninteractions,
      title={Beyond Single-Turn: A Survey on Multi-Turn Interactions with Large Language Models}, 
      author={Yubo Li and Xiaobin Shen and Xinyu Yao and Xueying Ding and Yidi Miao and Ramayya Krishnan and Rema Padman},
      year={2025},
      eprint={2504.04717},
      archivePrefix={arXiv},
      primaryClass={cs.CL},
      url={https://arxiv.org/abs/2504.04717}, 
}

@misc{yi2025surveyrecentadvancesllmbased,
      title={A Survey on Recent Advances in LLM-Based Multi-turn Dialogue Systems}, 
      author={Zihao Yi and Jiarui Ouyang and Zhe Xu and Yuwen Liu and Tianhao Liao and Haohao Luo and Ying Shen},
      year={2025},
      eprint={2402.18013},
      archivePrefix={arXiv},
      primaryClass={cs.CL},
      url={https://arxiv.org/abs/2402.18013}, 
}

@inproceedings{overcoming_miscalibrated_conversational,
author = {Herlihy, Christine and Neville, Jennifer and Schnabel, Tobias and Swaminathan, Adith},
title = {On overcoming miscalibrated conversational priors in LLM-based chatbots},
year = {2024},
publisher = {JMLR.org},
booktitle = {Proceedings of the Fortieth Conference on Uncertainty in Artificial Intelligence},
articleno = {75},
numpages = {22},
location = {Barcelona, Spain},
series = {UAI '24}
}

@inproceedings{chen-etal-2023-chatcot,
    title = "{C}hat{C}o{T}: Tool-Augmented Chain-of-Thought Reasoning on Chat-based Large Language Models",
    author = "Chen, Zhipeng  and
      Zhou, Kun  and
      Zhang, Beichen  and
      Gong, Zheng  and
      Zhao, Xin  and
      Wen, Ji-Rong",
    editor = "Bouamor, Houda  and
      Pino, Juan  and
      Bali, Kalika",
    booktitle = "Findings of the Association for Computational Linguistics: EMNLP 2023",
    month = dec,
    year = "2023",
    address = "Singapore",
    publisher = "Association for Computational Linguistics",
    url = "https://aclanthology.org/2023.findings-emnlp.985/",
    doi = "10.18653/v1/2023.findings-emnlp.985",
    pages = "14777--14790"
}

@inproceedings{
zheng2023judging,
title={Judging {LLM}-as-a-Judge with {MT}-Bench and Chatbot Arena},
author={Lianmin Zheng and Wei-Lin Chiang and Ying Sheng and Siyuan Zhuang and Zhanghao Wu and Yonghao Zhuang and Zi Lin and Zhuohan Li and Dacheng Li and Eric Xing and Hao Zhang and Joseph E. Gonzalez and Ion Stoica},
booktitle={Thirty-seventh Conference on Neural Information Processing Systems Datasets and Benchmarks Track},
year={2023},
url={https://openreview.net/forum?id=uccHPGDlao}
}

@misc{zhang2025surveymultiturninteractioncapabilities,
      title={A Survey on Multi-Turn Interaction Capabilities of Large Language Models}, 
      author={Chen Zhang and Xinyi Dai and Yaxiong Wu and Qu Yang and Yasheng Wang and Ruiming Tang and Yong Liu},
      year={2025},
      eprint={2501.09959},
      archivePrefix={arXiv},
      primaryClass={cs.CL},
      url={https://arxiv.org/abs/2501.09959}, 
}

@inproceedings{van-miltenburg-etal-2025-measure,
    title = "Measure only what is measurable: towards conversation requirements for evaluating task-oriented dialogue systems",
    author = "Van Miltenburg, Emiel  and
      Braggaar, Anouck  and
      Croes, Emmelyn  and
      Kunneman, Florian  and
      Liebrecht, Christine  and
      Martijn, Gabriella",
    editor = "Arviv, Ofir  and
      Clinciu, Miruna  and
      Dhole, Kaustubh  and
      Dror, Rotem  and
      Gehrmann, Sebastian  and
      Habba, Eliya  and
      Itzhak, Itay  and
      Mille, Simon  and
      Perlitz, Yotam  and
      Santus, Enrico  and
      Sedoc, Jo{\~a}o  and
      Shmueli Scheuer, Michal  and
      Stanovsky, Gabriel  and
      Tafjord, Oyvind",
    booktitle = "Proceedings of the Fourth Workshop on Generation, Evaluation and Metrics (GEM{\texttwosuperior})",
    month = jul,
    year = "2025",
    address = "Vienna, Austria and virtual meeting",
    publisher = "Association for Computational Linguistics",
    url = "https://aclanthology.org/2025.gem-1.18/",
    pages = "231--238",
    ISBN = "979-8-89176-261-9"
}

@misc{laban2025llmslostmultiturnconversation,
      title={LLMs Get Lost In Multi-Turn Conversation}, 
      author={Philippe Laban and Hiroaki Hayashi and Yingbo Zhou and Jennifer Neville},
      year={2025},
      eprint={2505.06120},
      archivePrefix={arXiv},
      primaryClass={cs.CL},
      url={https://arxiv.org/abs/2505.06120}, 
}

@inproceedings{wei-etal-2025-flexora,
    title = "Flexora: Flexible Low-Rank Adaptation for Large Language Models",
    author = "Wei, Chenxing  and
      Shu, Yao  and
      He, Ying Tiffany  and
      Yu, Fei",
    editor = "Che, Wanxiang  and
      Nabende, Joyce  and
      Shutova, Ekaterina  and
      Pilehvar, Mohammad Taher",
    booktitle = "Proceedings of the 63rd Annual Meeting of the Association for Computational Linguistics (Volume 1: Long Papers)",
    month = jul,
    year = "2025",
    address = "Vienna, Austria",
    publisher = "Association for Computational Linguistics",
    url = "https://aclanthology.org/2025.acl-long.713/",
    doi = "10.18653/v1/2025.acl-long.713",
    pages = "14643--14682",
    ISBN = "979-8-89176-251-0"
}

@article{sft,
author = {Chung, Hyung Won and Hou, Le and Longpre, Shayne and Zoph, Barret and Tai, Yi and Fedus, William and Li, Yunxuan and Wang, Xuezhi and Dehghani, Mostafa and Brahma, Siddhartha and Webson, Albert and Gu, Shixiang Shane and Dai, Zhuyun and Suzgun, Mirac and Chen, Xinyun and Chowdhery, Aakanksha and Castro-Ros, Alex and Pellat, Marie and Robinson, Kevin and Valter, Dasha and Narang, Sharan and Mishra, Gaurav and Yu, Adams and Zhao, Vincent and Huang, Yanping and Dai, Andrew and Yu, Hongkun and Petrov, Slav and Chi, Ed H. and Dean, Jeff and Devlin, Jacob and Roberts, Adam and Zhou, Denny and Le, Quoc V. and Wei, Jason},
title = {Scaling instruction-finetuned language models},
year = {2024},
issue_date = {January 2024},
publisher = {JMLR.org},
volume = {25},
number = {1},
issn = {1532-4435},
journal = {J. Mach. Learn. Res.},
month = jan,
articleno = {70},
numpages = {53}
}

@inproceedings{lester-etal-2021-power,
    title = "The Power of Scale for Parameter-Efficient Prompt Tuning",
    author = "Lester, Brian  and
      Al-Rfou, Rami  and
      Constant, Noah",
    editor = "Moens, Marie-Francine  and
      Huang, Xuanjing  and
      Specia, Lucia  and
      Yih, Scott Wen-tau",
    booktitle = "Proceedings of the 2021 Conference on Empirical Methods in Natural Language Processing",
    month = nov,
    year = "2021",
    address = "Online and Punta Cana, Dominican Republic",
    publisher = "Association for Computational Linguistics",
    url = "https://aclanthology.org/2021.emnlp-main.243/",
    doi = "10.18653/v1/2021.emnlp-main.243",
    pages = "3045--3059"
}

@inproceedings{
wei2025redit,
title={ReDit: Reward Dithering for Improved {LLM} Policy Optimization},
author={Chenxing Wei and Jiarui Yu and Ying Tiffany He and Hande Dong and Yao Shu and Fei Yu},
booktitle={2nd Workshop on Models of Human Feedback for AI Alignment},
year={2025},
url={https://openreview.net/forum?id=nDzBpkbxk7}
}

@misc{
shu2024ferret,
title={Ferret: Federated Full-Parameter Tuning at Scale for Large Language Models},
author={Yao Shu and Wenyang Hu and See-Kiong Ng and Bryan Kian Hsiang Low and Fei Yu},
year={2024},
url={https://openreview.net/forum?id=9H1uctBWgF}
}

@misc{chang2023surveyevaluationlargelanguage,
      title={A Survey on Evaluation of Large Language Models}, 
      author={Yupeng Chang and Xu Wang and Jindong Wang and Yuan Wu and Linyi Yang and Kaijie Zhu and Hao Chen and Xiaoyuan Yi and Cunxiang Wang and Yidong Wang and Wei Ye and Yue Zhang and Yi Chang and Philip S. Yu and Qiang Yang and Xing Xie},
      year={2023},
      eprint={2307.03109},
      archivePrefix={arXiv},
      primaryClass={cs.CL},
      url={https://arxiv.org/abs/2307.03109}, 
}

@misc{irvine2023rewardingchatbotsrealworldengagement,
      title={Rewarding Chatbots for Real-World Engagement with Millions of Users}, 
      author={Robert Irvine and Douglas Boubert and Vyas Raina and Adian Liusie and Ziyi Zhu and Vineet Mudupalli and Aliaksei Korshuk and Zongyi Liu and Fritz Cremer and Valentin Assassi and Christie-Carol Beauchamp and Xiaoding Lu and Thomas Rialan and William Beauchamp},
      year={2023},
      eprint={2303.06135},
      archivePrefix={arXiv},
      primaryClass={cs.CL},
      url={https://arxiv.org/abs/2303.06135}, 
}

@inproceedings{
hendrycks2021what,
title={What Would Jiminy Cricket Do? Towards Agents That Behave Morally},
author={Dan Hendrycks and Mantas Mazeika and Andy Zou and Sahil Patel and Christine Zhu and Jesus Navarro and Dawn Song and Bo Li and Jacob Steinhardt},
booktitle={Thirty-fifth Conference on Neural Information Processing Systems Datasets and Benchmarks Track (Round 2)},
year={2021},
url={https://openreview.net/forum?id=G1muTb5zuO7}
}

@inproceedings{
shinn2023reflexion,
title={Reflexion: language agents with verbal reinforcement learning},
author={Noah Shinn and Federico Cassano and Ashwin Gopinath and Karthik R Narasimhan and Shunyu Yao},
booktitle={Thirty-seventh Conference on Neural Information Processing Systems},
year={2023},
url={https://openreview.net/forum?id=vAElhFcKW6}
}

@inproceedings{
wu2024mathchat,
title={MathChat: Converse to Tackle Challenging Math Problems with {LLM} Agents},
author={Yiran Wu and Feiran Jia and Shaokun Zhang and Hangyu Li and Erkang Zhu and Yue Wang and Yin Tat Lee and Richard Peng and Qingyun Wu and Chi Wang},
booktitle={ICLR 2024 Workshop on Large Language Model (LLM) Agents},
year={2024},
url={https://openreview.net/forum?id=S7vIB7OGQe}
}

@inproceedings{
chu2025sft,
title={{SFT} Memorizes, {RL} Generalizes: A Comparative Study of Foundation Model Post-training},
author={Tianzhe Chu and Yuexiang Zhai and Jihan Yang and Shengbang Tong and Saining Xie and Dale Schuurmans and Quoc V Le and Sergey Levine and Yi Ma},
booktitle={Forty-second International Conference on Machine Learning},
year={2025},
url={https://openreview.net/forum?id=dYur3yabMj}
}

@misc{wei2025paftpromptagnosticfinetuning,
      title={PAFT: Prompt-Agnostic Fine-Tuning}, 
      author={Chenxing Wei and Yao Shu and Mingwen Ou and Ying Tiffany He and Fei Richard Yu},
      year={2025},
      eprint={2502.12859},
      archivePrefix={arXiv},
      primaryClass={cs.CL},
      url={https://arxiv.org/abs/2502.12859}, 
}

@inproceedings{
rafailov2023direct,
title={Direct Preference Optimization: Your Language Model is Secretly a Reward Model},
author={Rafael Rafailov and Archit Sharma and Eric Mitchell and Christopher D Manning and Stefano Ermon and Chelsea Finn},
booktitle={Thirty-seventh Conference on Neural Information Processing Systems},
year={2023},
url={https://openreview.net/forum?id=HPuSIXJaa9}
}

@inproceedings{
meng2024simpo,
title={Sim{PO}: Simple Preference Optimization with a Reference-Free Reward},
author={Yu Meng and Mengzhou Xia and Danqi Chen},
booktitle={The Thirty-eighth Annual Conference on Neural Information Processing Systems},
year={2024},
url={https://openreview.net/forum?id=3Tzcot1LKb}
}

@inproceedings{
wang2024mint,
title={{MINT}: Evaluating {LLM}s in Multi-turn Interaction with Tools and Language Feedback},
author={Xingyao Wang and Zihan Wang and Jiateng Liu and Yangyi Chen and Lifan Yuan and Hao Peng and Heng Ji},
booktitle={The Twelfth International Conference on Learning Representations},
year={2024},
url={https://openreview.net/forum?id=jp3gWrMuIZ}
}

@inproceedings{
zeng2024evaluating,
title={Evaluating Large Language Models at Evaluating Instruction Following},
author={Zhiyuan Zeng and Jiatong Yu and Tianyu Gao and Yu Meng and Tanya Goyal and Danqi Chen},
booktitle={The Twelfth International Conference on Learning Representations},
year={2024},
url={https://openreview.net/forum?id=tr0KidwPLc}
}

@misc{gao2024retrievalaugmentedgenerationlargelanguage,
      title={Retrieval-Augmented Generation for Large Language Models: A Survey}, 
      author={Yunfan Gao and Yun Xiong and Xinyu Gao and Kangxiang Jia and Jinliu Pan and Yuxi Bi and Yi Dai and Jiawei Sun and Meng Wang and Haofen Wang},
      year={2024},
      eprint={2312.10997},
      archivePrefix={arXiv},
      primaryClass={cs.CL},
      url={https://arxiv.org/abs/2312.10997}, 
}

@inproceedings{RAG,
author = {Lewis, Patrick and Perez, Ethan and Piktus, Aleksandra and Petroni, Fabio and Karpukhin, Vladimir and Goyal, Naman and K\"{u}ttler, Heinrich and Lewis, Mike and Yih, Wen-tau and Rockt\"{a}schel, Tim and Riedel, Sebastian and Kiela, Douwe},
title = {Retrieval-augmented generation for knowledge-intensive NLP tasks},
year = {2020},
isbn = {9781713829546},
publisher = {Curran Associates Inc.},
address = {Red Hook, NY, USA},
booktitle = {Proceedings of the 34th International Conference on Neural Information Processing Systems},
articleno = {793},
numpages = {16},
location = {Vancouver, BC, Canada},
series = {NIPS '20}
}

@inproceedings{
fang2025alphaedit,
title={AlphaEdit: Null-Space Constrained Model Editing for Language Models},
author={Junfeng Fang and Houcheng Jiang and Kun Wang and Yunshan Ma and Jie Shi and Xiang Wang and Xiangnan He and Tat-Seng Chua},
booktitle={The Thirteenth International Conference on Learning Representations},
year={2025},
url={https://openreview.net/forum?id=HvSytvg3Jh}
}

@inproceedings{yao-etal-2023-editing,
    title = "Editing Large Language Models: Problems, Methods, and Opportunities",
    author = "Yao, Yunzhi  and
      Wang, Peng  and
      Tian, Bozhong  and
      Cheng, Siyuan  and
      Li, Zhoubo  and
      Deng, Shumin  and
      Chen, Huajun  and
      Zhang, Ningyu",
    editor = "Bouamor, Houda  and
      Pino, Juan  and
      Bali, Kalika",
    booktitle = "Proceedings of the 2023 Conference on Empirical Methods in Natural Language Processing",
    month = dec,
    year = "2023",
    address = "Singapore",
    publisher = "Association for Computational Linguistics",
    url = "https://aclanthology.org/2023.emnlp-main.632/",
    doi = "10.18653/v1/2023.emnlp-main.632",
    pages = "10222--10240"
}

@inproceedings{
li2025testtime,
title={Test-Time Preference Optimization: On-the-Fly Alignment via Iterative Textual Feedback},
author={Yafu Li and Xuyang Hu and Xiaoye Qu and Linjie Li and Yu Cheng},
booktitle={Forty-second International Conference on Machine Learning},
year={2025},
url={https://openreview.net/forum?id=ArifAHrEVD}
}

@inproceedings{
shani2024multiturn,
title={Multi-turn Reinforcement Learning with Preference Human Feedback},
author={Lior Shani and Aviv Rosenberg and Asaf Cassel and Oran Lang and Daniele Calandriello and Avital Zipori and Hila Noga and Orgad Keller and Bilal Piot and Idan Szpektor and Avinatan Hassidim and Yossi Matias and Remi Munos},
booktitle={The Thirty-eighth Annual Conference on Neural Information Processing Systems},
year={2024},
url={https://openreview.net/forum?id=rVSc3HIZS4}
}

@misc{zuo2025ttrltesttimereinforcementlearning,
      title={TTRL: Test-Time Reinforcement Learning}, 
      author={Yuxin Zuo and Kaiyan Zhang and Li Sheng and Shang Qu and Ganqu Cui and Xuekai Zhu and Haozhan Li and Yuchen Zhang and Xinwei Long and Ermo Hua and Biqing Qi and Youbang Sun and Zhiyuan Ma and Lifan Yuan and Ning Ding and Bowen Zhou},
      year={2025},
      eprint={2504.16084},
      archivePrefix={arXiv},
      primaryClass={cs.CL},
      url={https://arxiv.org/abs/2504.16084}, 
}

@misc{hu2025slotsamplespecificlanguagemodel,
      title={SLOT: Sample-specific Language Model Optimization at Test-time},
      author={Yang Hu and Xingyu Zhang and Xueji Fang and Zhiyang Chen and Xiao Wang and Huatian Zhang and Guojun Qi},
      year={2025},
      eprint={2505.12392},
      archivePrefix={arXiv},
      primaryClass={cs.CL},
      url={[https://arxiv.org/abs/2505.12392](https://arxiv.org/abs/2505.12392)},
}

@inproceedings{
liu2023design,
title={Design from Policies: Conservative Test-Time Adaptation for Offline Policy Optimization},
author={Jinxin Liu and Hongyin Zhang and Zifeng Zhuang and Yachen Kang and Donglin Wang and Bin Wang},
booktitle={Thirty-seventh Conference on Neural Information Processing Systems},
year={2023},
url={https://openreview.net/forum?id=jZYf1GxH1V}
}

@inproceedings{ouyang2022traininglanguagemodelsfollow,
author = {Ouyang, Long and Wu, Jeff and Jiang, Xu and Almeida, Diogo and Wainwright, Carroll L. and Mishkin, Pamela and Zhang, Chong and Agarwal, Sandhini and Slama, Katarina and Ray, Alex and Schulman, John and Hilton, Jacob and Kelton, Fraser and Miller, Luke and Simens, Maddie and Askell, Amanda and Welinder, Peter and Christiano, Paul and Leike, Jan and Lowe, Ryan},
title = {Training language models to follow instructions with human feedback},
year = {2022},
isbn = {9781713871088},
publisher = {Curran Associates Inc.},
address = {Red Hook, NY, USA},
booktitle = {Proceedings of the 36th International Conference on Neural Information Processing Systems},
articleno = {2011},
numpages = {15},
location = {New Orleans, LA, USA},
series = {NIPS '22}
}

@article{hendrycksmath2021,
  title={Measuring Mathematical Problem Solving With the MATH Dataset},
  author={Dan Hendrycks and Collin Burns and Saurav Kadavath and Akul Arora and Steven Basart and Eric Tang and Dawn Song and Jacob Steinhardt},
  journal={NeurIPS},
  year={2021}
}

@article{lightman2023lets,
      title={Let's Verify Step by Step}, 
      author={Lightman, Hunter and Kosaraju, Vineet and Burda, Yura and Edwards, Harri and Baker, Bowen and Lee, Teddy and Leike, Jan and Schulman, John and Sutskever, Ilya and Cobbe, Karl},
      journal={arXiv preprint arXiv:2305.20050},
      year={2023}
}

@misc{aime,
      title={{AIME} Problems and Solutions},
      author={{AIME}},
      year={2025},
      url={https://artofproblemsolving.com/wiki/index.php/AIME_Problems_and_Solutions}
}

@misc{yang2025qwen3technicalreport,
      title={Qwen3 Technical Report}, 
      author={An Yang and Anfeng Li and Baosong Yang and Beichen Zhang and Binyuan Hui and Bo Zheng and Bowen Yu and Chang Gao and Chengen Huang and Chenxu Lv and Chujie Zheng and Dayiheng Liu and Fan Zhou and Fei Huang and Feng Hu and Hao Ge and Haoran Wei and Huan Lin and Jialong Tang and Jian Yang and Jianhong Tu and Jianwei Zhang and Jianxin Yang and Jiaxi Yang and Jing Zhou and Jingren Zhou and Junyang Lin and Kai Dang and Keqin Bao and Kexin Yang and Le Yu and Lianghao Deng and Mei Li and Mingfeng Xue and Mingze Li and Pei Zhang and Peng Wang and Qin Zhu and Rui Men and Ruize Gao and Shixuan Liu and Shuang Luo and Tianhao Li and Tianyi Tang and Wenbiao Yin and Xingzhang Ren and Xinyu Wang and Xinyu Zhang and Xuancheng Ren and Yang Fan and Yang Su and Yichang Zhang and Yinger Zhang and Yu Wan and Yuqiong Liu and Zekun Wang and Zeyu Cui and Zhenru Zhang and Zhipeng Zhou and Zihan Qiu},
      year={2025},
      eprint={2505.09388},
      archivePrefix={arXiv},
      primaryClass={cs.CL},
      url={https://arxiv.org/abs/2505.09388}, 
}

@misc{qwen2025qwen25technicalreport,
      title={Qwen2.5 Technical Report}, 
      author={Qwen and : and An Yang and Baosong Yang and Beichen Zhang and Binyuan Hui and Bo Zheng and Bowen Yu and Chengyuan Li and Dayiheng Liu and Fei Huang and Haoran Wei and Huan Lin and Jian Yang and Jianhong Tu and Jianwei Zhang and Jianxin Yang and Jiaxi Yang and Jingren Zhou and Junyang Lin and Kai Dang and Keming Lu and Keqin Bao and Kexin Yang and Le Yu and Mei Li and Mingfeng Xue and Pei Zhang and Qin Zhu and Rui Men and Runji Lin and Tianhao Li and Tianyi Tang and Tingyu Xia and Xingzhang Ren and Xuancheng Ren and Yang Fan and Yang Su and Yichang Zhang and Yu Wan and Yuqiong Liu and Zeyu Cui and Zhenru Zhang and Zihan Qiu},
      year={2025},
      eprint={2412.15115},
      archivePrefix={arXiv},
      primaryClass={cs.CL},
      url={https://arxiv.org/abs/2412.15115}, 
}

@misc{deepseekai2025deepseekr1incentivizingreasoningcapability,
      title={DeepSeek-R1: Incentivizing Reasoning Capability in LLMs via Reinforcement Learning}, 
      author={DeepSeek-AI and Daya Guo and Dejian Yang and Haowei Zhang and Junxiao Song and Ruoyu Zhang and Runxin Xu and Qihao Zhu and Shirong Ma and Peiyi Wang and Xiao Bi and Xiaokang Zhang and Xingkai Yu and Yu Wu and Z. F. Wu and Zhibin Gou and Zhihong Shao and Zhuoshu Li and Ziyi Gao and Aixin Liu and Bing Xue and Bingxuan Wang and Bochao Wu and Bei Feng and Chengda Lu and Chenggang Zhao and Chengqi Deng and Chenyu Zhang and Chong Ruan and Damai Dai and Deli Chen and Dongjie Ji and Erhang Li and Fangyun Lin and Fucong Dai and Fuli Luo and Guangbo Hao and Guanting Chen and Guowei Li and H. Zhang and Han Bao and Hanwei Xu and Haocheng Wang and Honghui Ding and Huajian Xin and Huazuo Gao and Hui Qu and Hui Li and Jianzhong Guo and Jiashi Li and Jiawei Wang and Jingchang Chen and Jingyang Yuan and Junjie Qiu and Junlong Li and J. L. Cai and Jiaqi Ni and Jian Liang and Jin Chen and Kai Dong and Kai Hu and Kaige Gao and Kang Guan and Kexin Huang and Kuai Yu and Lean Wang and Lecong Zhang and Liang Zhao and Litong Wang and Liyue Zhang and Lei Xu and Leyi Xia and Mingchuan Zhang and Minghua Zhang and Minghui Tang and Meng Li and Miaojun Wang and Mingming Li and Ning Tian and Panpan Huang and Peng Zhang and Qiancheng Wang and Qinyu Chen and Qiushi Du and Ruiqi Ge and Ruisong Zhang and Ruizhe Pan and Runji Wang and R. J. Chen and R. L. Jin and Ruyi Chen and Shanghao Lu and Shangyan Zhou and Shanhuang Chen and Shengfeng Ye and Shiyu Wang and Shuiping Yu and Shunfeng Zhou and Shuting Pan and S. S. Li and Shuang Zhou and Shaoqing Wu and Shengfeng Ye and Tao Yun and Tian Pei and Tianyu Sun and T. Wang and Wangding Zeng and Wanjia Zhao and Wen Liu and Wenfeng Liang and Wenjun Gao and Wenqin Yu and Wentao Zhang and W. L. Xiao and Wei An and Xiaodong Liu and Xiaohan Wang and Xiaokang Chen and Xiaotao Nie and Xin Cheng and Xin Liu and Xin Xie and Xingchao Liu and Xinyu Yang and Xinyuan Li and Xuecheng Su and Xuheng Lin and X. Q. Li and Xiangyue Jin and Xiaojin Shen and Xiaosha Chen and Xiaowen Sun and Xiaoxiang Wang and Xinnan Song and Xinyi Zhou and Xianzu Wang and Xinxia Shan and Y. K. Li and Y. Q. Wang and Y. X. Wei and Yang Zhang and Yanhong Xu and Yao Li and Yao Zhao and Yaofeng Sun and Yaohui Wang and Yi Yu and Yichao Zhang and Yifan Shi and Yiliang Xiong and Ying He and Yishi Piao and Yisong Wang and Yixuan Tan and Yiyang Ma and Yiyuan Liu and Yongqiang Guo and Yuan Ou and Yuduan Wang and Yue Gong and Yuheng Zou and Yujia He and Yunfan Xiong and Yuxiang Luo and Yuxiang You and Yuxuan Liu and Yuyang Zhou and Y. X. Zhu and Yanhong Xu and Yanping Huang and Yaohui Li and Yi Zheng and Yuchen Zhu and Yunxian Ma and Ying Tang and Yukun Zha and Yuting Yan and Z. Z. Ren and Zehui Ren and Zhangli Sha and Zhe Fu and Zhean Xu and Zhenda Xie and Zhengyan Zhang and Zhewen Hao and Zhicheng Ma and Zhigang Yan and Zhiyu Wu and Zihui Gu and Zijia Zhu and Zijun Liu and Zilin Li and Ziwei Xie and Ziyang Song and Zizheng Pan and Zhen Huang and Zhipeng Xu and Zhongyu Zhang and Zhen Zhang},
      year={2025},
      eprint={2501.12948},
      archivePrefix={arXiv},
      primaryClass={cs.CL},
      url={https://arxiv.org/abs/2501.12948}, 
}

@book{Atkinson1988,
  author    = {Atkinson, Kendall E.},
  title     = {An Introduction to Numerical Analysis},
  year      = {1988},
  edition   = {2nd},
  publisher = {John Wiley and Sons},
  isbn      = {978-0-471-50023-0},
  note      = {Section 8.9}
}

@inproceedings{rein2024gpqa,
      title={{GPQA}: A Graduate-Level Google-Proof Q\&A Benchmark},
      author={David Rein and Betty Li Hou and Asa Cooper Stickland and Jackson Petty and Richard Yuanzhe Pang and Julien Dirani and Julian Michael and Samuel R. Bowman},
      booktitle={First Conference on Language Modeling},
      year={2024},
      url={https://openreview.net/forum?id=Ti67584b98}
}

@article{hendryckstest2021,
  title={Measuring Massive Multitask Language Understanding},
  author={Dan Hendrycks and Collin Burns and Steven Basart and Andy Zou and Mantas Mazeika and Dawn Song and Jacob Steinhardt},
  journal={Proceedings of the International Conference on Learning Representations (ICLR)},
  year={2021}
}

@misc{pteam2025supergpqascalingllmevaluation,
      title={SuperGPQA: Scaling LLM Evaluation across 285 Graduate Disciplines}, 
      author={M-A-P Team and Xinrun Du and Yifan Yao and Kaijing Ma and Bingli Wang and Tianyu Zheng and Kang Zhu and Minghao Liu and Yiming Liang and Xiaolong Jin and Zhenlin Wei and Chujie Zheng and Kaixing Deng and Shuyue Guo and Shian Jia and Sichao Jiang and Yiyan Liao and Rui Li and Qinrui Li and Sirun Li and Yizhi Li and Yunwen Li and Dehua Ma and Yuansheng Ni and Haoran Que and Qiyao Wang and Zhoufutu Wen and Siwei Wu and Tianshun Xing and Ming Xu and Zhenzhu Yang and Zekun Moore Wang and Junting Zhou and Yuelin Bai and Xingyuan Bu and Chenglin Cai and Liang Chen and Yifan Chen and Chengtuo Cheng and Tianhao Cheng and Keyi Ding and Siming Huang and Yun Huang and Yaoru Li and Yizhe Li and Zhaoqun Li and Tianhao Liang and Chengdong Lin and Hongquan Lin and Yinghao Ma and Zhongyuan Peng and Zifan Peng and Qige Qi and Shi Qiu and Xingwei Qu and Yizhou Tan and Zili Wang and Chenqing Wang and Hao Wang and Yiya Wang and Yubo Wang and Jiajun Xu and Kexin Yang and Ruibin Yuan and Yuanhao Yue and Tianyang Zhan and Chun Zhang and Jingyang Zhang and Xiyue Zhang and Xingjian Zhang and Yue Zhang and Yongchi Zhao and Xiangyu Zheng and Chenghua Zhong and Yang Gao and Zhoujun Li and Dayiheng Liu and Qian Liu and Tianyu Liu and Shiwen Ni and Junran Peng and Yujia Qin and Wenbo Su and Guoyin Wang and Shi Wang and Jian Yang and Min Yang and Meng Cao and Xiang Yue and Zhaoxiang Zhang and Wangchunshu Zhou and Jiaheng Liu and Qunshu Lin and Wenhao Huang and Ge Zhang},
      year={2025},
      eprint={2502.14739},
      archivePrefix={arXiv},
      primaryClass={cs.CL},
      url={https://arxiv.org/abs/2502.14739}, 
}

@misc{chen2021evaluatinglargelanguagemodels,
      title={Evaluating Large Language Models Trained on Code}, 
      author={Mark Chen and Jerry Tworek and Heewoo Jun and Qiming Yuan and Henrique Ponde de Oliveira Pinto and Jared Kaplan and Harri Edwards and Yuri Burda and Nicholas Joseph and Greg Brockman and Alex Ray and Raul Puri and Gretchen Krueger and Michael Petrov and Heidy Khlaaf and Girish Sastry and Pamela Mishkin and Brooke Chan and Scott Gray and Nick Ryder and Mikhail Pavlov and Alethea Power and Lukasz Kaiser and Mohammad Bavarian and Clemens Winter and Philippe Tillet and Felipe Petroski Such and Dave Cummings and Matthias Plappert and Fotios Chantzis and Elizabeth Barnes and Ariel Herbert-Voss and William Hebgen Guss and Alex Nichol and Alex Paino and Nikolas Tezak and Jie Tang and Igor Babuschkin and Suchir Balaji and Shantanu Jain and William Saunders and Christopher Hesse and Andrew N. Carr and Jan Leike and Josh Achiam and Vedant Misra and Evan Morikawa and Alec Radford and Matthew Knight and Miles Brundage and Mira Murati and Katie Mayer and Peter Welinder and Bob McGrew and Dario Amodei and Sam McCandlish and Ilya Sutskever and Wojciech Zaremba},
      year={2021},
      eprint={2107.03374},
      archivePrefix={arXiv},
      primaryClass={cs.LG},
      url={https://arxiv.org/abs/2107.03374}, 
}

@article{son2025linguistic,
  title={Linguistic Generalizability of Test-Time Scaling in Mathematical Reasoning},
  author={Son, Guijin and Hong, Jiwoo and Ko, Hyunwoo and Thorne, James},
  journal={arXiv preprint arXiv:2502.17407},
  year={2025}
}

@misc{shi2025wildfeedbackaligningllmsinsitu,
      title={WildFeedback: Aligning LLMs With In-situ User Interactions And Feedback}, 
      author={Taiwei Shi and Zhuoer Wang and Longqi Yang and Ying-Chun Lin and Zexue He and Mengting Wan and Pei Zhou and Sujay Jauhar and Sihao Chen and Shan Xia and Hongfei Zhang and Jieyu Zhao and Xiaofeng Xu and Xia Song and Jennifer Neville},
      year={2025},
      eprint={2408.15549},
      archivePrefix={arXiv},
      primaryClass={cs.CL},
      url={https://arxiv.org/abs/2408.15549}, 
}

@inproceedings{
chen2025codesteer,
title={CodeSteer: Symbolic-Augmented Language Models via Code/Text Guidance},
author={Yongchao Chen and Yilun Hao and Yueying Liu and Yang Zhang and Chuchu Fan},
booktitle={Forty-second International Conference on Machine Learning},
year={2025},
url={https://openreview.net/forum?id=ezna4V4zHs}
}

@inproceedings{
qu2024recursive,
title={Recursive Introspection: Teaching Language Model Agents How to Self-Improve},
author={Yuxiao Qu and Tianjun Zhang and Naman Garg and Aviral Kumar},
booktitle={The Thirty-eighth Annual Conference on Neural Information Processing Systems},
year={2024},
url={https://openreview.net/forum?id=DRC9pZwBwR}
}

@misc{kumar2024traininglanguagemodelsselfcorrect,
      title={Training Language Models to Self-Correct via Reinforcement Learning}, 
      author={Aviral Kumar and Vincent Zhuang and Rishabh Agarwal and Yi Su and John D Co-Reyes and Avi Singh and Kate Baumli and Shariq Iqbal and Colton Bishop and Rebecca Roelofs and Lei M Zhang and Kay McKinney and Disha Shrivastava and Cosmin Paduraru and George Tucker and Doina Precup and Feryal Behbahani and Aleksandra Faust},
      year={2024},
      eprint={2409.12917},
      archivePrefix={arXiv},
      primaryClass={cs.LG},
      url={https://arxiv.org/abs/2409.12917}, 
}

@inproceedings{shi-etal-2024-direct,
    title = "Direct Multi-Turn Preference Optimization for Language Agents",
    author = "Shi, Wentao  and
      Yuan, Mengqi  and
      Wu, Junkang  and
      Wang, Qifan  and
      Feng, Fuli",
    editor = "Al-Onaizan, Yaser  and
      Bansal, Mohit  and
      Chen, Yun-Nung",
    booktitle = "Proceedings of the 2024 Conference on Empirical Methods in Natural Language Processing",
    month = nov,
    year = "2024",
    address = "Miami, Florida, USA",
    publisher = "Association for Computational Linguistics",
    url = "https://aclanthology.org/2024.emnlp-main.138/",
    doi = "10.18653/v1/2024.emnlp-main.138",
    pages = "2312--2324"
}

@misc{yang2025zeroguiautomatingonlinegui,
      title={ZeroGUI: Automating Online GUI Learning at Zero Human Cost}, 
      author={Chenyu Yang and Shiqian Su and Shi Liu and Xuan Dong and Yue Yu and Weijie Su and Xuehui Wang and Zhaoyang Liu and Jinguo Zhu and Hao Li and Wenhai Wang and Yu Qiao and Xizhou Zhu and Jifeng Dai},
      year={2025},
      eprint={2505.23762},
      archivePrefix={arXiv},
      primaryClass={cs.AI},
      url={https://arxiv.org/abs/2505.23762}, 
}

@inproceedings{
chen2025learning,
title={Learning to Clarify: Multi-turn Conversations with Action-Based Contrastive Self-Training},
author={Maximillian Chen and Ruoxi Sun and Tomas Pfister and Sercan O Arik},
booktitle={The Thirteenth International Conference on Learning Representations},
year={2025},
url={https://openreview.net/forum?id=SIE6VFps9x}
}

@inbook{sgd,
  title={Optimization for machine learning},
  author={Sra, Suvrit and Nowozin, Sebastian and Wright, Stephen J},
  pages = {351–368},
  year={2011},
  publisher={Mit Press}
}

@misc{kingma2017adammethodstochasticoptimization,
      title={Adam: A Method for Stochastic Optimization}, 
      author={Diederik P. Kingma and Jimmy Ba},
      year={2017},
      eprint={1412.6980},
      archivePrefix={arXiv},
      primaryClass={cs.LG},
      url={https://arxiv.org/abs/1412.6980}, 
}

@inproceedings{ou-etal-2024-inductive,
    title = "Inductive-Deductive Strategy Reuse for Multi-Turn Instructional Dialogues",
    author = "Ou, Jiao  and
      Wu, Jiayu  and
      Liu, Che  and
      Zhang, Fuzheng  and
      Zhang, Di  and
      Gai, Kun",
    editor = "Al-Onaizan, Yaser  and
      Bansal, Mohit  and
      Chen, Yun-Nung",
    booktitle = "Proceedings of the 2024 Conference on Empirical Methods in Natural Language Processing",
    month = nov,
    year = "2024",
    address = "Miami, Florida, USA",
    publisher = "Association for Computational Linguistics",
    url = "https://aclanthology.org/2024.emnlp-main.964/",
    doi = "10.18653/v1/2024.emnlp-main.964",
    pages = "17402--17431"
}

@inproceedings{sun-etal-2024-parrot,
    title = "Parrot: Enhancing Multi-Turn Instruction Following for Large Language Models",
    author = "Sun, Yuchong  and
      Liu, Che  and
      Zhou, Kun  and
      Huang, Jinwen  and
      Song, Ruihua  and
      Zhao, Xin  and
      Zhang, Fuzheng  and
      Zhang, Di  and
      Gai, Kun",
    editor = "Ku, Lun-Wei  and
      Martins, Andre  and
      Srikumar, Vivek",
    booktitle = "Proceedings of the 62nd Annual Meeting of the Association for Computational Linguistics (Volume 1: Long Papers)",
    month = aug,
    year = "2024",
    address = "Bangkok, Thailand",
    publisher = "Association for Computational Linguistics",
    url = "https://aclanthology.org/2024.acl-long.525/",
    doi = "10.18653/v1/2024.acl-long.525",
    pages = "9729--9750"
}

@inproceedings{zhou-etal-2024-enhancing,
    title = "Enhancing the General Agent Capabilities of Low-Paramter {LLM}s through Tuning and Multi-Branch Reasoning",
    author = "Zhou, Qinhao  and
      Zhang, Zihan  and
      Xiang, Xiang  and
      Wang, Ke  and
      Wu, Yuchuan  and
      Li, Yongbin",
    editor = "Duh, Kevin  and
      Gomez, Helena  and
      Bethard, Steven",
    booktitle = "Findings of the Association for Computational Linguistics: NAACL 2024",
    month = jun,
    year = "2024",
    address = "Mexico City, Mexico",
    publisher = "Association for Computational Linguistics",
    url = "https://aclanthology.org/2024.findings-naacl.184/",
    doi = "10.18653/v1/2024.findings-naacl.184",
    pages = "2922--2931"
}

@misc{zhou2025sweetrltrainingmultiturnllm,
      title={SWEET-RL: Training Multi-Turn LLM Agents on Collaborative Reasoning Tasks}, 
      author={Yifei Zhou and Song Jiang and Yuandong Tian and Jason Weston and Sergey Levine and Sainbayar Sukhbaatar and Xian Li},
      year={2025},
      eprint={2503.15478},
      archivePrefix={arXiv},
      primaryClass={cs.LG},
      url={https://arxiv.org/abs/2503.15478}, 
}

@misc{liu2025spiralselfplayzerosumgames,
      title={SPIRAL: Self-Play on Zero-Sum Games Incentivizes Reasoning via Multi-Agent Multi-Turn Reinforcement Learning}, 
      author={Bo Liu and Leon Guertler and Simon Yu and Zichen Liu and Penghui Qi and Daniel Balcells and Mickel Liu and Cheston Tan and Weiyan Shi and Min Lin and Wee Sun Lee and Natasha Jaques},
      year={2025},
      eprint={2506.24119},
      archivePrefix={arXiv},
      primaryClass={cs.AI},
      url={https://arxiv.org/abs/2506.24119}, 
}

@misc{zhang2025surveytesttimescalinglarge,
      title={A Survey on Test-Time Scaling in Large Language Models: What, How, Where, and How Well?}, 
      author={Qiyuan Zhang and Fuyuan Lyu and Zexu Sun and Lei Wang and Weixu Zhang and Wenyue Hua and Haolun Wu and Zhihan Guo and Yufei Wang and Niklas Muennighoff and Irwin King and Xue Liu and Chen Ma},
      year={2025},
      eprint={2503.24235},
      archivePrefix={arXiv},
      primaryClass={cs.CL},
      url={https://arxiv.org/abs/2503.24235}, 
}

@inproceedings{
hu2022lora,
title={Lo{RA}: Low-Rank Adaptation of Large Language Models},
author={Edward J Hu and Yelong Shen and Phillip Wallis and Zeyuan Allen-Zhu and Yuanzhi Li and Shean Wang and Lu Wang and Weizhu Chen},
booktitle={International Conference on Learning Representations},
year={2022},
url={https://openreview.net/forum?id=nZeVKeeFYf9}
}

@inproceedings{Sheng_2025, series={EuroSys ’25},
   title={HybridFlow: A Flexible and Efficient RLHF Framework},
   url={http://dx.doi.org/10.1145/3689031.3696075},
   DOI={10.1145/3689031.3696075},
   booktitle={Proceedings of the Twentieth European Conference on Computer Systems},
   publisher={ACM},
   author={Sheng, Guangming and Zhang, Chi and Ye, Zilingfeng and Wu, Xibin and Zhang, Wang and Zhang, Ru and Peng, Yanghua and Lin, Haibin and Wu, Chuan},
   year={2025},
   month=mar, pages={1279–1297},
   collection={EuroSys ’25} }

@inproceedings{
khattab2024dspy,
title={{DSP}y: Compiling Declarative Language Model Calls into State-of-the-Art Pipelines},
author={Omar Khattab and Arnav Singhvi and Paridhi Maheshwari and Zhiyuan Zhang and Keshav Santhanam and Sri Vardhamanan A and Saiful Haq and Ashutosh Sharma and Thomas T. Joshi and Hanna Moazam and Heather Miller and Matei Zaharia and Christopher Potts},
booktitle={The Twelfth International Conference on Learning Representations},
year={2024},
url={https://openreview.net/forum?id=sY5N0zY5Od}
}

@article{yuksekgonul2025optimizing,
  title={Optimizing generative ai by backpropagating language model feedback},
  author={Yuksekgonul, Mert and Bianchi, Federico and Boen, Joseph and Liu, Sheng and Lu, Pan and Huang, Zhi and Guestrin, Carlos and Zou, James},
  journal={Nature},
  volume={639},
  number={8055},
  pages={609--616},
  year={2025},
  publisher={Nature Publishing Group UK London}
}

@misc{ding2025scalingtextualgradientssamplingbased,
      title={Scaling Textual Gradients via Sampling-Based Momentum}, 
      author={Zixin Ding and Junyuan Hong and Zhan Shi and Jiachen T. Wang and Zinan Lin and Li Yin and Meng Liu and Zhangyang Wang and Yuxin Chen},
      year={2025},
      eprint={2506.00400},
      archivePrefix={arXiv},
      primaryClass={cs.CL},
      url={https://arxiv.org/abs/2506.00400}, 
}

@inproceedings{
zuo2026ttrl,
title={{TTRL}: Test-Time Reinforcement Learning},
author={Yuxin Zuo and Kaiyan Zhang and Li Sheng and Shang Qu and Ganqu Cui and Xuekai Zhu and Haozhan Li and Yuchen Zhang and Xinwei Long and Ermo Hua and Biqing Qi and Youbang Sun and Zhiyuan Ma and Lifan Yuan and Ning Ding and Bowen Zhou},
booktitle={The Thirty-ninth Annual Conference on Neural Information Processing Systems},
year={2026},
url={https://openreview.net/forum?id=VuVhgEiu20}
}

@inproceedings{
ren2025learning,
title={Learning Dynamics of {LLM} Finetuning},
author={Yi Ren and Danica J. Sutherland},
booktitle={The Thirteenth International Conference on Learning Representations},
year={2025},
url={https://openreview.net/forum?id=tPNHOoZFl9}
}

@inproceedings{rosa-wu-2025-aloe,
    title = "Aligning {LLM}s with Individual Preferences via Interaction",
    author = "Wu, Shujin  and
      Fung, Yi R.  and
      Qian, Cheng  and
      Kim, Jeonghwan  and
      Hakkani-Tur, Dilek  and
      Ji, Heng",
    editor = "Rambow, Owen  and
      Wanner, Leo  and
      Apidianaki, Marianna  and
      Al-Khalifa, Hend  and
      Eugenio, Barbara Di  and
      Schockaert, Steven",
    booktitle = "Proceedings of the 31st International Conference on Computational Linguistics",
    month = jan,
    year = "2025",
    address = "Abu Dhabi, UAE",
    publisher = "Association for Computational Linguistics",
    url = "https://aclanthology.org/2025.coling-main.511/",
    pages = "7648--7662"
}
